\documentclass[journal]{IEEEtran}
\usepackage{amsmath,amsfonts}
\usepackage{algorithm}
\usepackage{algpseudocode}
\usepackage{color}
\usepackage{soul}
\usepackage{cite}
\usepackage{graphicx}
\usepackage{amsthm}
\usepackage{multirow}
\def\BibTeX{{\rm B\kern-.05em{\sc i\kern-.025em b}\kern-.08em
    T\kern-.1667em\lower.7ex\hbox{E}\kern-.125emX}}
\usepackage{subfigure}
\DeclareMathOperator*{\argmax}{argmax}
\begin{document}
\title{Integrated Decision Making and Trajectory Planning for Autonomous Driving Under Multimodal Uncertainties: A Bayesian Game Approach}
\author{Zhenmin Huang, Tong Li, Shaojie Shen, and Jun Ma
\thanks{Zhenmin Huang, Tong Li, Shaojie Shen, and Jun Ma are with the Department of Electronic and Computer Engineering, The Hong Kong University of Science and Technology, Hong Kong SAR, China (e-mail: zhuangdf@connect.ust.hk; tlibm@connect.ust.hk; eeshaojie@ust.hk; jun.ma@ust.hk).

This work has been submitted to the IEEE for possible publication.
Copyright may be transferred without notice, after which this version may
no longer be accessible.
}}

% \markboth{IEEE Transactions on Intelligent Transportation Systems}{}
\maketitle

\begin{abstract} 
Modeling the interaction between traffic agents is a key issue in designing safe and non-conservative maneuvers in autonomous driving. This problem can be challenging when multi-modality and behavioral uncertainties are engaged. Existing methods either fail to plan interactively or consider unimodal behaviors that could lead to catastrophic results. In this paper, we introduce an integrated decision-making and trajectory planning framework based on Bayesian game (i.e., game of incomplete information). Human decisions inherently exhibit discrete characteristics and therefore are modeled as types of players in the game. A general solver based on no-regret learning is introduced to obtain a corresponding Bayesian Coarse Correlated Equilibrium, which captures the interaction between traffic agents in the multimodal context. With the attained equilibrium, decision-making and trajectory planning are performed simultaneously, and the resulting interactive strategy is shown to be optimal over the expectation of rivals' driving intentions. Closed-loop simulations on different traffic scenarios are performed to illustrate the generalizability and the effectiveness of the proposed framework.
\end{abstract}

\begin{IEEEkeywords}
Autonomous driving, game theory, multi-agent systems, decision-making, trajectory planning, game of incomplete information, Bayesian coarse correlated equilibrium, no-regret learning.
\end{IEEEkeywords}

\theoremstyle{definition}
\newtheorem{definition}{Definition}
\theoremstyle{definition}
\newtheorem{remark}{Remark}
\theoremstyle{definition}
\newtheorem{assumption}{Assumption}
\theoremstyle{definition}
\newtheorem{lemma}{Lemma}
\theoremstyle{definition}
\newtheorem{theorem}{Theorem}
\graphicspath{ {pictures/}}

\section{Introduction}

Recent years have witnessed great advancements in the technology of autonomous vehicles, which aims to provide a new transportation style that improves traffic efficiency and safety. On the other hand, many technical difficulties remain unsolved in decision-making and trajectory planning for autonomous vehicles across highly complex traffic scenarios. Within this topic, one of the key issues is to effectively model the interaction between autonomous vehicles and other traffic participants, such as human drivers. This problem is particularly challenging, as the rationality of humans is often limited and the preferences of human drivers vary, leading to large uncertainties in decisions and actions. Many of the approaches in the literature adopt an architecture such that prediction and trajectory planning are essentially decoupled. In particular, a prediction module is utilized to forecast the future trajectories of human-driving vehicles. Afterward, the human-driving vehicles are fed into the planning module as immutable objects, and the planning process is grounded on this immutable hypothesis. These approaches are limited as they largely fail to capture the interactive nature of driving behaviors by ignoring the influence of the planned trajectories on other traffic participants. Also, they can lead to the ``frozen robot'' problem~\cite{trautman2010unfreezing} and overly conservative behaviors when all possible trajectories are potentially unsafe according to the predictions.

In light of this, a branch of methods based on game theory is receiving increasing attention owing to its strength in modeling the interactions between agents. The game theory models the self-interested nature of all participants that drives them to a particular equilibrium, which is optimal in the sense that none of the participants can gain more profits by unilaterally altering its own strategy. By calculating the equilibrium of the game, the prediction and the planning are essentially coupled, resulting in an optimal trajectory with interaction properly considered. Efforts have been paid in this direction to model various traffic scenarios such as intersections~\cite{jia2023interactive,hang2022decision}, ramp-merging~\cite{cleac2019algames,hang2021cooperative}, lane-changing~\cite{fabiani2019multi,lopez2022game}, and roundabouts~\cite{hang2021decision}. Among game-based methods, various types of games are intensively studied such as potential games~\cite{liu2023potential,kavuncu2021potential}, congestion game~\cite{zanardi2023bad}, auction game~\cite{chandra2022gameplan}, and Stackelberg games~\cite{fisac2019hierarchical,yu2018human}. Nevertheless, the majority of game-based methods focus on achieving a single equilibrium and largely neglect the multimodal nature of the solution. In reality, the decision space of human drivers is often discrete, and multiple distinct decisions can be reasonable with respect to a specific traffic scenario, while a single equilibrium can only reveal one of them. Catastrophic results are prone to occur when there is a discrepancy between the mode of the calculated equilibrium and the actual decisions made by other traffic participants. A few game-based methods aim to solve this problem by considering the multimodality in behaviors of the rival traffic participants~\cite{deng2022lane,zhao2023non,zhang2022human}, but the game settings are simple and scenario-specific, which prevent these methods from generalizing to complex game settings and a variety of traffic scenarios. Also, in these methods, the decision-making and the planning are separated, and therefore the results are prone to be suboptimal.

To address the aforementioned challenges, we present an innovative integrated decision-making and trajectory planning framework for autonomous vehicles, which properly considers the multimodal nature of driving behaviors utilizing game theory. In particular, interactions of traffic participants are modeled as a general Bayesian game (i.e., game of incomplete information) that can potentially be applied to various traffic scenarios. By computing the corresponding Bayesian Coarse Correlated Equilibrium (Bayes-CCE), the decision and the corresponding trajectory are obtained simultaneously, which are shown to be optimal over the expectations of the underlying intentions of other participants. An updating strategy based on Bayesian filtering is also introduced to update the estimation of the driving intentions. The main contribution of this paper is listed as follows:

\begin{itemize}
\item[$\bullet$]
The interactions between traffic participants are effectively modeled from the perspective of a Bayesian game, such that the multimodalities in driving behaviors are properly addressed. Unlike most of the existing methods, the game formulation is generic with multiple game stages and the information structures are potentially adaptive. Therefore, it exhibits the generalizability to various traffic scenarios.
\end{itemize}

\begin{itemize}
\item[$\bullet$]
A parallelizable Monte Carlo search algorithm is proposed to solve the introduced game with arbitrarily complex information structures. Meanwhile, rigorous proof is established to demonstrate the almost-sure convergence of the introduced algorithm to a Bayes-CCE.
\end{itemize}

\begin{itemize}
\item[$\bullet$]
With the attained Bayes-CCE, an integrated decision-making and trajectory planning strategy is presented to obtain the decision and the corresponding trajectory simultaneously, which bridges the gap between decision-making and trajectory planning to avoid suboptimality. Moreover, theoretic analysis is performed to demonstrate the optimality of the resulting strategy in terms of the expectation, given the up-to-date beliefs over the underlying driving intentions of other vehicles.
\end{itemize}

\begin{itemize}
\item[$\bullet$]
With a belief updating strategy based on Bayesian filtering to maintain the estimation of the participants' driving intentions, a closed-loop autonomous driving framework is given. A series of simulations are performed to illustrate the superiority of the proposed method over traditional game methods.
\end{itemize}

\section{Related Works}
\subsection{Non-Game-Theoretic Decision-Making and Trajectory Planning}
In the past years, efforts have been made to tackle the decision-making and trajectory-planning problems for autonomous driving. State-machine-based methods~\cite{ziegler2014making,urmson2008autonomous,montemerlo2008junior}, which utilize hand-crafted rules to generate decisions and trajectories based on the current state of the autonomous vehicle, are popular at the early stage due to their simplicity and interpretability. However, the innumerable nature of traffic scenarios prevents the wide application of these methods, as the difficulty in maintaining such a state machine increases dramatically with increasing scenario complexities. The research focus soon shifts to a more generic formulation of the problem. In light of this, various methods are proposed by actively employing the concept of searching~\cite{wei2014behavioral,zhan2017spatially,ajanovicsearch}, optimization~\cite{kessler2019cooperative,Ma2022Alternating,huang2023decentralized}, sampling~\cite{mustafa2024racp,werling2012optimal}, and POMDP~\cite{liu2015situation}, etc. Recently, a branch of multi-policy planning methods, including MPDM~\cite{cunningham2015mpdm}, EUDM~\cite{zhang2020efficient}, EPSILON~\cite{9526613}, and MARC~\cite{li2023marc}, are gaining attention due to their ability in efficient behavior planning. These methods rely on dedicated forward simulators to evaluate the performance of each policy and perform the best one selectively. Nonetheless, the aforementioned approaches generally suffer from difficulties in modeling the interactions, particularly in the effect of the trajectory of the autonomous vehicle on the behaviors of the other vehicles. Meanwhile, a series of recent works address the decision-making and trajectory planning problem using a unified neural network trained through end-to-end imitation learning ~\cite{hu2023planning,huang2023differentiable,huang2023gameformer}. However, the lack of interpretability brings safety concerns, which prevents the wide application of these methods in real-world situations.

\subsection{Game-Theoretic Decision-Making and Trajectory Planning}
To tackle the problem mentioned in the previous subsection, game-based methods are extensively researched. In this area, many of the existing methods calculate a single equilibrium and assume that all participants will follow the computed equilibrium~\cite{fridovich2020efficient,cleac2019algames,fisac2019hierarchical,li2022efficient,liu2023potential}. These methods generally assume no uncertainty in agents' behaviors. In~\cite{fridovich2020efficient,cleac2019algames}, efficient game-solving strategies are proposed to resolve differential games and reach typical Nash equilibriums. In~\cite{fisac2019hierarchical}, the interaction between traffic agents is formulated as a Stackelberg game, and a solving scheme based on dynamic programming is proposed to solve the game in real time. In~\cite{liu2023potential}, a general decision-making framework
is introduced, which shows that typical traffic scenarios can be modeled as a potential game. In potential games, a pure-strategy Nash equilibrium is proven to exist and reachable, and therefore effective decision-making strategies can be derived. In~\cite{li2022efficient}, an efficient game-theoretic trajectory planning algorithm is introduced based on Monte Carlo Tree Search, together with an online estimation algorithm to identify the driving preferences of road participants. Although different types of games and various solving schemes are intensively studied in these researches, the uncertainties in human driving behaviors are not effectively revealed in the attained equilibria. Meanwhile, a few studies try to actively model different kinds of uncertainties involved in a game process. In~\cite{schwarting2021stochastic}, uncertainties in the observation model are actively addressed by formulating the interaction between agents as a stochastic dynamic game in belief space. However, the proposed method only considers observation uncertainties with Gaussian noise. In~\cite{mehr2023maximum}, the maximum entropy model is utilized to model the uncertainties in the equilibrium strategies for all agents. Nevertheless, the objectives and intentions of all agents are assumed to be known and fixed, which is unrealistic in complicated urban traffic scenarios. 

To address the uncertainties residing in objectives and obtain the optimal strategy against the resulting multi-modal behaviors of rival agents, some of the researches study the game of incomplete information and the corresponding Bayesian equilibrium, where the equilibrium strategy is optimal over the distribution of rivals' possible objectives and their equilibrium strategies. Representative works include~\cite{deng2022lane,zhang2022human,shao2020discretionary,yao2022modelling,zhao2023non,li2024cooperative}. Nonetheless, these methods adopt overly simplified game settings of tiny scale and use hand-crafted solvers for obtaining the corresponding Bayesian equilibrium, which confines the application of these methods to simple scenarios like lane-switching. Besides, the decision-making and trajectory-planning processes are usually separated, and this discrepancy can lead to suboptimal decisions. A recent research~\cite{peters2024contingency} takes a game-theoretic perspective on contingency planning, and the resulting contingency game enables interaction with other agents by generating strategic motion plans conditioned on multiple possible intents for other actors. However, it is a motion planning framework lacking in the capacity to perform high-level decision-making. Also, the dedicated and unsymmetrical game settings make it hard to generalize to different types of games with multiple players, multiple policies, and multiple stages.

\subsection{No-Regret Learning}
Another branch of methods closely related to this research is no-regret learning, which aims at providing general solutions to various extensive-form games (games with multiple stages). Despite their effectiveness in modeling sequential decision-making processes, solving for large extensive games has been a long-standing challenge. In~\cite{zinkevich2007regret}, a counter-factual regret minimization (CFR) algorithm is introduced to iteratively solve this form of games, which is extended by MCCFR~\cite{lanctot2009monte} with a Monte Carlo sampling process to avoid the difficulties traversing the entire game tree in each iteration. A follow-up work is proposed in~\cite{lisy2015online}, which further extends the MCCFR method to enable dynamic construction of the game tree. The convergence of these methods to a corresponding Nash equilibrium is only guaranteed on two-player zero-sum games. To extend these methods to multi-player general-sum games, CFR-S is introduced in~\cite{celli2019learning}, which performs sampling at each iteration and illustrates that the frequency distribution of the sampled strategies converges to a Coarse Correlated Equilibrium (CCE). For games of incomplete information with Bayes settings, the condition of convergence to a Bayes-CCE is provided in~\cite{hartline2015no} without detailed solution schemes.

\section{Problem Statement}

Consider a set of $N$ vehicles $\mathcal{N}=\{1,2,...,N\}$ in a traffic scenario. For each vehicle $i\in\mathcal{N}$, we denote the set of its potential intentions as $T_i$. For each driving intention $t_i\in T_i$, the associated action space and utility function are denoted as $A_i(t_i)$ and $u_i(t_i)$, respectively. Further, we use $T=\times_{i\in\mathcal{N}}T_i$ to denote the product space of $\{T_i\}_{i\in\mathcal{N}}$, such that each $t\in T$ is a vector of intentions of all vehicles, $t=[t_i]_{i\in\mathcal{N}}$. It should be noted that the intention of vehicle $i$, $t_i$, is private and unknown to all other vehicles and vice versa. Instead, each vehicle is trying to maintain a probability distribution over the intentions of all other vehicles. We further assume that the historical trajectories of all vehicles, together with the surroundings, are fully observable by each vehicle. Since the predictions on driving intentions are determined by historical trajectories and surroundings, we can reasonably assume that all vehicles can arrive at the same belief on the intentions. Formally, the following assumption is introduced.

\begin{assumption}
A common probability distribution $p$ over $T$ can be established and is known to all vehicles, such that $p=[p^i]_{i\in\mathcal{N}}$, where $p_i$ is the common belief of all vehicles $j\in\mathcal{N}, j\neq i$ over the intention of vehicle $i$.
\end{assumption}

Note that although vehicle $i$ has complete knowledge of its own intention, it also recognizes the fact that its intention is not known to others and that the belief of other vehicles over its own intention is $p_i$.

With the aforementioned notations and assumptions, we formally introduce the following definition of games with incomplete information.

\begin{definition}(Harsanyi game of incomplete Information~\cite[Chapter~9.4]{maschler2020game})
A Harsanyi game of incomplete information is a vector $(\mathcal{N},(T_i)_{i\in\mathcal{N}},p,(g_t)_{t\in T})$, where $g_t=((A_i(t_i))_{i\in\mathcal{N}},(u_i(t_i))_{i\in\mathcal{N}})$ is the state game for the type vector $t$.
\end{definition}

Note that the driving intentions of vehicles are modeled as types of players in one-to-one correspondence. Through playing a Harsanyi game of incomplete information, a desirable result is the corresponding Bayesian Nash equilibrium strategy. In particular, a behavior strategy $\sigma_i$ of vehicle $i$, is a map from each type $t_i\in T_i$ to a probability distribution over the available actions: $\sigma_i(t_i)\in\Delta(A_i(t_i))$. Further, given the strategy vector of all vehicles as $\sigma=[\sigma_i]_{i\in\mathcal{N}}$, we denote the expected payoff of vehicle $i$ with type $t_i$ as 
\begin{equation}
U_i(\sigma|t_i) = \mathbb{E}_{t_{-i}\sim p_{-i}}
\mathbb{E}_{
a\sim\sigma(t)
}u_i(t_i;a).
\end{equation}
The corresponding Bayesian Nash equilibrium is defined as follows.
\begin{definition}(Bayesian Nash Equilibrium~\cite[Chapter~9.4]{maschler2020game})
a strategy vector $\sigma^*=[\sigma_i^*]_{i\in\mathcal{N}}$ is a (mixed-strategy) Bayesian Nash Equilibrium if for all $i\in\mathcal{N}$, for all $t_i\in T_i$, and for all $a_i\in A_i(t_i)$, the following inequality holds:
\begin{equation}
U_i(\sigma|t_i)\geq U_i((a_i,\sigma^*_i)|t_i).
\end{equation}
\end{definition}
When a Bayesian Nash equilibrium is reached, no player of any type can obtain a higher expected payoff by unilaterally modifying its own behavioral strategy. In the context of autonomous driving, it means that for an arbitrary vehicle $i$, once its intention is determined, its strategy is optimal over the expectation of the intentions of all other vehicles, given that their strategies are fixed. Since this optimality condition holds for all vehicles, the bilateral signaling effects are modeled, such that the influence of own actions on the belief of other vehicles about the intention of ego vehicle, and further, their preferences on actions, can be properly considered. Thus, an interaction model that handles multimodality on both sides is established.

However, solving for such a mixed-strategy Bayesian Nash Equilibrium is not trivial, especially for a multi-player general-sum game, which is typically the case in common urban traffic scenarios. To leverage the merit of game of incomplete information, we consider instead the Bayes-CCE, which is a generalization of the Bayesian Nash equilibrium, and propose a general solver that is proven to converge to a Bayes-CCE with minimum assumptions required. To introduce the concept of Bayes-CCE, we first define the plan $s_i$ of player $i$ as a direct map from type to a specific action: $s_i(t_i)\in A_i(t_i)$, and further, $s=[s_i]_{i\in\mathcal{N}}$ is the concatenated vector of plans of all players. $\Sigma$ is the set of $s$ such that $s\in\Sigma$. The formal definition of a Bayes-CCE is as follows.

\begin{figure*}[t]
\centering
\includegraphics[scale=0.165]{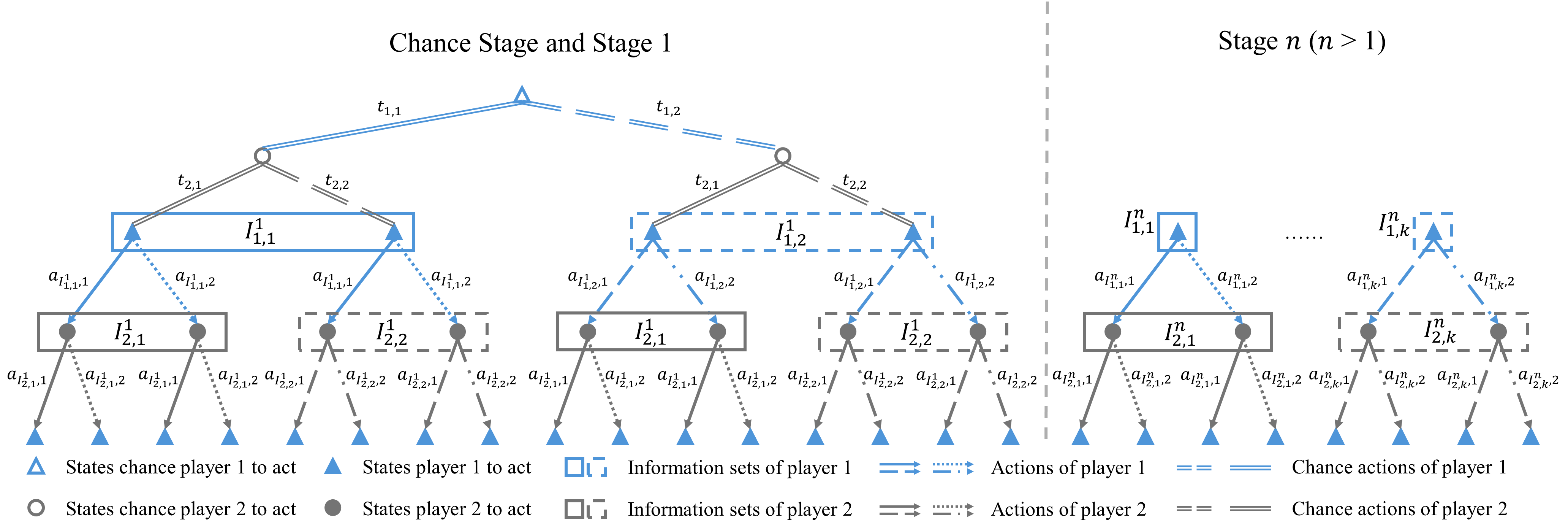}
\caption{Settings of the game of incomplete information for traffic interaction. The game is an extensive-form game containing multiple stages. In the beginning, chance players select types for all players in a chance stage (corresponding to the chance actions in the figure), where the selections are subject to distribution $p$. Then in each of the stages, all players perform a single action in turn. Eventually, the actions of players form a path leading from the root node to one of the leaf nodes, where the payoffs of all players are determined. The notation $I^n_{i,k}$ represents the $k^{th}$ information set of player $i$ at stage $n$, and $a_{I,m}$ denotes the $m^{th}$ action under information set $I$. Each of the information sets contains one or more states, such that the agent to play, under that information set, cannot distinguish between the game states contained in the information set. Therefore it will take the same action (or the same strategic profile of actions) for all the states contained in the information sets. Further, we assume that at the end of each stage except for the chance stage, all the states are distinguishable, and therefore each game state at the end of a stage constitutes an information set at the beginning of the next stage.}
\label{fig:game}
\end{figure*}

\begin{definition}(Bayes-CCE~\cite{hartline2015no})
A strategy profile $\psi\in\Delta(\Sigma)$ is a Bayes-CCE if for every player $i\in\mathcal{N}$, for every type $t_i\in T_i$, and for every possible plan $s'_i$, the following inequality holds:
\begin{equation}
\begin{aligned}
\mathbb{E}_{s\sim\psi}\mathbb{E}_{t_{-i}\sim p_{-i}}u_i(t_i;s(t_i,t_{-i}))&\geq\\
\mathbb{E}_{s\sim\psi}\mathbb{E}_{t_{-i}\sim p_{-i}}&u_i(t_i;s'_i(t_i),s_{-i}(t_{-i})).
\end{aligned}
\end{equation}
\end{definition}

In this paper, we formulate the driving game in urban traffic scenarios as a Harsanyi game of incomplete information. Specifically, the action spaces are represented as sets of sampled candidate trajectories, each associated with a particular driving intention (accelerating, lane switching, etc.) The utility functions include pertinent driving performance indices like smoothness and safety. Solving for the Bayes-CCE yields an open-loop equilibrium strategy, which is a joint probability distribution over candidate trajectories. We assume that all vehicles perform this equilibrium strategy. As time evolves and new information is gained, the prediction of vehicle intentions can be updated through Bayesian filtering, and the corresponding driving game can be resolved repeatedly, resulting in a closed-loop driving scheme.

\section{Proposed Approach}

\subsection{Parallel Solver for Bayes-CCE}
In this section, a parallel solver utilizing classic Monte Carlo counter-factual regret minimization (MCCFR) is introduced to resolve the game and enable a real-time solution for the open-loop Bayes-CCE in general traffic settings. In particular, the structure of a typical extensive-form game of incomplete information in traffic settings is shown as the game tree in Fig. \ref{fig:game}. Each of the nodes on the game tree corresponds to a particular game state, and these game states are partitioned into information sets, such that players cannot distinguish between two states in the same information set. Therefore, the behavior strategy of each player is made with respect to information sets, $\sigma_i(t_i)=[\sigma_i(t_i,I)]_{I\in\mathcal{I}_i}$, where $\mathcal{I}_i$ is the information partition of player $i$, or the set of information sets under which it is the turn of player $i$ to move. $\sigma_i(t_i,I)$ is a distribution over possible actions $A_i(t_i,I)$ for player $i$ of type $t_i$ under Information set $I$. At the beginning stage of the game, an extra chance player (also known as the nature) will take chance moves that will randomly select a type for each player subject to the common probability distribution $p$. After the types of all players are determined, players will take moves in turns to proceed through the game tree until a leaf node is reached and the utilities of all players are thus determined.
\begin{remark}
    We assume that at each stage, the moves of all players are taken simultaneously, which means that each player cannot distinguish between the moves taken by others before it takes its own move. This setting is clearly revealed by the structure of the information sets in the game tree.
\end{remark}
\begin{remark}
    Fig. \ref{fig:game} is an illustrative figure corresponding to a game with 2 players. A game with 3 or more players can be defined similarly. Meanwhile, the game tree description of a game is general enough such that other forms of games, such as a Stackelberg game, can also be well-represented by simply changing the configurations of the information sets. It will be shown that the proposed solving scheme does not rely on the particular structure of the game tree and therefore is also general.
\end{remark}
Before we formally propose the solving scheme, we give a brief introduction to two relevant algorithms as preliminaries.

\textbf{CFR}~\cite{zinkevich2007regret} and \textbf{MCCFR}~\cite{lanctot2009monte}: The vanilla CFR only considers a two-player zero-sum game without types of players. Therefore, $\sigma_i(t_i,I)$, $u_i(t_i;a)$, and $A_i(t_i,I)$ can be rewritten as $\sigma_i(I)$, $u_i(a)$, and $A_i(I)$, respectively. Given the overall strategy as $\sigma$, the counterfactual value $v_i(\sigma,I)$ for information set $I$ is defined as
\begin{equation}
    v_i(\sigma,I) = \sum_{z\in Z_I} \pi^\sigma_{-i}(z[I])\pi^\sigma(z[I],z)u_i(z)
\end{equation}
where $Z_I$ is the set of all the paths on the game tree from root to leaf that pass through I, and $z$ is one of the paths in $Z_I$ containing all historical actions of all players. $z[I]$ is the prefix of $z$ from root to $I$. $\pi^\sigma_{-i}(z[I])$ is the cumulative product of the probabilities of moves taken by other agents along $z[I]$, and $\pi^\sigma(z[I],z)$ is the cumulative product of the probabilities of moves taken by all agents from $z[I]$ to $z$. In general, the counterfactual value $v_i(\sigma,I)$ describes the expected utility of Information set $I$ given strategy profile $\sigma$ with the extra assumption that agent $i$ plays to reach $I$.

The immediate counterfactual regrets of each action $a\in A_i(I)$ is $r(\sigma,I,a)=v_i(\sigma_{(I\rightarrow a)},I)-v_i(\sigma,I)$, with $\sigma_{(I\rightarrow a)}$ define as the strategy profile identical to $\sigma$ except that $\sigma_{(a\rightarrow I)}(a|I)=1$, namely that action $a$ is selected with probability $1$. In the entire no-regret learning process, rounds of playing are conducted. In each round, the regret-matching strategy is conducted to improve the strategy. Particularly in round $k+1$, the strategy $\sigma^{k+1}$ is determined by
    $\sigma^{k+1}_i(a|I)=R^{k,+}_i(I,a)/\sum_{a'\in A_i(I)}R^{k,+}_i(I,a')$,
where $R^{k,+}_i(I,a)=\max\{\sum_{\tau=1}^k r(\sigma^\tau,I,a),0\}$. It is proven that when $k\rightarrow \infty$, the average strategy $\overline{\sigma}^k$ defined as $\overline{\sigma}^k_i(a|I)=\sum_{\tau=1}^k \pi^{\sigma^\tau}_i(I)\sigma^\tau(a|I)/\sum_{\tau=1}^k \pi^{\sigma^\tau}_i(I)$ converges to a Nash equilibrium for a two-player zero-sum game.

A drawback of the vanilla CFR method is that in each round of playing, traversing the entire game tree is required, which is time-consuming and even intractable when the size of the game is large. To cope with this problem, the method MCCFR is proposed, such that in each round of playing, only part of the game tree is explored. Particularly in outcome-sampling MCCFR, only one leaf node of the game tree is sampled according to some pre-defined distribution, $z\sim q(z)$, and all other leaves are dropped (equivalently, their utilities are zeroed out). The following sampled counterfactual value is introduced:
\begin{equation}
\Tilde{v}_i(\sigma,I;z)=
\left\{
\begin{array}{lll}
\frac{1}{q(z)}\pi^\sigma_{-i}(z[I])\pi^\sigma(z[I],z)u_i(z)&\ &\textup{if}\ z\in Z_I\\
0&\ & \textup{else.}
\end{array}
\right.
\end{equation}
It is shown that $\mathbb{E}_{z\sim p(z)}\Tilde{v}_i(\sigma,I;z)=v_i(\sigma,I)$. Replacing $v_i(\sigma,I)$ by $\Tilde{v}_i(\sigma,I;z)$ yields the outcome-sampling MCCFR. It is obvious that in each round of playing, the only information sets that require updating are those passed through by the sampled path $z$.

The previous restatement of CFR and MCCFR describes a solution to obtain the (mixed-strategy) Nash equilibrium in a two-player zero-sum extensive-form game. We integrate MCCFR, the no-regret learning method in game of incomplete information~\cite{hartline2015no}, and CFR-S~\cite{celli2019learning} to introduce MCCFR-S for game of incomplete information, a parallelizable solving scheme that is proven to converge to a Bayes-CCE for a multi-player general-sum game of incomplete information settings. Details are shown in Algorithm 1. We give a brief explanation of each procedure used in this algorithm.

\textbf{Bayes-CCE}$(G,M,p)$: this is the main procedure of Algorithm 1. $G$ is the game described in Fig.\ref{fig:game}. $M$ is the number of iterations and $p$ is the common distribution over intentions of vehicles. $re$, $\sigma$, and $\phi$ are the array of cumulative counterfactual regrets, the array of behavioral strategies, and the frequency of plan $s$, respectively. $va$ is the array of the cumulative sampled counterfactual values. All those arrays are zero-initialized. In each iteraction, a type vector $t$ representing vehicles' intentions are sampled from distribution $p$, and then the outcome-sampling MCCFR is performed once to update the regret array $re$ and the strategy array $\sigma$. A sampling process is then performed based on updated $\sigma$ strategy such that a plan $s$ is sampled from $\sigma$ and the frequency $\phi$ is updated.

\textbf{MCCFR}$(G,re,\sigma,t,va,z,pr,l)$: this is a recursive implementation of MCCFR~\cite{lanctot2009monte}, where $G,re,\sigma,t$ are defined as above. $z$ is the history of play or the past actions of all players concatenated together. For a particular $z$, $pr$ is the vector containing probabilities of all players play to reach $z$ according to current strategy $\sigma$, namely $\{\pi^\sigma_i(z)\}_{i\in\mathcal{N}}$. $l$ is the probability of sampling $z$, namely $q(z)$. The algorithm contains a forward pass and a backward pass. In the forward pass we randomly proceed through the game tree from the root to one of the leaf by sampling action from the $\epsilon$-modulated strategy at each node (line 17 and 18) and calculate the $pr$ and $l$ for current $z$ (line 19 and 20). Once a leaf is reached, we obtain the utility for all players (line 11 to 14). In the backward pass, we revisit all the information sets we visited in the forward pass in reverse order. For an information set $I$, we obtain the probability playing from $I$ to the sampled leaf according to current strategy (line 23), and updates the cumulative counterfactual regret for $I$ (line 24 to 31). A standard regret matching is performed to update the behavioral strategy associated with $I$ (line 32).

\textbf{RM}$(re,i,t_i,I)$: this is the standard regret matching algorithm for updating behavioral strategy associated with information set $I$ called by the MCCFR procedure.

\textbf{Sample}$(\sigma,\phi)$: this is the sampling procedure. given updated strategy $\sigma$, a plan $s_i$ is sampled for each $i\in\mathcal{N}$. The concatenated plan $s\in\Sigma$ is then obtained and record. $\phi$ measures the time when $s$ is sampled, such that $\phi/M$ is an empirical distribution over the space of plan $\Sigma$.

\begin{algorithm}
\caption{MCCFR-S for Game of Incomplete Information}\label{alg:alg1}
\begin{algorithmic}[1]

\Procedure{Bayes-CCE}{$G,M,p$}
\State \textbf{Init} $re,\sigma,\phi,va$
\For{$1$ to $M$}
    \State $t\sim p$
    \State $\textup{MCCFR}(re,\sigma,t,va,\{\},[1]_{i\in\mathcal{N}},1)$
    \State $\textup{Sample}(\sigma,\phi)$
\EndFor
\State \textbf{return} $\phi,va$
\EndProcedure
\Statex
\Procedure{MCCFR}{$G,re,\sigma,t,va,z,pr,l$}
\If {$z\in\mathcal{Z}$}

    \State $\{u_i\}_{i\in\mathcal{N}}\gets \textup{getUtility}(G,z)$
    \State \textbf{return} $(1,l,\{u_i\}_{i\in\mathcal{N}})$
\EndIf
\State $i\gets\textup{getPlayer}(G,z)$
\State $I\gets\textup{getInfoset}(G,z,i)$
\State $a\sim(1-\epsilon)\sigma_i(t_i,I)+\epsilon/|A_i(t_i,I)|$
\State $z\gets z+a$
\State $l\gets l*((1-\epsilon)\sigma_i(a|t_i,I)+\epsilon/|A_i(t_i,I)|)$
\State $pr[i]\gets pr[i]*\sigma_i(a|t_i,I)$
\State $x,l,\{u_i\}_{i\in\mathcal{N}}\gets \textup{MCCFR}(G,re,\sigma,t,va,z,pr,l)$
\State $c\gets x$
\State $x\gets x*\sigma_i(a|t_i,I)$
\State $w\gets u_i/l*\prod_{j\neq i}pr[j]$
\For{$a'\in A_i(t_i,I)$}
    \If{$a'=a$}
        \State $re[i,t_i,I,a']\gets re[i,t_i,I,a'] + (c-x)*w$
    \Else
        \State $re[i,t_i,I,a']\gets re[i,t_i,I,a'] + (-x)*w$
    \EndIf
\EndFor
\State $\sigma[i,t_i,I] \gets \textup{RM}(re,i,t_i,I)$
\State $va[i,t_i,I] \gets va[i,t_i,I]+x*w$
\State \textbf{return} $(x,l,u)$
\EndProcedure
\Statex
\Procedure{RM}{$re,i,t_i,I$}
\State $total\gets 0$
\For{$a\in A_i(t_i,I)$}
    \State $total\gets total+ \max(re[i,t_i,I,a],0)$
\EndFor
\For{$a\in A_i(t_i,I)$}
    \State $\sigma_{i,t_i,I}[a] \gets \max(re[i,t_i,I,a],0)/total$
\EndFor
\State \textbf{return} $\sigma_{i,t_i,I}$
\EndProcedure
\Statex
\Procedure{Sample}{$\sigma,\phi$}
\State \textbf{Init} $\{s_i\}_{i\in\mathcal{N}}$
\For{$i\in\mathcal{N},t_i\in T_i,I\in \mathcal{I}_i(t_i)$}
    \State $a\sim\sigma[i,t_i,I]$
    \State $s_i(t_i,I)\gets a$
\EndFor
\State $s\gets \{s_i\}_{i\in\mathcal{N}}$
\State $\phi[s]\gets \phi[s]+1$
\EndProcedure
\end{algorithmic}
\label{alg1}
\end{algorithm}

Algorithm 1 is easily parallelizable and implemented with multiple processes, as each process samples the type vector $t$, proceeds along the game tree, and samples the overall plan $s$, independently. The only coupled part is the line 32, as each process needs to gather the $re$ computed by all other processes to perform the regret matching and update the strategy. For Algorithm 1, we have the following theorem.
\begin{theorem}
The empirical distribution $\phi/M$ obtained with Algorithm 1 converges almost surely to a Bayes-CCE with common prior $p$ when $M\rightarrow \infty$.
\end{theorem}
\begin{proof}
For a particular type of vehicle $i$, $t_i\in T_i$, let $\mathcal{M}_{t_i}$ denote the set of iteration number such that $t^\tau_i=t_i$ when $\tau \in \mathcal{M}_{t_i}$. Define $r^\tau_i(\sigma_i,\sigma^\tau)=u_i(t_i^\tau;\sigma_i(t_i^\tau),\sigma^\tau_{-i}(t^\tau_{-i}))-u_i(t_i^\tau;\sigma^\tau(t^\tau))$, which is the regret value that the player $i$ experience at iteration $\tau$ when its strategy is replaced by $\sigma_i$ from the overall strategy $\sigma^\tau$. Furthermore, the cumulative regret value for player $i$ with type $t_i$, supposed that player $i$ adopts strategy $\sigma_i$ all the time, is defined as
\begin{equation}
    R_{t_i}(\sigma_i)=\frac{1}{|\mathcal{M}_{t_i}|}\sum_{\tau\in\mathcal{M}_{t_i}}r^\tau_i(\sigma_i,\sigma^\tau)).
\end{equation}
Since each player of each type drawn from line 4 of Algorithm 1 maintains a separate MCCFR regret minimizer, it follows directly from \cite[Theorem 5]{lanctot2009monte} that for all possible $\sigma_i$, we have $R_{t_i}(\sigma_i)\leq(1+\frac{\sqrt{2}}{\sqrt{p}})(\frac{1}{\delta})G_{t_i}/\sqrt{|\mathcal{M}_{t_i}|}$ with probability $1-p$ for any $p\in(0,1]$, where $G_{t_i}$ is a game constant and $\delta$ is the smallest probability of a leaf being sampled in the forward pass in the MCCFR process. It is obvious that when $M\rightarrow \infty$, it is also surely that $R_{t_i}(\sigma_i)\leq 0$ if $\delta >0$. In particular, a plan $s_i$ is also a strategy, and therefore this conclusion naturally holds for all $R_{t_i}(s_i)$, namely
\begin{equation}
\label{neq1}
    R_{t_i}(s_i)\leq\left(1+\frac{\sqrt{2}}{\sqrt{p}}\right)\left(\frac{1}{\delta}\right)G_{t_i}/\sqrt{|\mathcal{M}_{t_i}|}
\end{equation}
holds for all $s_i$ with probability $1-p$.

Next, at each iteration $\tau$, we sample a plan $s^\tau$ from the strategy $\sigma^\tau$, and we define the sampled cumulative regret as
\begin{equation}
    \hat{R}_{t_i}(s_i)=\frac{1}{|\mathcal{M}_{t_i}|}\sum_{\tau\in\mathcal{M}_{t_i}}r^\tau_i(s_i,s^\tau).
\end{equation}
Since $s^\tau\sim\sigma^\tau$, the following equations hold naturally:
\begin{equation}
\label{s_sigma}
\begin{aligned}
    &\mathbb{E}_{s^\tau\sim\sigma^\tau}u_i(t_i^\tau;s^\tau(t^\tau)) = u_i(t^\tau_i;\sigma^\tau(t^\tau)),\\
    &\mathbb{E}_{s^\tau\sim\sigma^\tau}u_i(t_i^\tau;s_i(t^\tau_i),s_{-i}(t^\tau_{-i})) = u_i(t^\tau_i;s_i(t^\tau_i),\sigma_{-i}^\tau(t^\tau_{-i})).
\end{aligned}
\end{equation}
Therefore, the following conclusion is directly obtained:
\begin{equation}
\label{mean}
    \mathbb{E}_{\{s^\tau\sim\sigma^\tau\}_{\tau\in\mathcal{M}_{t_i}}}[\hat{R}_{t_i}(s_i)-R_{t_i}(s_i)]=0.
\end{equation}
For simplicity we rewrite $\mathbb{E}_{\{s^\tau\sim\sigma^\tau\}_{\tau\in\mathcal{M}_{t_i}}}$ as $\mathbb{E}_{\mathcal{M}_{t_i}}$, $\mathbb{E}_{s^\tau\sim\sigma^\tau}$ as $\mathbb{E}_\tau$, $r^\tau_i(s_i,s^\tau)-r^\tau_i(s_i,\sigma^\tau))$ as $\zeta_\tau$, and we inspect the following expectation:
\begin{equation}
\begin{aligned}
\label{Var}
    &\mathbb{E}_{\mathcal{M}_{t_i}}\left[(\hat{R}_{t_i}(s_i)-R_{t_i}(s_i))^2\right]\\
    &=\frac{1}{|\mathcal{M}_{t_i}|^2}\mathbb{E}_{\mathcal{M}_{t_i}}\left[\left(\sum_{\tau\in\mathcal{M}_{t_i}}\zeta_\tau\right)^2\right]\\
    &=\frac{1}{|\mathcal{M}_{t_i}|^2}\left(\sum_{\tau\in\mathcal{M}_{t_i}}\mathbb{E}_\tau[\zeta_\tau^2]+2\sum_{\tau\in\mathcal{M}_{t_i}}\sum_{\tau'\in\mathcal{M}_{t_i},\tau'\neq\tau}\mathbb{E}_\tau[\zeta_\tau\zeta_{\tau'}]\right)\\
    &=\frac{1}{|\mathcal{M}_{t_i}|^2}\left(\sum_{\tau\in\mathcal{M}_{t_i}}\mathbb{E}_\tau[\zeta_\tau^2]\right)\leq\frac{\Delta u_{t_i}^2}{|\mathcal{M}_{t_i}|}.
\end{aligned}
\end{equation}
The last two steps follow from the facts that $\zeta_\tau$ and $\zeta_{\tau'}$ are independent with zero expectations, and that $\mathbb{E}_\tau[\zeta_\tau^2]\leq\Delta u_{t_i}^2$ where $\Delta u_{t_i}^2$ is the range of utility of player $i$ with type $t_i$.

Combining (\ref{mean}) and (\ref{Var}), and following \cite[Lemma 1]{lanctot2009monte}, we obtain that
\begin{equation}
\label{neq2}
\hat{R}_{t_i}(s_i)\leq R_{t_i}(s_i)+\frac{\Delta u_{t_i}}{\sqrt{p'}\sqrt{|\mathcal{M}_{t_i}|}}
\end{equation}
with probability $1-p'$.
Combining (\ref{neq1}) and (\ref{neq2}), we have
\begin{equation}
\label{neq3}
    \hat{R}_{t_i}(s_i)\leq \frac{1}{\sqrt{|\mathcal{M}_{t_i}|}}\left(\left(1+\frac{\sqrt{2}}{\sqrt{p}}\right)\left(\frac{G_{t_i}}{\delta}\right)+\frac{\Delta u_{t_i}}{\sqrt{p'}}\right)
\end{equation}
with probability $(1-p)(1-p')$. As a result, it is almost sure that for any $s_i$, we have $\sum_{\tau\in\mathcal{M}_{t_i}}r^\tau_i(s_i,s^\tau)\leq o(|\mathcal{M}_{t_i}|)$ for all $i$, all $t_i$, and all $s_i$. Following \cite[Lemma 10]{hartline2015no}, we conclude that the empirical distribution $\phi/M$ converges almost surely to a Bayes-CCE.
\end{proof}

Sampling an entire plan $s_i$ requires traversing the entire game tree and sampling action for each information set $I\in\mathcal{I}_i$, which is computationally heavy for extensive-form games with multiple stages, as the number of information sets grows exponentially with respect to the number of stages. In the proposed framework, the receding horizon control is adopted, as each agent only performs the action computed in the first stage (details will be discussed in Section IV-D). Therefore, we only need to record the actions taken by players at the first stage, resulting in a partial plan. The resulting distribution can be viewed as a marginal distribution of the origin Bayes-CCE.

\subsection{Integrated Decision Selection and Trajectory Planning}

In a typical game of incomplete information, the type of each player is usually drawn by nature, a chance player, from the prior distribution. In the settings of autonomous driving, however, the types are self-determined by players, and therefore we can select the best type to perform. In the proposed framework, we view the selection of the best type $t^*_i$ from the type list $T_i$ as equivalent to the decision-making process. Formally given the computed Bayes-CCE $\psi$, we denote the expected utility of a type $t'_i$ as
\begin{equation}
    V(t'_i) = \mathbb{E}_{s\sim\psi}\mathbb{E}_{t_{-i}\sim p_{-i}}u_i(t'_i;s(t'_i,t_{-i})).
\end{equation}
Then the best type $t^*_i$ and equivalently the best decision, is determined by
$t^*_i = \argmax_{t'_i\in T_i}V(t'_i).$
However, as we mentioned in Section IV-A, we only keep a partial record of the Bayes-CCE $\psi$ for computation efficiency, and therefore it is not feasible to directly calculate $V(t'_i)$. Instead, we introduce an estimator of this expectation and use this estimator as a criterion of decision-making. We denote the root information set corresponding to $t_i$ as $I_{0,t_i}$, and the estimator is given as $\hat{V}(t'_i)=va[i,t'_i,I_{0,t'_i}]/|\mathcal{M}_{t'_i}|$. Immediately, we have the following theorem.
\begin{theorem}
    Suppose that the empirical distribution $\phi/M$ converges to a Bayes-CCE $\psi$ when $M\rightarrow\infty$, then $\hat{V}(t'_i)\rightarrow V(t'_i)$ when $M\rightarrow\infty$.
\end{theorem}
\begin{proof}
For $V(t'_i)$ we have
\begin{equation}
    \begin{aligned}
    V(t'_i) &= \mathbb{E}_{s\sim\psi}\mathbb{E}_{t_{-i}\sim p_{-i}}u_i(t'_i;s(t'_i,t_{-i}))\\
    &=\frac{1}{p_i(t'_i)}\mathbb{E}_{s\sim\psi}p_i(t'_i)\mathbb{E}_{t_{-i}\sim p_{-i}}u_i(t'_i;s(t'_i,t_{-i}))\\
    &=\frac{1}{p_i(t'_i)}\mathbb{E}_{s\sim\psi}\mathbb{E}_{t_{i}\sim p_{i}}\textbf{1}\{t_i=t'_i\}\mathbb{E}_{t_{-i}\sim p_{-i}}u_i(t_i;s(t_i,t_{-i}))\\
    &=\frac{1}{p_i(t'_i)}\mathbb{E}_{s\sim\psi}\mathbb{E}_{t\sim p}\textbf{1}\{t_i=t'_i\}u_i(t_i;s(t))
    \end{aligned}
\end{equation}
where $\textbf{1}\{\cdot\}$ is a binary function such that it equals to $1$ if the condition holds and $0$ otherwise. $p_i(t'_i)$ is the probability of $t'_i$. For $\hat{V}(t'_i)$, we examine its expectation
\begin{equation}
\begin{aligned}
    \mathbb{E}[\hat{V}(t'_i)]&=\frac{1}{|\mathcal{M}_{t'_i}|}\sum_{\tau\in\mathcal{M}_{t'_i}}\mathbb{E}_{z\sim q(z)}\frac{1}{q(z)}\pi^{\sigma^\tau(t^\tau)}(z)u_i(t'_i;z)\\
    &=\frac{1}{|\mathcal{M}_{t'_i}|}\sum_{\tau\in\mathcal{M}_{t'_i}}u_i(t'_i;\sigma^\tau(t^\tau))
\end{aligned}
\end{equation}
which follows from ~\cite[Lemma 1]{lanctot2009monte}. Plug in (\ref{s_sigma}) and we have
\begin{equation}
\begin{aligned}
&\mathbb{E}[\hat{V}(t'_i)]=\frac{M}{|\mathcal{M}_{t'_i}|}\frac{1}{M}\sum_{\tau\in M}\mathbb{E}_{s^\tau\sim\sigma^\tau}\textbf{1}\{t^\tau_i=t'_i\}u_i(t_i;s^\tau(t^\tau))\\
&=\frac{M}{|\mathcal{M}_{t'_i}|}\frac{1}{M}\sum_{\tau\in M}\lim_{Q\rightarrow\infty}\frac{1}{Q}\sum_{k\in Q}\textbf{1}\{t^\tau_i=t'_i\}u_i(t_i;s^{\tau,k}(t^\tau))\\
&=\frac{M}{|\mathcal{M}_{t'_i}|}\lim_{Q\rightarrow\infty}\frac{1}{Q}\sum_{k\in Q}\frac{1}{M}\sum_{\tau\in M}\textbf{1}\{t^\tau_i=t'_i\}u_i(t_i;s^{\tau,k}(t^\tau))
\end{aligned}
\end{equation}
where $s^{\tau,k}$ is sampled from $\sigma^\tau$. The above equation follows from the law of large numbers. Obviously, when $M\rightarrow \infty$, we have
\begin{equation}
\begin{aligned}
    \frac{1}{M}\sum_{\tau\in M}\textbf{1}\{t^\tau_i=t'_i\}u_i(t_i;s^{\tau,k}(t^\tau))\rightarrow\\
    \mathbb{E}_{s\sim\psi}\mathbb{E}_{t\sim p}\textbf{1}\{t_i=t'_i\}u_i(t_i;s(t))
\end{aligned}
\end{equation}
as $\{s^{\tau,k}\}_{\tau\in M}$ can be viewed as samples from the Bayes-CCE $\psi$, and $\{t^\tau\}_{\tau\in M}$ are samples from $p$. Again, the law of large numbers is applied. Therefore,
\begin{equation}
\begin{aligned}
&\lim_{M\rightarrow\infty}\mathbb{E}[\hat{V}(t'_i)]\\
&=\lim_{M\rightarrow\infty}\frac{M}{|\mathcal{M}_{t'_i}|}\lim_{Q\rightarrow\infty}\frac{1}{Q}\sum_{k\in Q}\mathbb{E}_{s\sim\psi}\mathbb{E}_{t\sim p}\textbf{1}\{t_i=t'_i\}u_i(t_i;s(t))\\
&=\lim_{M\rightarrow\infty}\frac{M}{|\mathcal{M}_{t'_i}|}\mathbb{E}_{s\sim\psi}\mathbb{E}_{t\sim p}\textbf{1}\{t_i=t'_i\}u_i(t_i;s(t))\\
&=\frac{1}{p_i(t'_i)}\mathbb{E}_{s\sim\psi}\mathbb{E}_{t\sim p}\textbf{1}\{t_i=t'_i\}u_i(t_i;s(t))=V(t'_i).
\end{aligned}
\end{equation}
Noted that each term $\hat{V}^\tau(t'_i)=\pi^{\sigma^\tau(t^\tau)}/q(z)u_i(t'_i;z)$, is independent and with bounded covariance, following the Kolmogorov’s strong law of large numbers ~\cite[Theorem 2.3.10]{sen2017large}, we conclude that $\hat{V}(t'_i)\rightarrow V(t'_i)$ when $M\rightarrow\infty$.
\end{proof}

Theorem 2 indicates that $\hat{V}(t'_i)$ is an unbiased and consistent estimator of $V(t'_i)$. Moreover, it can be seen from Algorithm 1 that calculating $\hat{V}(t'_i)$, which corresponds to line 33 of the algorithm, adds only minimal computation load to the algorithm as no extra sampling is required. Therefore, we propose the decision-making strategy as 
\begin{equation}
\label{decision}
t^*_i=\argmax_{t'_i\in T_i}\hat{V}(t'_i)
\end{equation}
such that by selecting its type as $t^*_i$, the player $i$ makes the decision that maximizes its overall expected utility corresponding to the given Bayes-CCE $\psi$.
\begin{remark}
    The decision-making strategy presented does not distinguish between ego vehicle and surrounding vehicles. However, we should not assume that surrounding vehicles will adopt the same strategy since their decisions are subject to many factors (navigation, driving preference, bounded rationality, etc.) and are usually sub-optimal. Potential danger may occur if we assume that surrounding vehicles act optimally. Therefore, (\ref{decision}) is only used to make decisions for the ego vehicle but not to predict the driving purposes of surrounding vehicles.
\end{remark}

Once the type $t^*_i$ is determined, we can select the optimal plan $s^*_i$ from the Bayes-CCE $\psi$, such that the trajectory is given by $s^*_i(t^*_i)$. We introduce two different schemes to perform the Bayes-CCE $\psi$ and select the optimal plan.

\textbf{Accurate implementation of Bayes-CCE $\psi$}: since a Bayes-CCE is correlated, the accurate performance of a Bayes-CCE will require either negotiation beforehand (which is unlikely in urban traffic scenarios) or a leader-follower assumption which is often adopted in the Stackelberg game formulation of traffic interactions ~\cite{hang2021decision,fisac2019hierarchical}. Without loss of generality, we assume that player 1 is the leader, and player $i$ is a direct follower of player $i-1$ for all $i\in\mathcal{N}$ and $i\neq 1$. The joint probability distribution $\psi$ can be decomposed into a series of marginal distributions and conditional distributions, such that the probability of a joint plan $s=\{s_i\}_{i\in\mathcal{N}}$ is given by
\begin{equation}
\psi(s)=\psi_1(s_1)\prod_{i\in\mathcal{N},i\neq 1}\psi'_i(s_i|\{s_j\}_{j<i})
\end{equation}
such that $\psi_1$ is the marginal distribution of the plan corresponding to player $1$, and $\psi'_i$ is the conditional distribution of the plans of player $i$ conditioned on the plan of all the players before $i$. If plans are selected by maximum likelihood, the plans selected are
\begin{equation}
    s^*_1 = \argmax_{s_1}\psi_1(s_1), s^*_i=\argmax_{s_i}\psi'_i(s_i|\{s^*_j\}_{j\leq i}).
\end{equation}

We consider this scheme to be overly idealistic due to the following reasons:
\begin{itemize}
    \item Although being widely adopted, the leader-follower assumption is not always valid, especially when more than two vehicles are involved.
    \item Due to partial observation, vehicle $i$ cannot know for sure the plan $s_j$ adopted by vehicle $j$. At best, it can only maintain a belief over $s_j$.
\end{itemize}

Due to these reasons, we introduce the following marginal implementation of the Bayes-CCE.

\textbf{Marginal implementation of Bayes-CCE $\psi$}: to avoid the problem mentioned before, we propose an approximation performance scheme of the Bayes-CCE $\psi$. We define $\psi_i$ as the marginal distribution of the plan of player $i$, and then the plan is determined as
\begin{equation}
    s^*_i = \argmax_{s_i}\psi_i(s_i).
\end{equation}

The marginal implementation of Bayes-CCE is used in the simulation. Once the optimal plan $s^*_i$ is determined by one of the schemes introduced before, the planning process is conducted by taking $s^*_i(t^*_i)$ with the optimal decision $t^*_i$. As a result, an integrated decision-making and trajectory planning scheme is introduced based on the Bayes-CCE.

\subsection{Bayesian Intention Update}

% \begin{figure*}[t]
% \centering
% \includegraphics[scale=0.275]{system.png}
% \caption{System overview.}
% \label{fig:system}
% \end{figure*}

Following a Bayes-CCE, the probabilities of actions taken by vehicles depend on their underlying driving intentions. Therefore, the actual actions performed by vehicles provide extra information gains over the belief of driving intentions. In this section, we introduce strategies to update the common belief over the driving intentions of all vehicles based on the observed actions. Corresponding to different implementations of the Bayes-CCE, different updating strategies are introduced.

\textbf{Update strategy for the accurate implementation}: suppose that at time stamp $\lambda$ the common prior on intentions of vehicles is given as $p^{\lambda}$, the Bayes-CCE is given as $\psi^\lambda$, and the action performed by all vehicles is $a^\lambda=\{a_i^\lambda\}_{i\in\mathcal{N}}$. The following equation gives the updated common prior $p^{\lambda+1}$
\begin{equation}
    p^{\lambda+1}(t)=\frac{f(a^\lambda,t;\psi^\lambda)p^\lambda(t)}{\sum_{t'\in T}f(a^\lambda,t';\psi^\lambda)p^\lambda(t')}
\end{equation}
which follows the standard Bayes' rule of posterior probability. $f$ is the observation model such that $f(a^\lambda,t;\psi^\lambda)$ gives the probability that action $a^\lambda$ is taken if all players follow the Bayes-CCE $\psi^\lambda$ and their underlying types are defined by $t$. To effectively model $f$, we adopt the Mixture-of-Gaussians distributions, and $f$ is defined as
\begin{equation}
    f(a^\lambda,t;\psi^\lambda)=\sum_{s}\eta(\mu(a^\lambda);s(t))\psi^\lambda(s(t))
\end{equation}
where $\eta(\cdot;s(t))$ is the probability density function of a Gaussian distribution $N(\mu(s(t)),\Sigma)$. $\mu(a^\lambda)$ and $\mu(s(t))$ are the end states of vehicles when action $a^\lambda$ and $s(t)$ are taken, respectively. $\Sigma$ is a predefined covariance matrix.

\textbf{Update strategy for the marginal implementation}: under this implementation, the update rule is similarly given as
\begin{equation}
p_i^{\lambda+1}(t_i)=\frac{f_i(a^\lambda_i,t_i;\psi^\lambda_i)p_i^\lambda(t_i)}{\sum_{t'_i\in T_i}f_i(a^\lambda_i,t'_i;\psi^\lambda_i)p_i^\lambda(t'_i)}
\end{equation}
where
\begin{equation}
f_i(a^\lambda_i,t_i;\psi^\lambda_i)=\sum_{s_i}\eta(\mu(a^\lambda_i);s_i(t_i))\psi^\lambda_i(s_i(t_i)).
\end{equation}
This update strategy is performed individually for all $i\in\mathcal{N}$.

\begin{remark}
    In the updating process, the belief of intentions is updated for both the surrounding vehicles and also the ego vehicle. Following Assumption 1, the updated belief is still common among all the vehicles involved in the game. 
\end{remark}

\subsection{System Overview and Details of Implementation}

Combining all aforementioned modules, we introduce Algorithm 2, which is an integrated decision-making and trajectory planning framework for autonomous driving. Detailed explanations are given as follow: 

\begin{algorithm}
\caption{Integrated Decision Selection and Trajectory Planning for Autonomous Vehicle}\label{alg:alg2}
\begin{algorithmic}[1]
\State \textbf{Init} $p,M$
\State $j \leftarrow \textup{getOwnID}()$
\While{True}
\For{$i\in N$}
\State $x_i\leftarrow \textup{getPosition}(i)$
\For{$t_i\in T_i$}
\State $Tree(i,t_i)\leftarrow \textup{getTrajTree}(x_i,i,t_i)$
\EndFor
\EndFor
\State $G\leftarrow \textup{getGame}(Tree,\{x_i\}_{i\in N})$
\State $\phi,va\leftarrow \textup{Bayes-CCE}(G,M,p)$
\State $\psi\leftarrow\phi/M$
\For{$t_j\in T_j$}
\State $\hat{V}(t_j)\leftarrow va[j,t_j,I_{0,t_j}]/p_j(t_j)$
\EndFor
\State $t^*_j\leftarrow\argmax_{t_j\in T_j}\hat{V}(t_j)$
\State $s^*_j\leftarrow\argmax_{s_j}\psi_j(s_j)$
\State $a_j\leftarrow s^*_j(t^*_j)$
\For{$i\in N$}
\State $a_i\leftarrow \textup{getObservation(i)}$
\EndFor
\State $p\leftarrow \textup{getUpdate}(p,\psi,\{a_i\}_{i\in N})$
\EndWhile
\end{algorithmic}
\label{alg2}
\end{algorithm}

% The system structure is shown in Fig. \ref{fig:system}.
\begin{enumerate}
    \item A two-stage trajectory tree is constructed for each driving intention of each vehicle, which constitutes the action space (line 4 to 9). The methodology used to construct those trajectory trees is the same as the one used in \cite{mustafa2024racp}. 
    \item Given the sets of trajectory trees and the prior distribution, a game of incomplete information is established, and the solver introduced in Section IV-A is invoked to solve for the Bayes-CCE (line 10 to 12).
    \item Given the Bayes-CCE, the optimal decision is selected following the strategy described in equation (\ref{decision}) (line 13 to 16).
    \item One of the two implementation strategies of Bayes-CCE described in Section IV-B is applied to select the optimal plan. Together with the optimal decision selected in the previous step, the trajectory is determined and executed (line 17 to 18).
    \item Given the Bayes-CCE and the new trajectories, the corresponding update strategy is invoked to update the driving intentions of all vehicles. Thus, a new common belief over the driving intentions is established, and the algorithm goes back to Step $1)$ for a new round of playing (line 19 to 22).
\end{enumerate}

The prior belief over the driving intentions is initialized with uniform distribution. Note that data-driven prediction methods can also be applied to perform the initialization and obtain a better estimation of the driving intentions based on historical trajectories.

Considering important performance indices such as passenger comfort and safety, the utility is defined as
\begin{equation}
    u_i(t_i;z)= -J^i_c(z) - J^i_s(z) - J^i_p(z) - J^i_\textup{ref}(z).
\end{equation}
The trajectories of all vehicles are determined by the action history $z$. To compute for the above performance indices, we sample along the trajectories of all vehicles with equal time spacing $h\in[0,1,..., H]$. The comfort indice $J^i_c$ is given by
\begin{equation}
\begin{aligned}
    J^i_c(z)=&\sum_{h=0}^H\omega_{a,lat}|a^h_{i,lat}(z)|^2+\omega_{j,lat}|j^h_{i,lat}(z)|^2\\
    &+\omega_{a,long}|a^h_{i,long}(z)|^2+\omega_{j,long}|j^h_{i,long}(z)|^2
\end{aligned}
\end{equation}
where $a^h_{i,lat}$, $a^h_{i,long}$, $j^h_{i,lat}$, and $j^h_{i,long}$ are the lateral acceleration, the longitudinal acceleration, the lateral jerk, and the longitudinal jerk of vehicle $i$ at $h$, respectively. $\omega_{a,lat}$, $\omega_{a,long}$, $\omega_{j,lat}$, and $\omega_{j,long}$ are the corresponding weighting coefficients. The progress indice $J^i_p$ is given by
\begin{equation}
    J^i_p(z)=\omega_p\sum_{h=0}^H\min\{v^h_{i,long}(z)-v_{slow},0\}^2
\end{equation}
such that longitudinal velocity $v^h_{i,long}$ smaller than $v_{slow}$ will be penalized. $\omega_p$ is the weighting coefficient. Indice $J^i_\textup{ref}$ is given by
\begin{equation}
    J^i_\textup{ref}(z) = \omega_\textup{ref}\sum_{h=0}^H x^h_{i,lat}(z)^2
\end{equation}
which penalizes the deviations from the reference line by penalizing the lateral displacement $x^h_{i,lat}$. $\omega_\textup{ref}$ is the corresponding weighting coefficient. For collision avoidance, we represent each vehicle with two identical circles aligned along the longitudinal axis of the vehicle. For vehicle $i$, the positions of the centers of circles are denoted as $x_{i,f}$ and $x_{i,r}$, respectively. The safety indice $J^i_s$ is defined as
\begin{equation}
J^i_s(z)=\omega_s\sum_{h=0}^H\sum_{
j,\beta,\gamma
}\min\{D(x^h_{i,\beta}(z)-x^h_{j,\gamma}(z))-d,0\}^2
\end{equation}
where $D(\cdot,\cdot)$ defines the Euclidean distance, $j\in\mathcal{N}-\{i\}$, and $\beta,\gamma\in\{f,r\}$. $\omega_s$ is the weighting coefficient. As a result, a negative utility (namely a penalty) is added towards the overall utility when vehicle $i$ fails to maintain a safe distance $d$ with respect to other vehicles. Values of parameters are shown in Table I.

\begin{table}[t]
\centering
\caption{Parameter settings used in simulations}
\begin{tabular}{p{0.65cm}p{0.65cm}p{0.65cm}p{0.65cm}p{0.65cm}p{0.65cm}p{0.65cm}p{0.65cm}}
\hline
Param. & Value & Param. & Value & Param.& Value & Param.& Value \\ \hline
$\omega_{a,lat}$ & 0.5 & $\omega_{a,long}$ & 1.0 &$\omega_{j,lat}$ & 0.5 & $\omega_{j,long}$ & 1.0 \\
$\omega_p$ & 20.0 & $\omega_\textup{ref}$  & 10.0 & $\omega_s$  & 2000.0 & $d$ & 4.0\,m \\ \hline
\end{tabular}
\end{table}

\section{Simulation Results}

To illustrate the effectiveness of the introduced method in handling multi-modal driving behaviors and to verify the generalizability, two case studies of different traffic scenarios and a quantitative analysis compared with a baseline are included in this section. In both case studies, each agent maintains its game solver and solves the game separately. The types of Human-driving vehicles are given and fixed while the autonomous vehicle determines its own type adopting the strategy described in (\ref{decision}). All algorithms are implemented in Python 3.8 running on a server with Intel(R) Xeon(R) Platinum 8358P CPU @ 2.60GHz.

\subsection{Case I: Ramp Merging}

\begin{figure*}[t]
\centering
\subfigure[Simulation results of Scenario A at 0.8\,s, 1.6\,s, 2.4\,s, and 3.2\,s, where HV1 yields and AV merges into the gap between the two HVs.]{
\includegraphics[scale=0.0355]{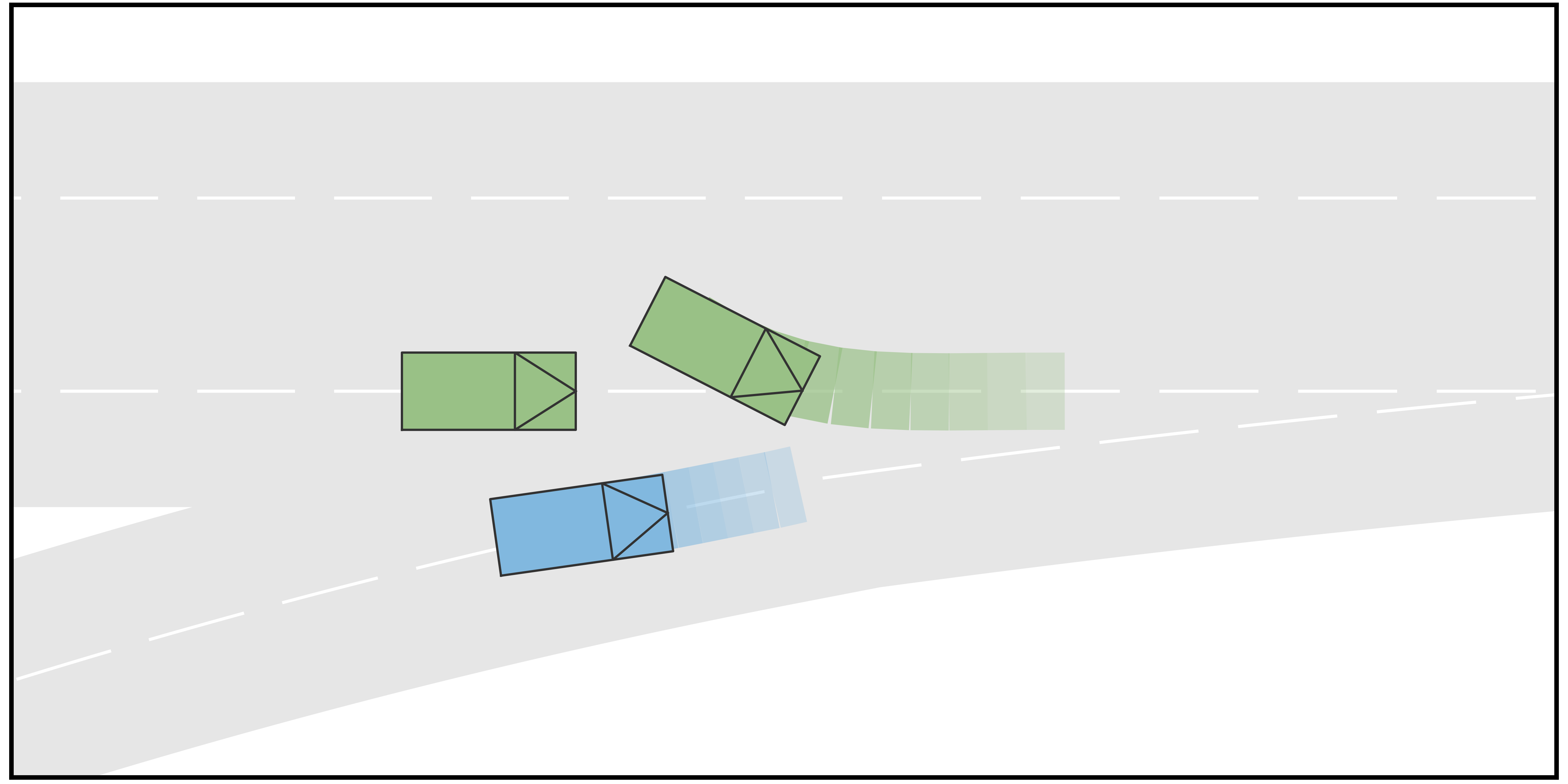}
\includegraphics[scale=0.0355]{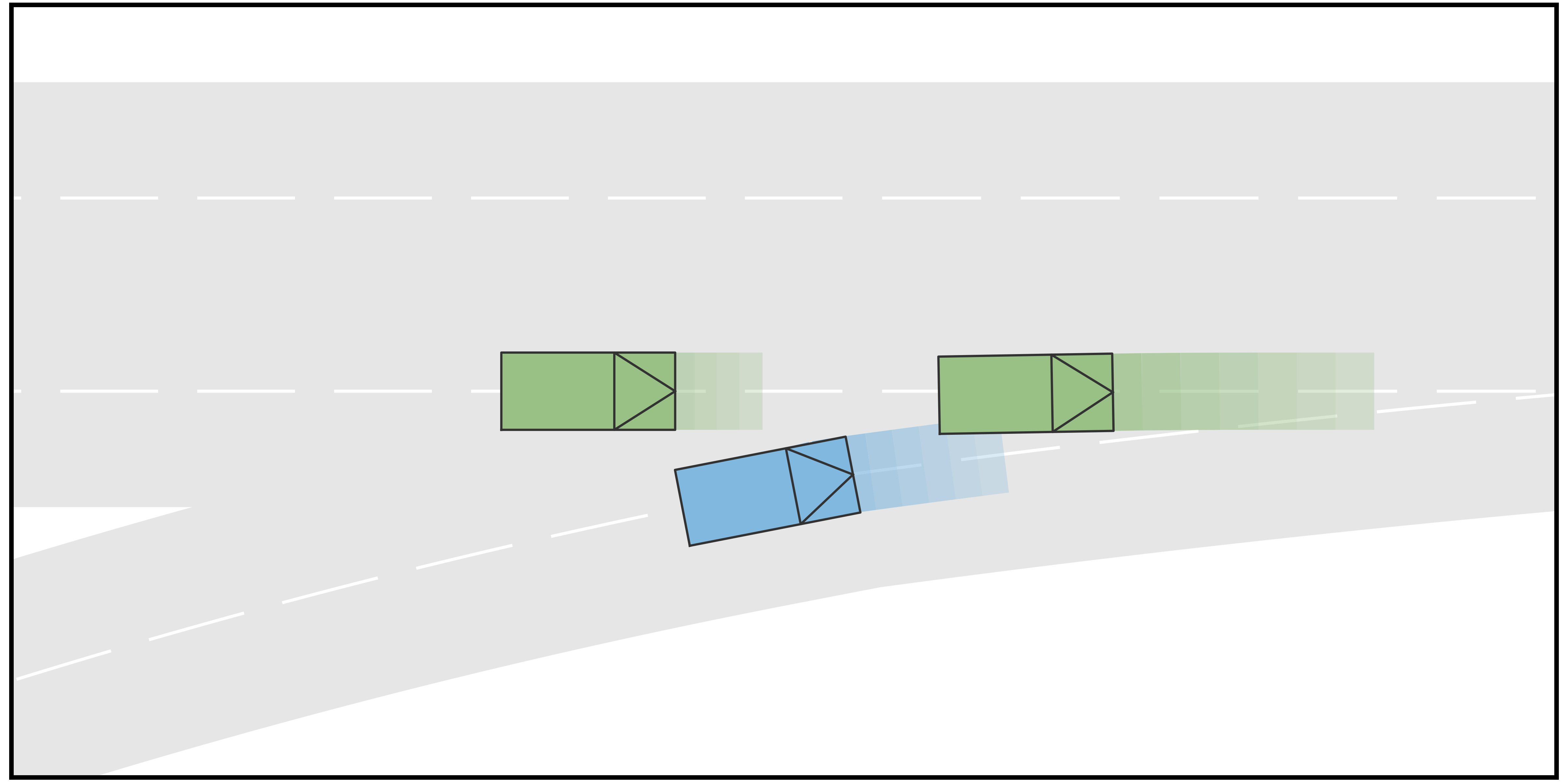}
\includegraphics[scale=0.0355]{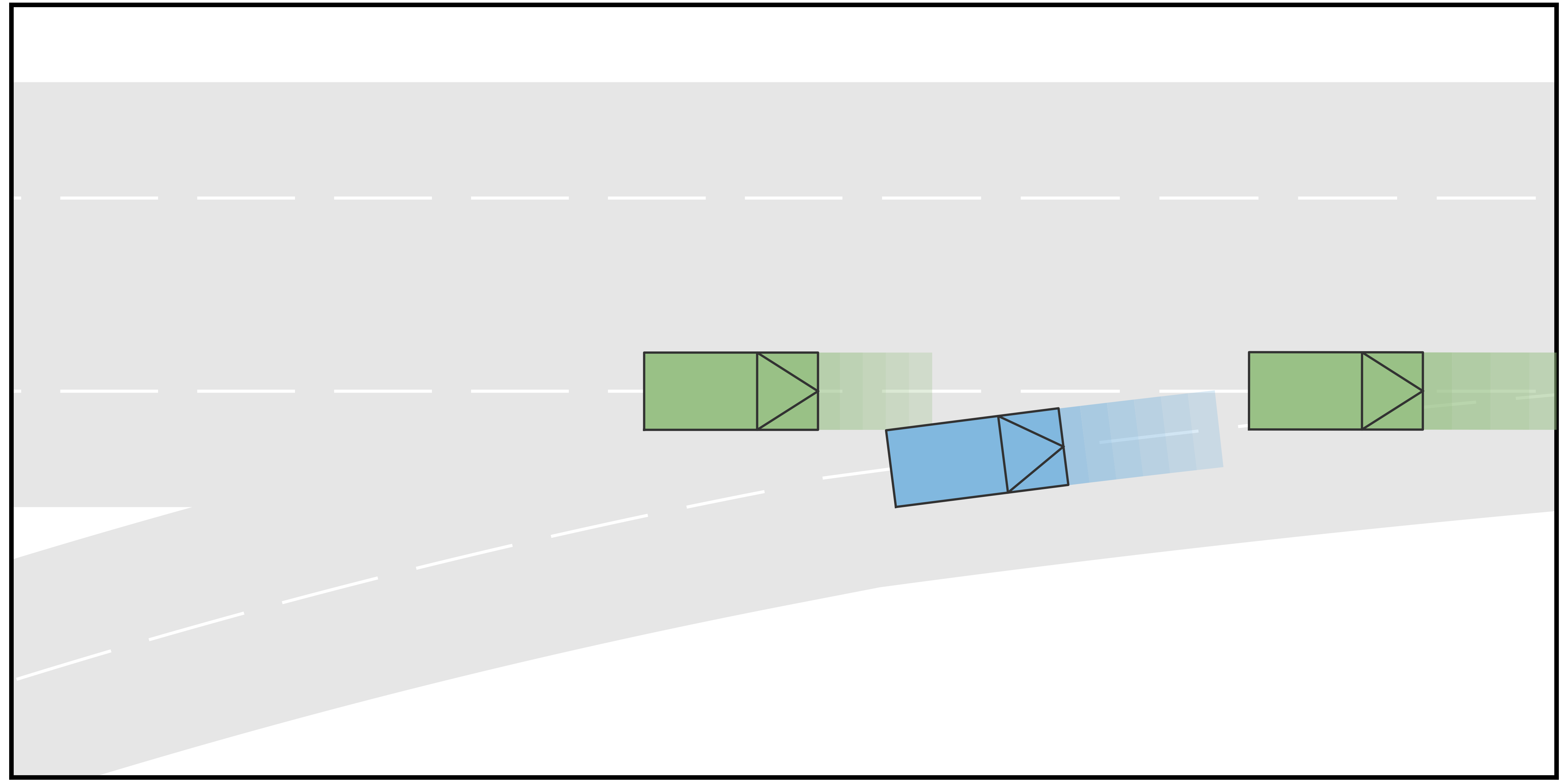}
\includegraphics[scale=0.0355]{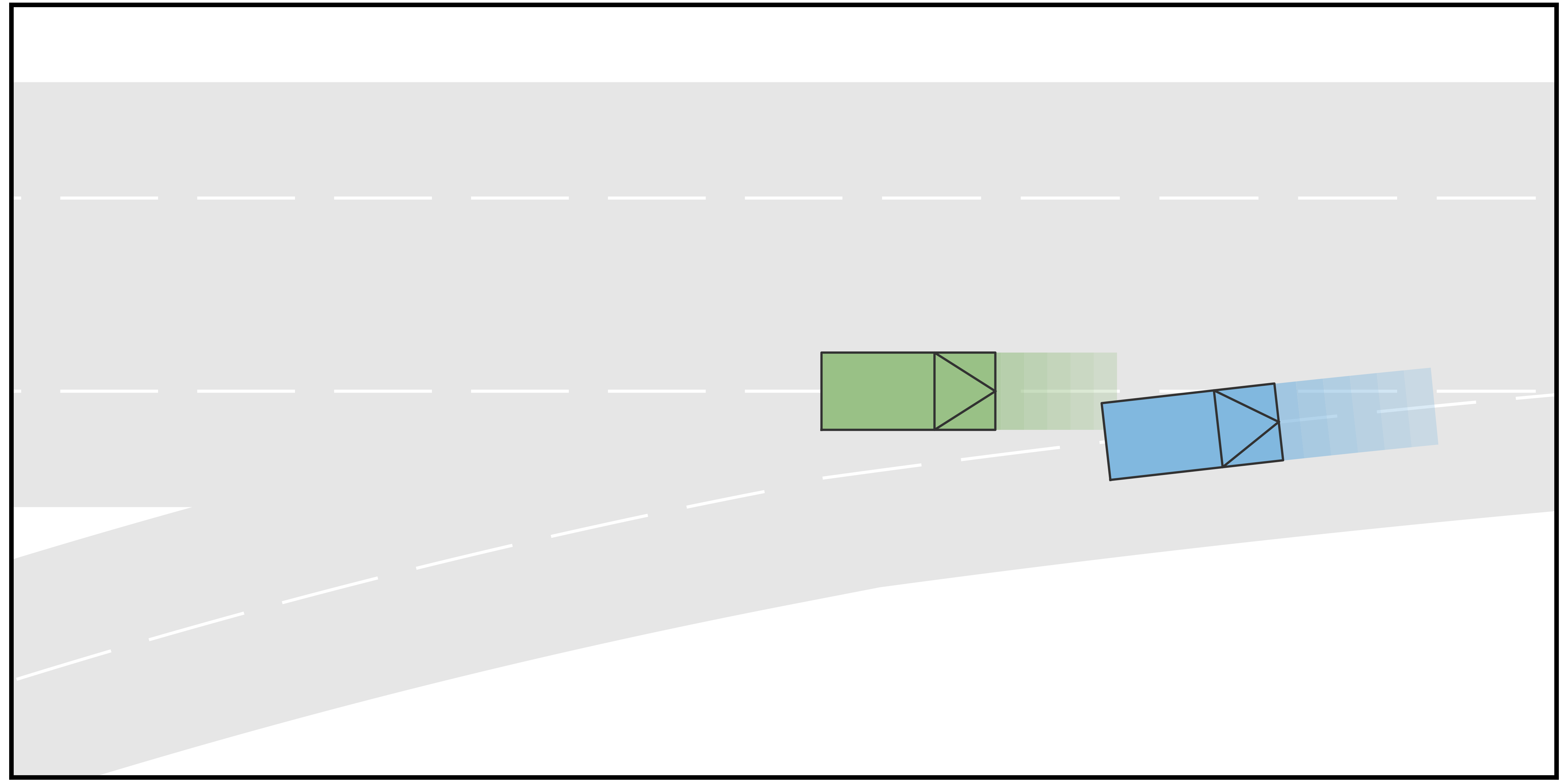}
}
\subfigure[Simulation results of Scenario B at 0.8\,s, 1.6\,s, 2.4\,s, and 3.2\,s, where HV1 does not yield and AV merges behind both HVs.]{
\includegraphics[scale=0.0355]{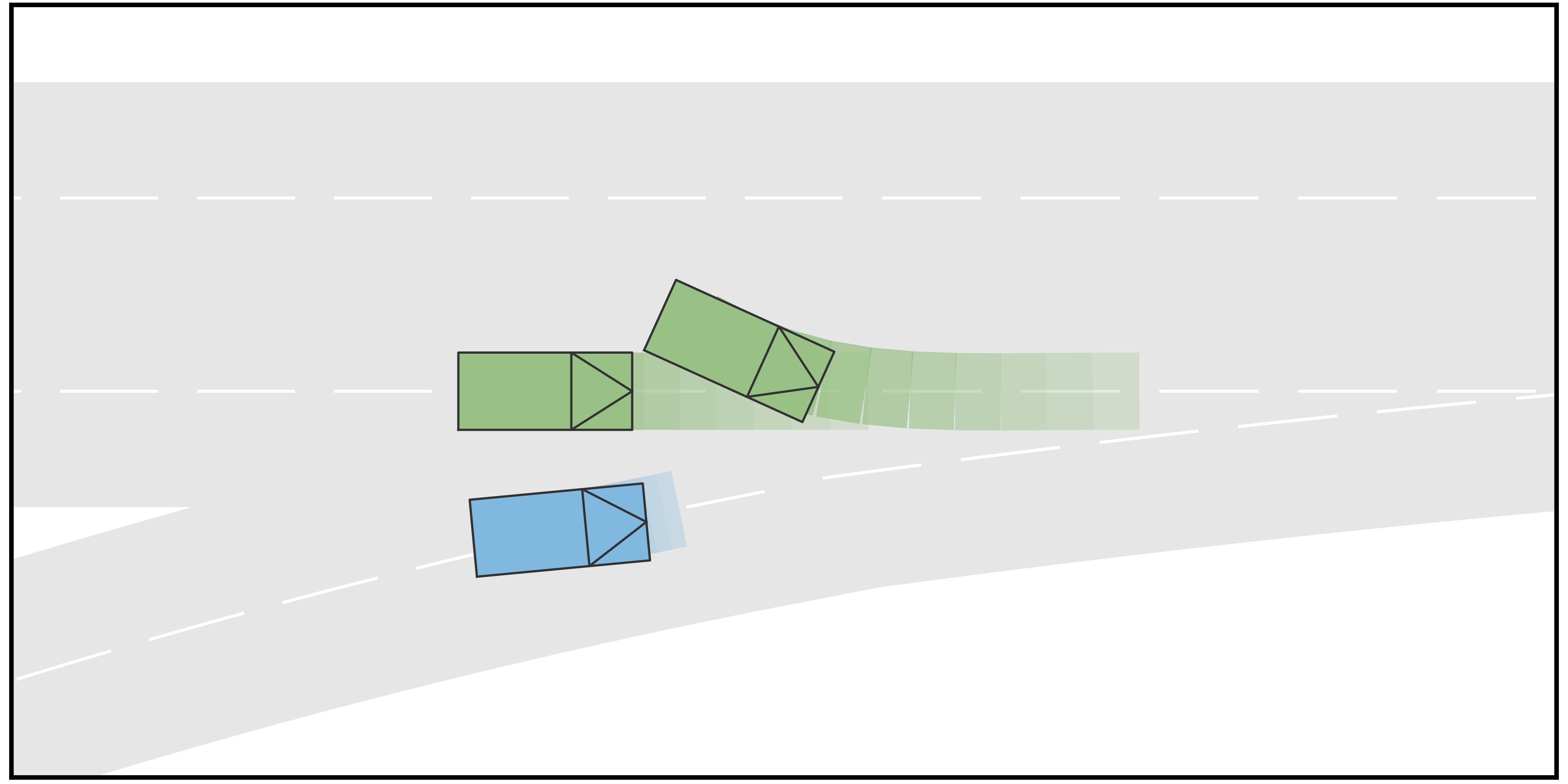}
\includegraphics[scale=0.0355]{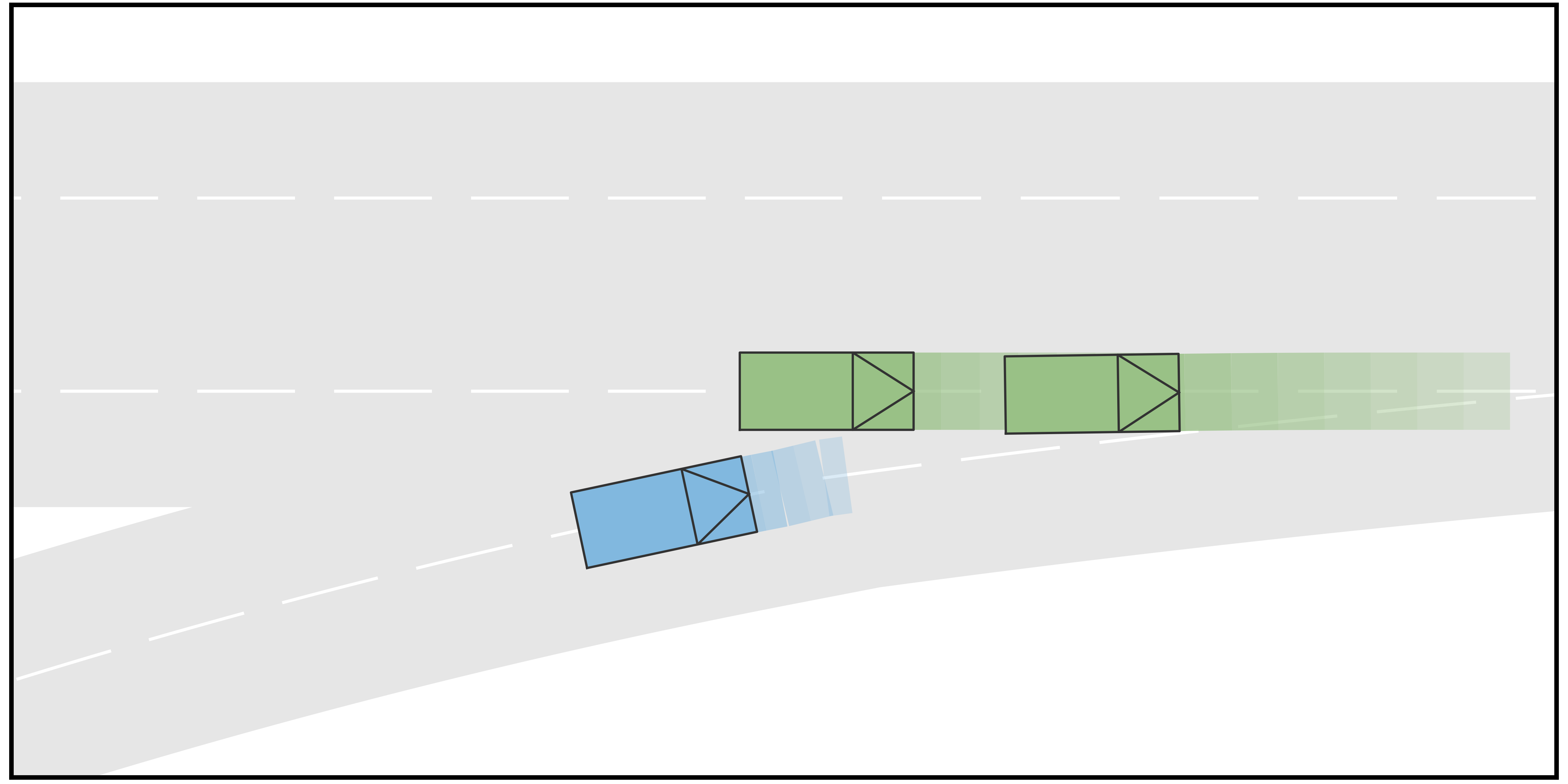}
\includegraphics[scale=0.0355]{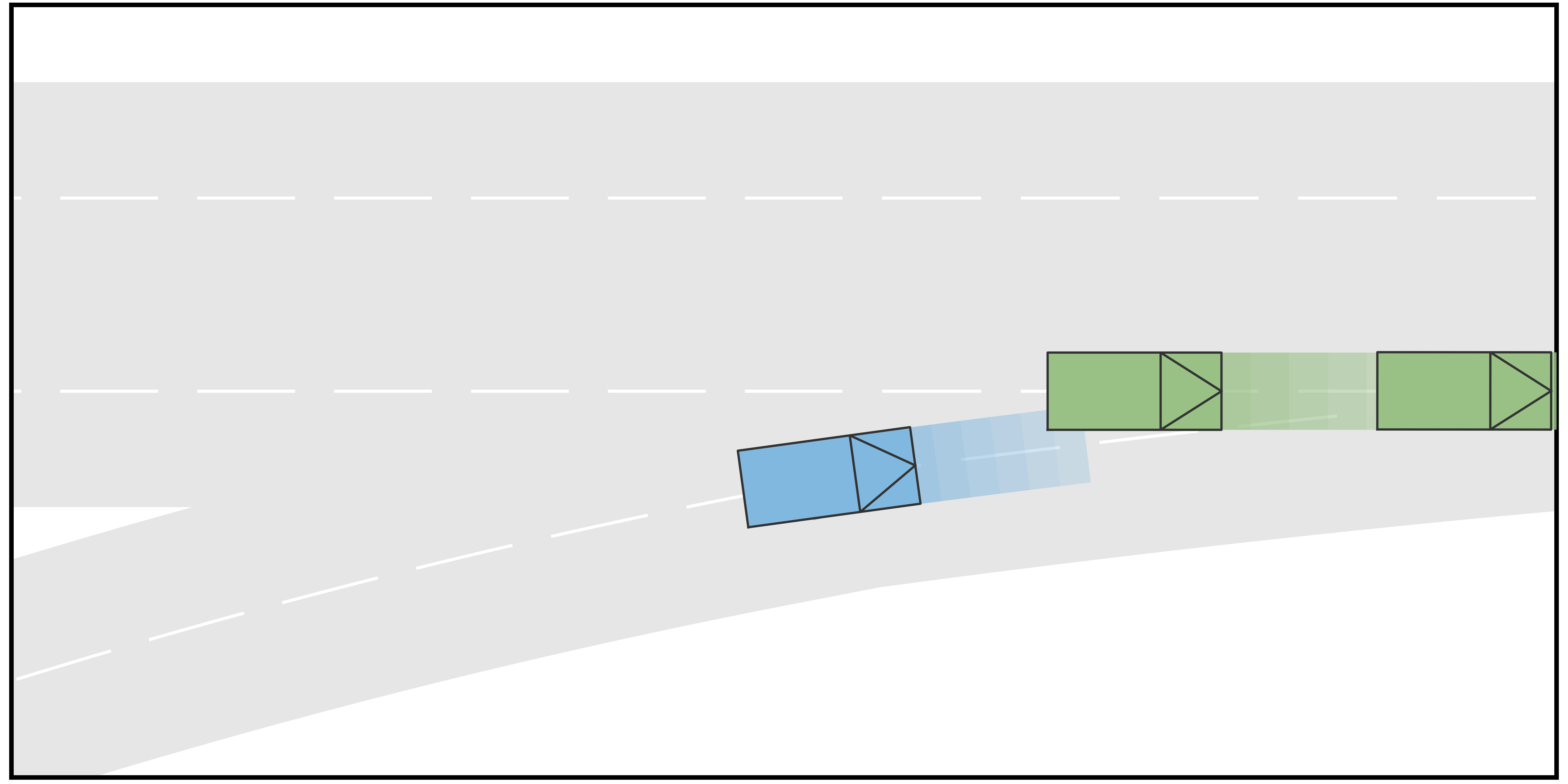}
\includegraphics[scale=0.0355]{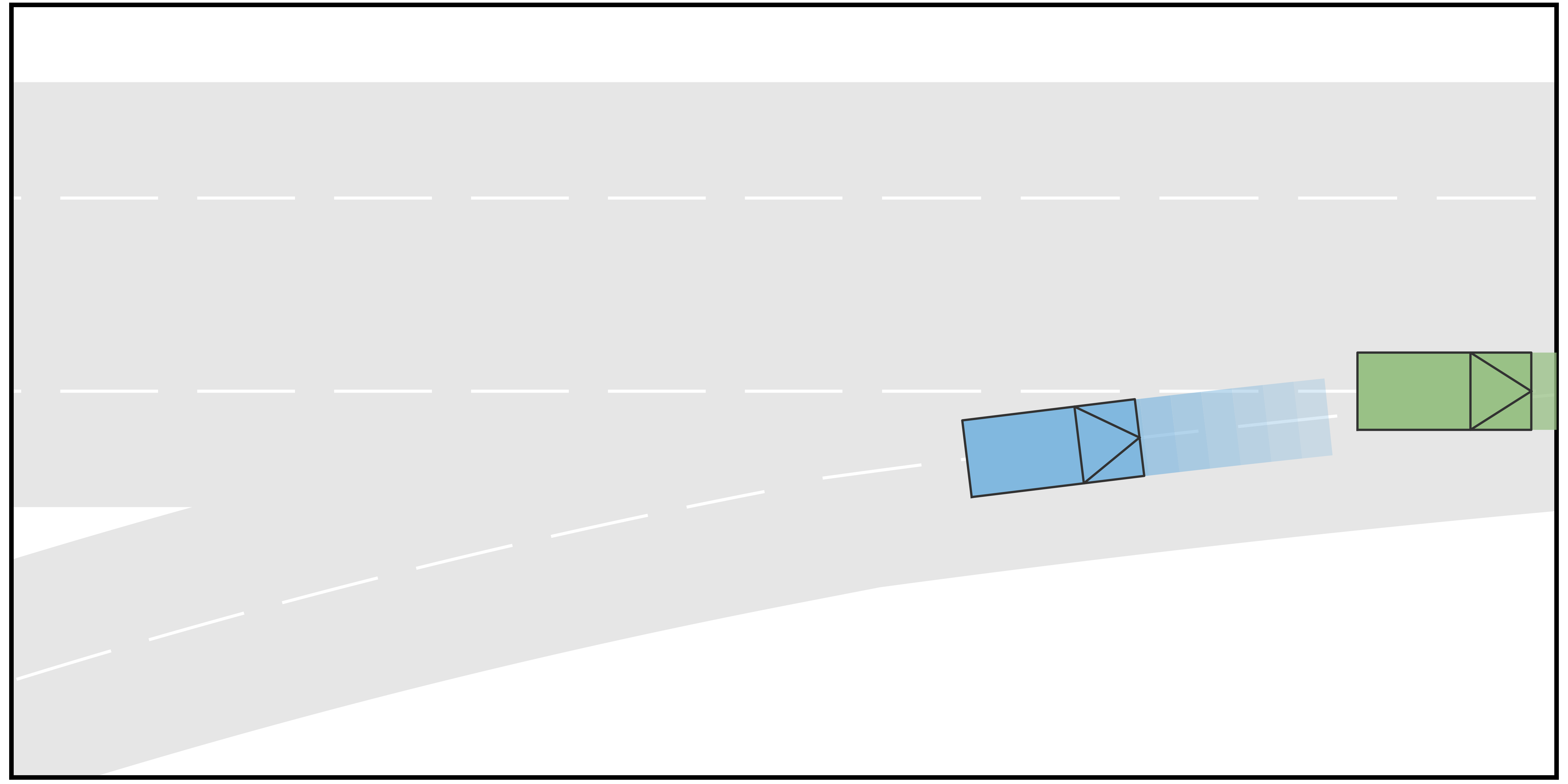}
}
\subfigure[Simulation results of Scenario C at 0.8\,s, 1.6\,s, 2.4\,s, and 3.2\,s. where HV2 yields and AV merges into the gap between the two HVs.]{
\includegraphics[scale=0.0355]{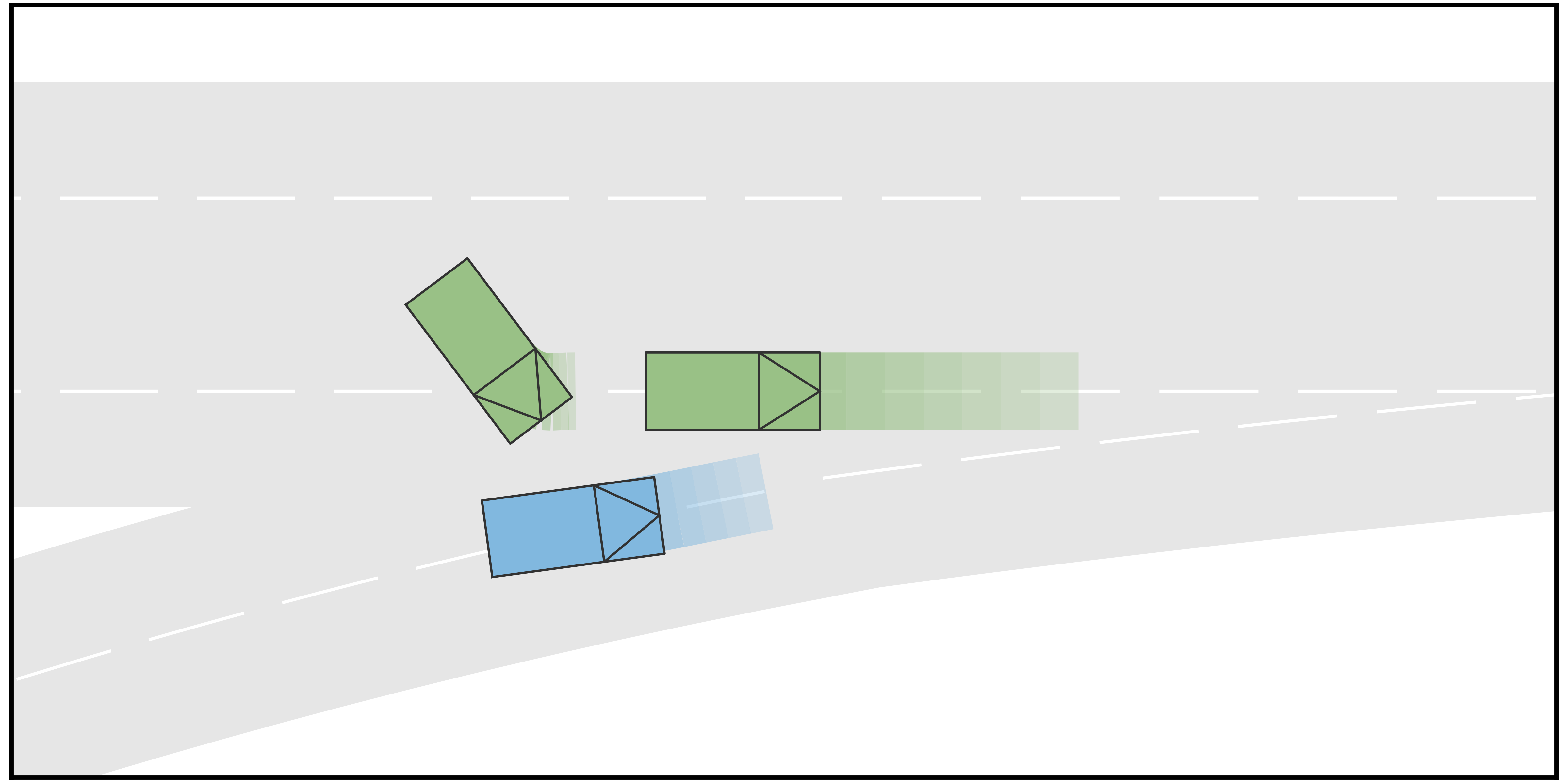}
\includegraphics[scale=0.0355]{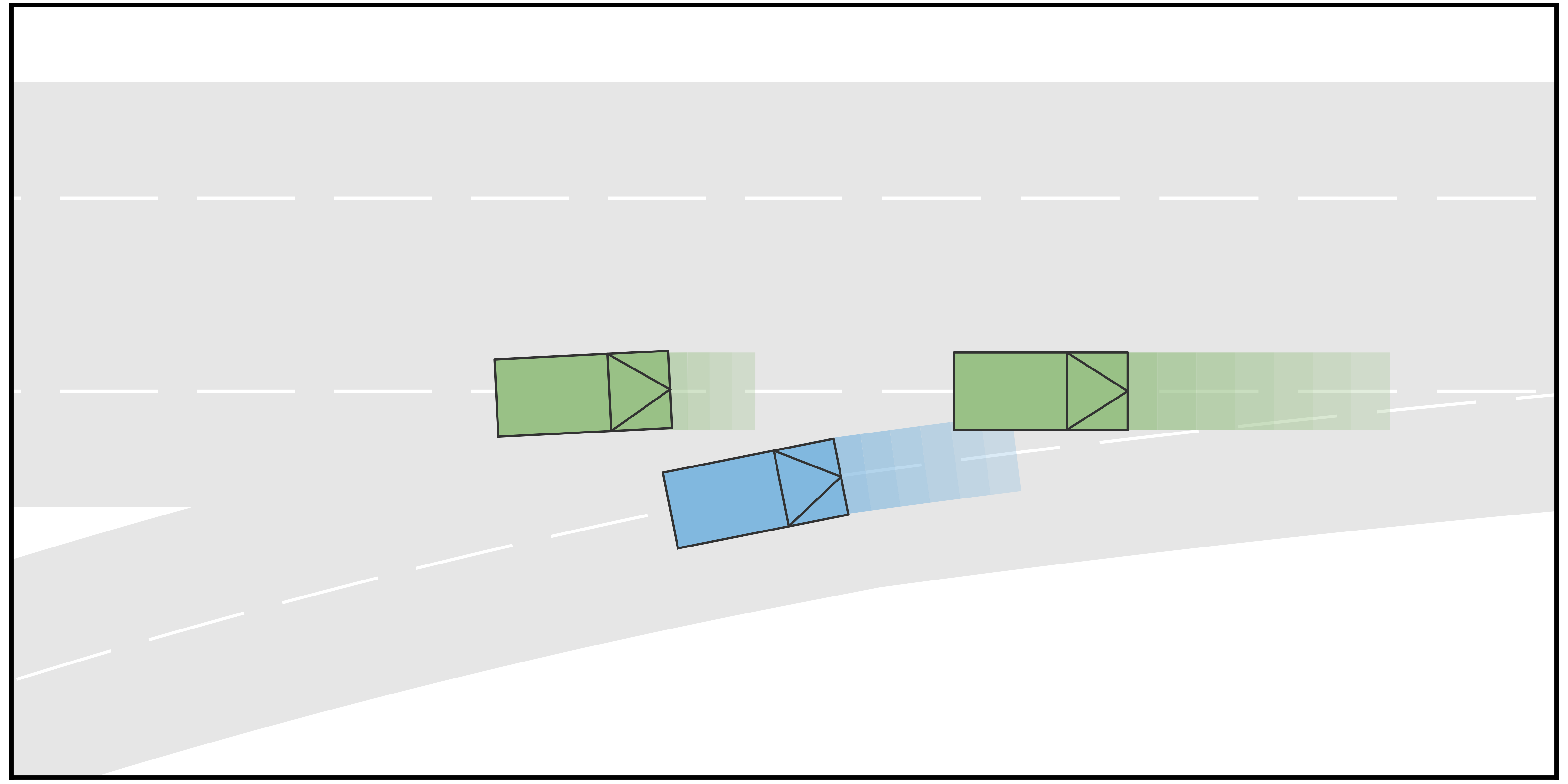}
\includegraphics[scale=0.0355]{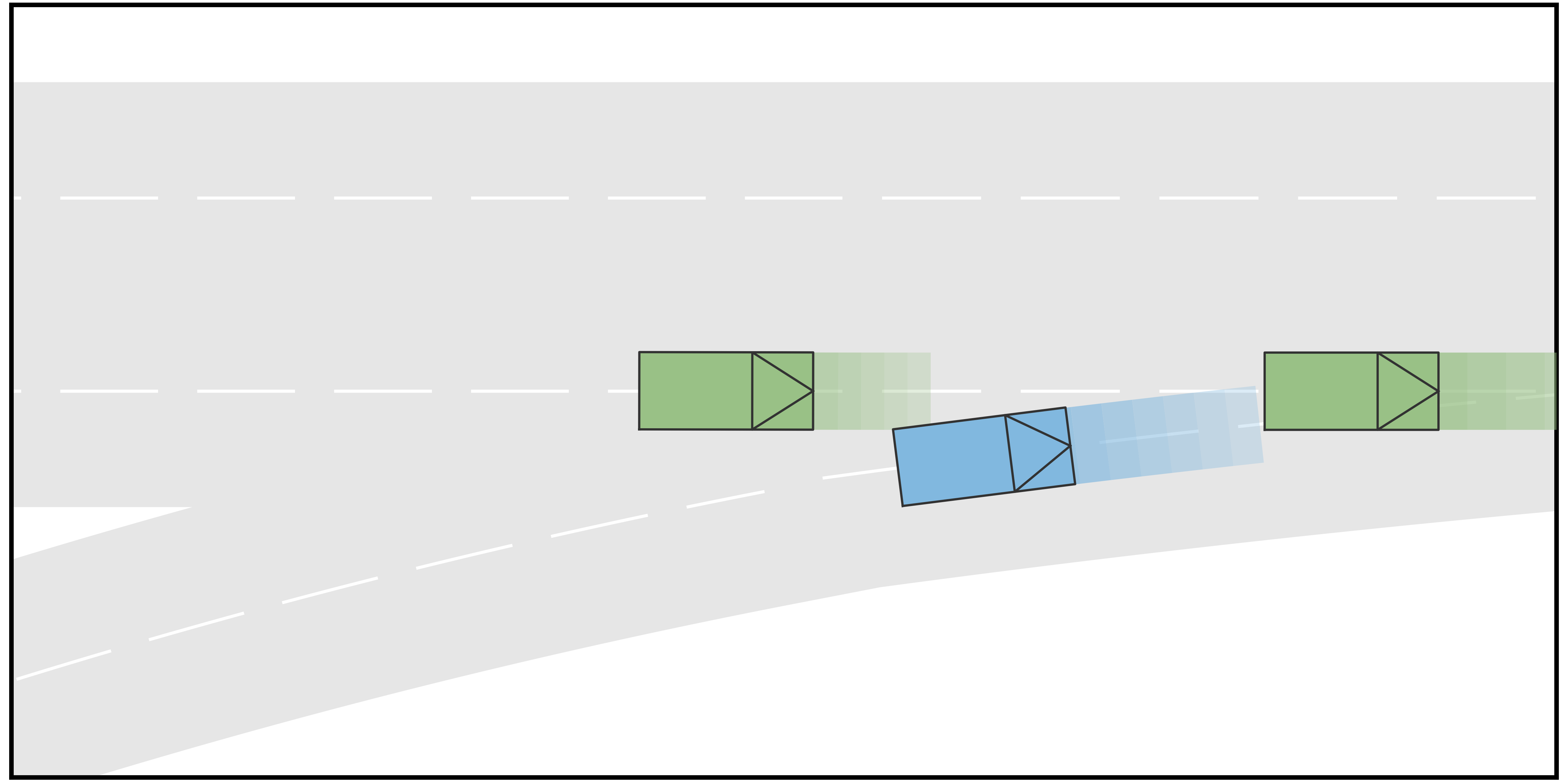}
\includegraphics[scale=0.0355]{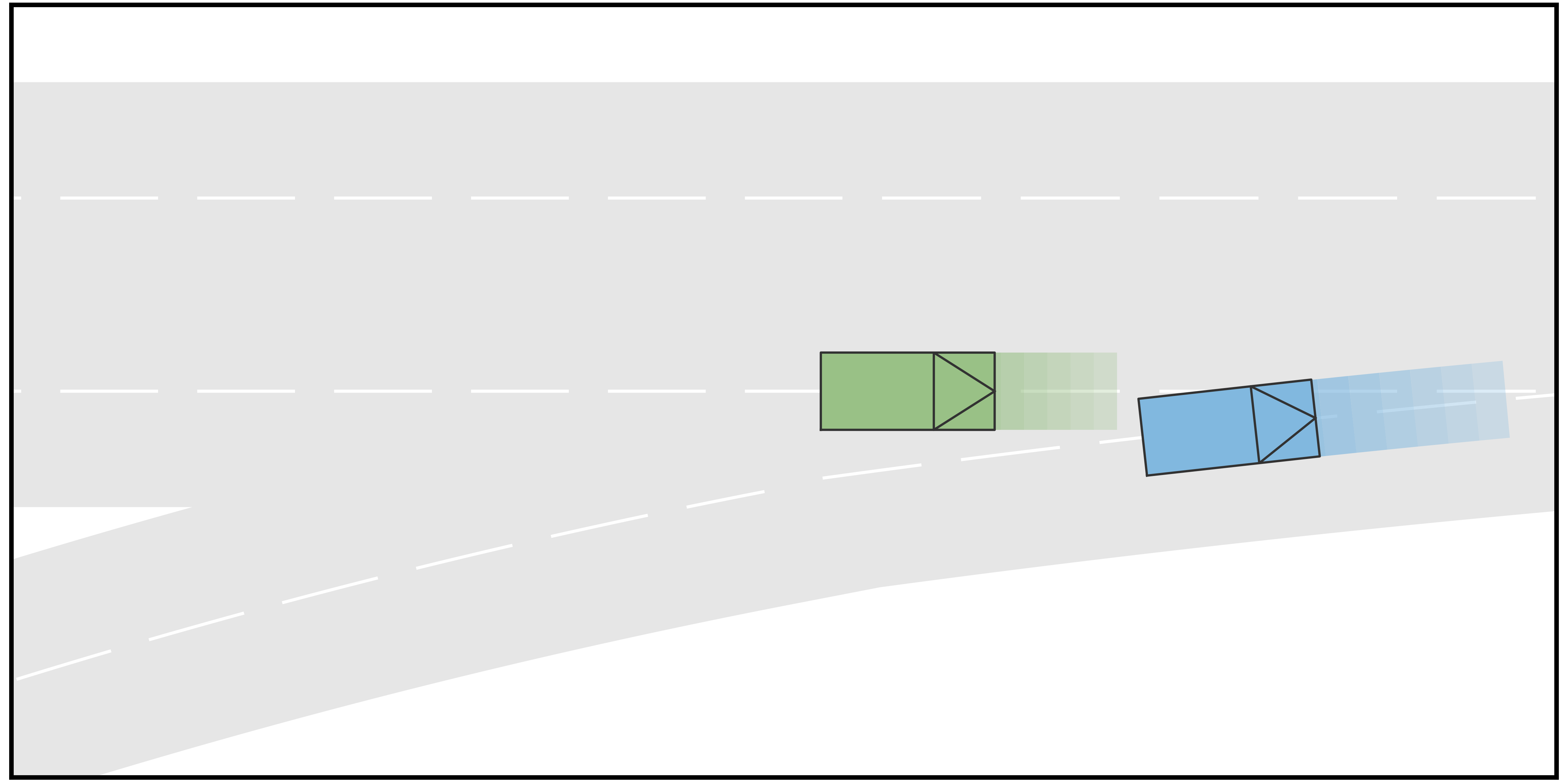}
}
\subfigure[Simulation results of Scenario D at 0.8\,s, 1.6\,s, 2.4\,s, and 3.2\,s, where HV2 does not yield and AV merges behind both HVs.]{
\includegraphics[scale=0.0355]{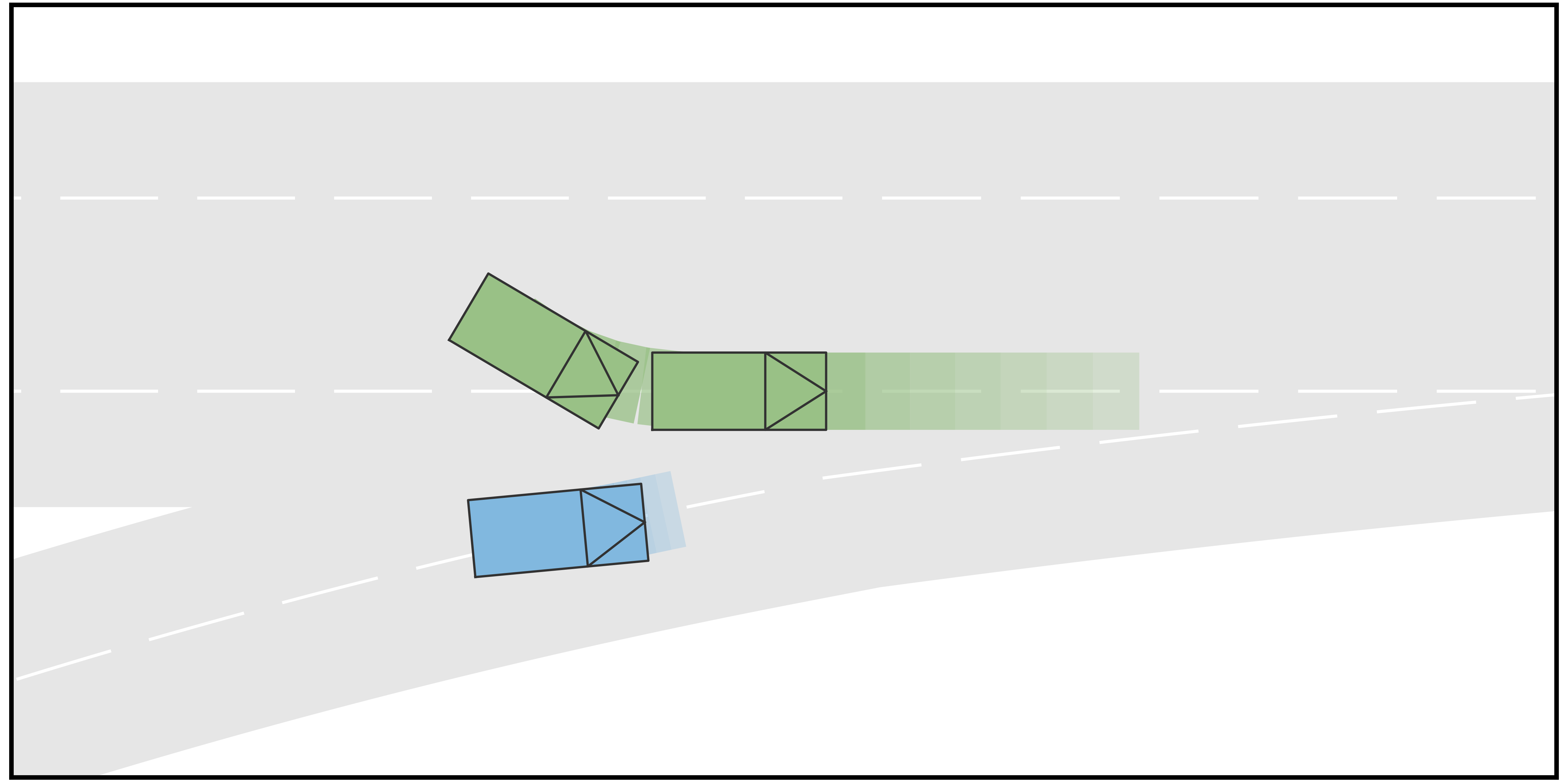}
\includegraphics[scale=0.0355]{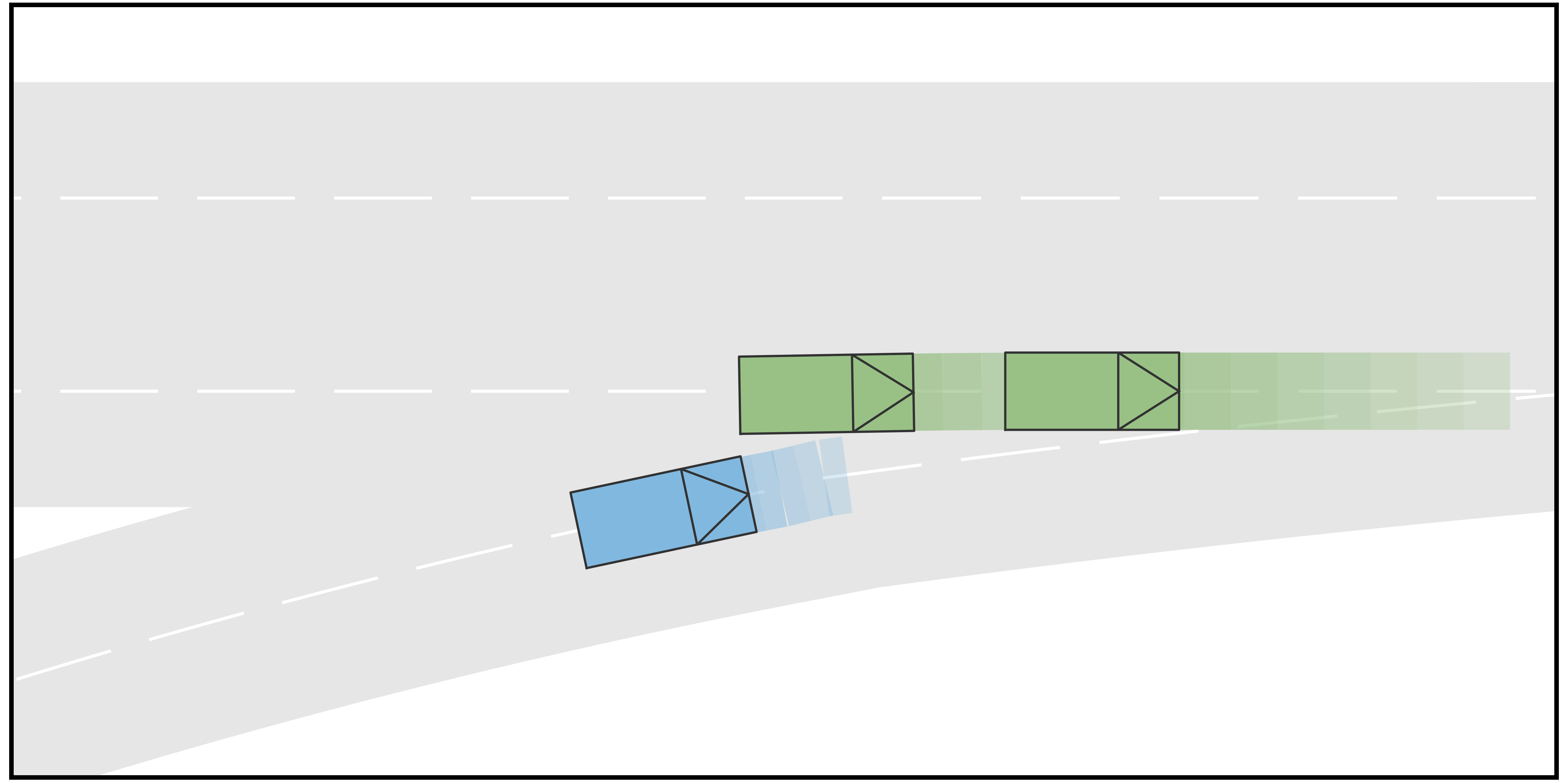}
\includegraphics[scale=0.0355]{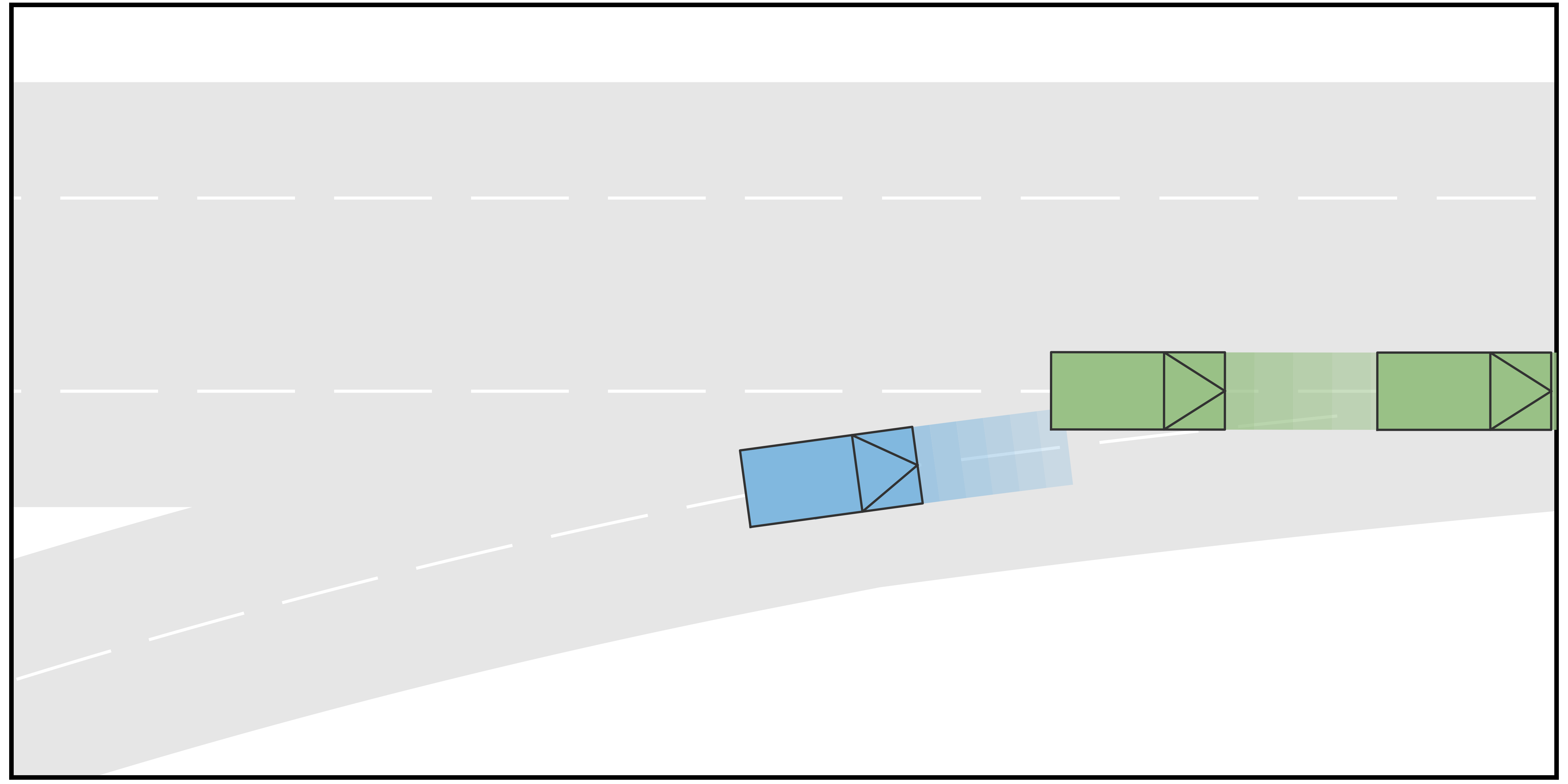}
\includegraphics[scale=0.0355]{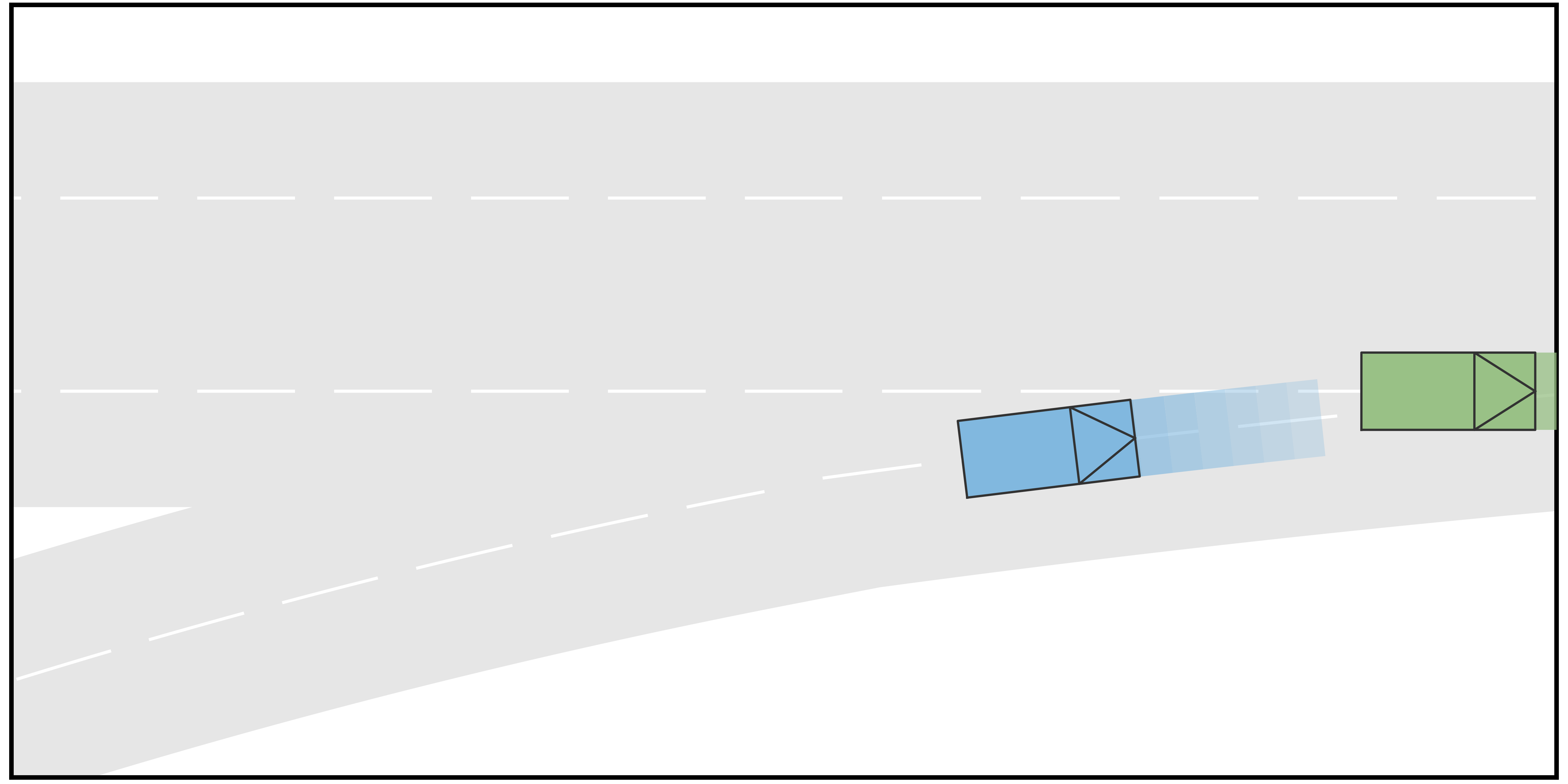}
}

\caption{Simulation results for Case I. In all scenarios, AV manages to merge into the main road without collision. Depending on the differentiated behaviors of HVs, AV either merges into the gap between two HVs or merges behind both HVs.}

\label{fig:case_A}
\end{figure*}

In this case, the goal of the autonomous driving vehicle (AV) is to merge into a two-lane road from the ramp. One human-driving vehicle (HV1) is currently driving in the target lane of AV, while another human-driving vehicle (HV2) is running in another lane of the main road but also tries to merge into the target lane of AV. The initial velocity of all three vehicles is $7\,\textup{m}/\textup{s}$. For the type settings, we consider each vehicle to be either a conservative-type vehicle or an aggressive-type vehicle. For each type, we sample the terminal longitudinal velocities that constitute the action space. In particular, for the aggressive type, the action space is determined to be $[7,8,10,12]\,\textup{m}/\textup{s}$, while for the conservative type, the action space is set to be $[6,4,2,0]\,\textup{m}/\textup{s}$. Furthermore, in this case, four different scenarios are considered. In Scenario A, the longitudinal initial positions of AV, HV1, and HV2 are $10\,\textup{m}$, $8\,\textup{m}$, and $12\,\textup{m}$, respectively. HV1 and HV2 are conservative and aggressive, respectively. In Scenario B, the other settings are the same as in Scenario A but HV1 is also aggressive. In Scenario C, the longitudinal initial positions of AV, HV1, and HV2 are $10\,\textup{m}$, $12\,\textup{m}$, and $8\,\textup{m}$, respectively. HV1 and HV2 are aggressive and conservative, respectively. In Scenario D, the other settings are the same as in Scenario C but the HV2 is also aggressive. The time span for both stages in the trajectory tree is $1\,\textup{s}$. 

The simulation results are shown in Fig. \ref{fig:case_A}. It can be seen from the results that when the rear HV decides to yield, the AV manages to identify the type of the HV and successfully merge into the gap between the two HVs, such that the driving efficiency is enhanced. On the contrary, when the rear HV decides not to yield, the AV also identifies and performs the merging after both HVs to avoid collisions. Fig. \ref{fig:long_velocity_A} shows the longitudinal velocities of AV for each scenario. It can be seen that for all 4 scenarios, the AV decelerates at the beginning because at the moment it is uncertain about the driving intentions of the HVs. With incoming observations, the AV keeps updating the belief on HVs' intentions. When it is certain that the rear HV is going to yield (in Scenarios A and C), it accelerates to merge into the gap between the two HVs. On the contrary, when AV is certain that the rear HV is not going to yield (in Scenarios B and D), it takes further braking to avoid collisions and merges behind both HVs.

\begin{figure}[t]
\centering
\includegraphics[scale=0.58]{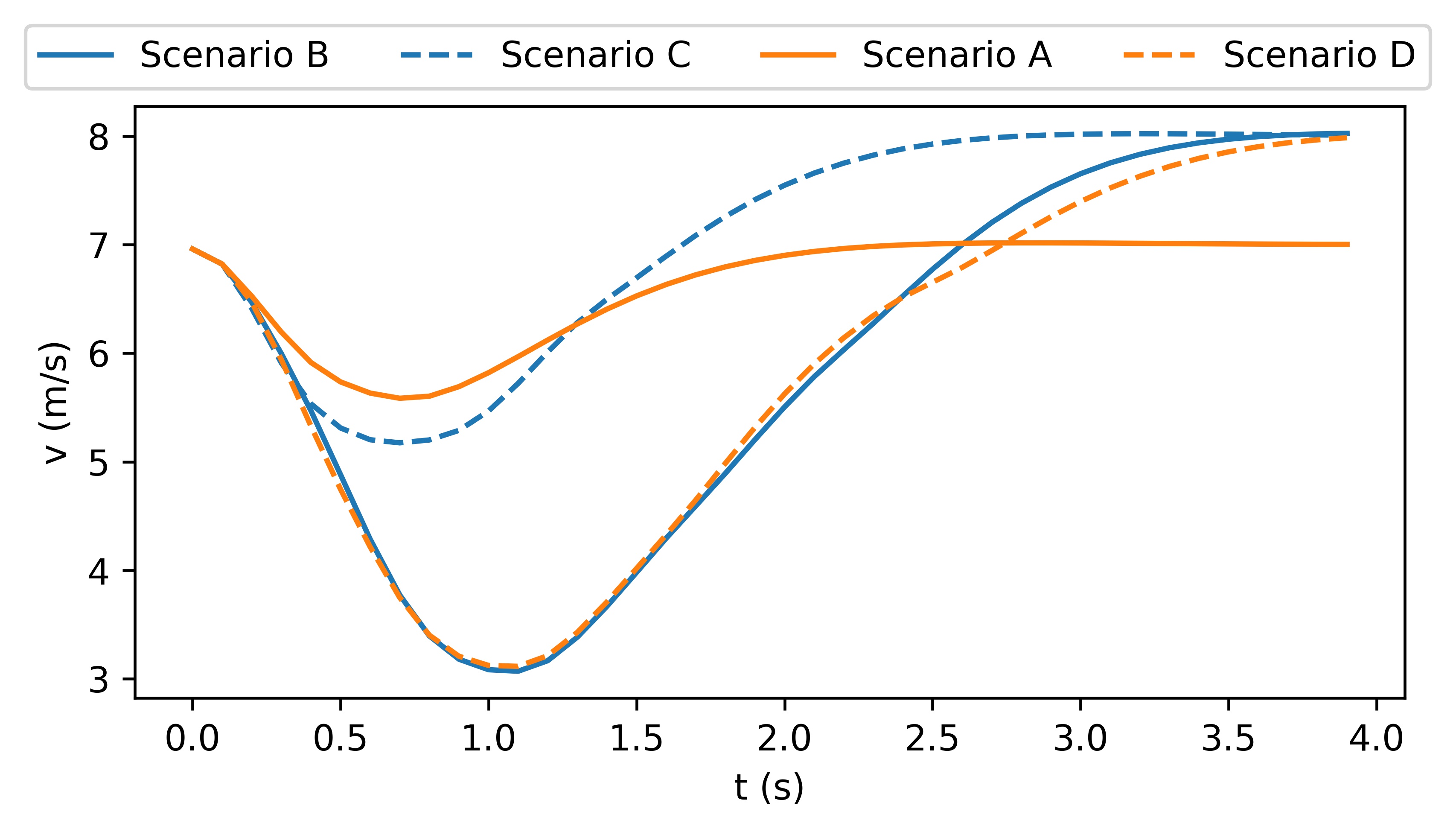}
\caption{Longitudinal velocities of AV in different scenarios in Case I.}
\label{fig:long_velocity_A}
\end{figure}

\subsection{Case II: Unprotected Left-Turn}

\begin{figure*}[t]
\centering
\subfigure[Simulation results of Scenario A at 1.5\,s, 2.5\,s, and 3.5\,s.]{
\includegraphics[scale=0.02209]{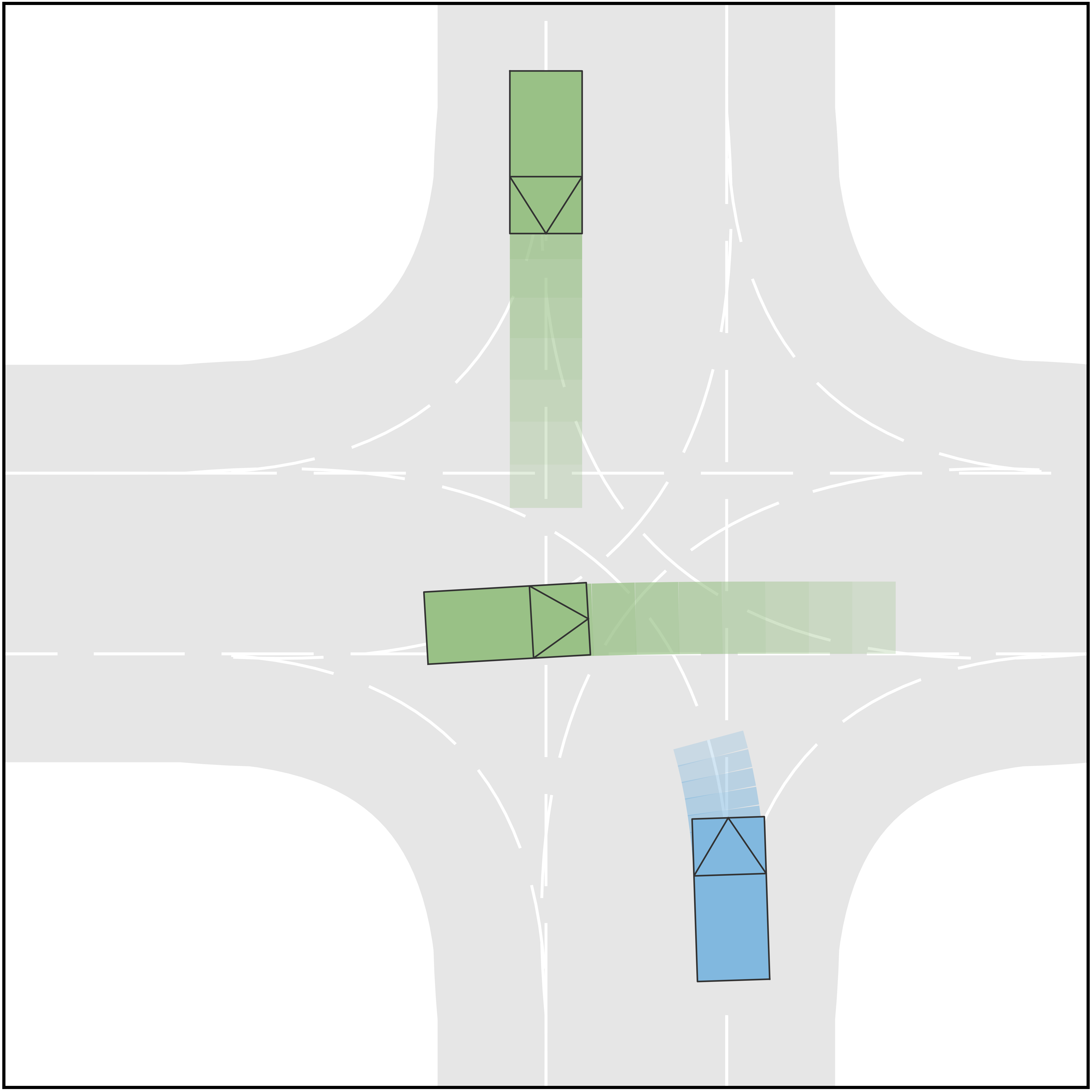}
\includegraphics[scale=0.02209]{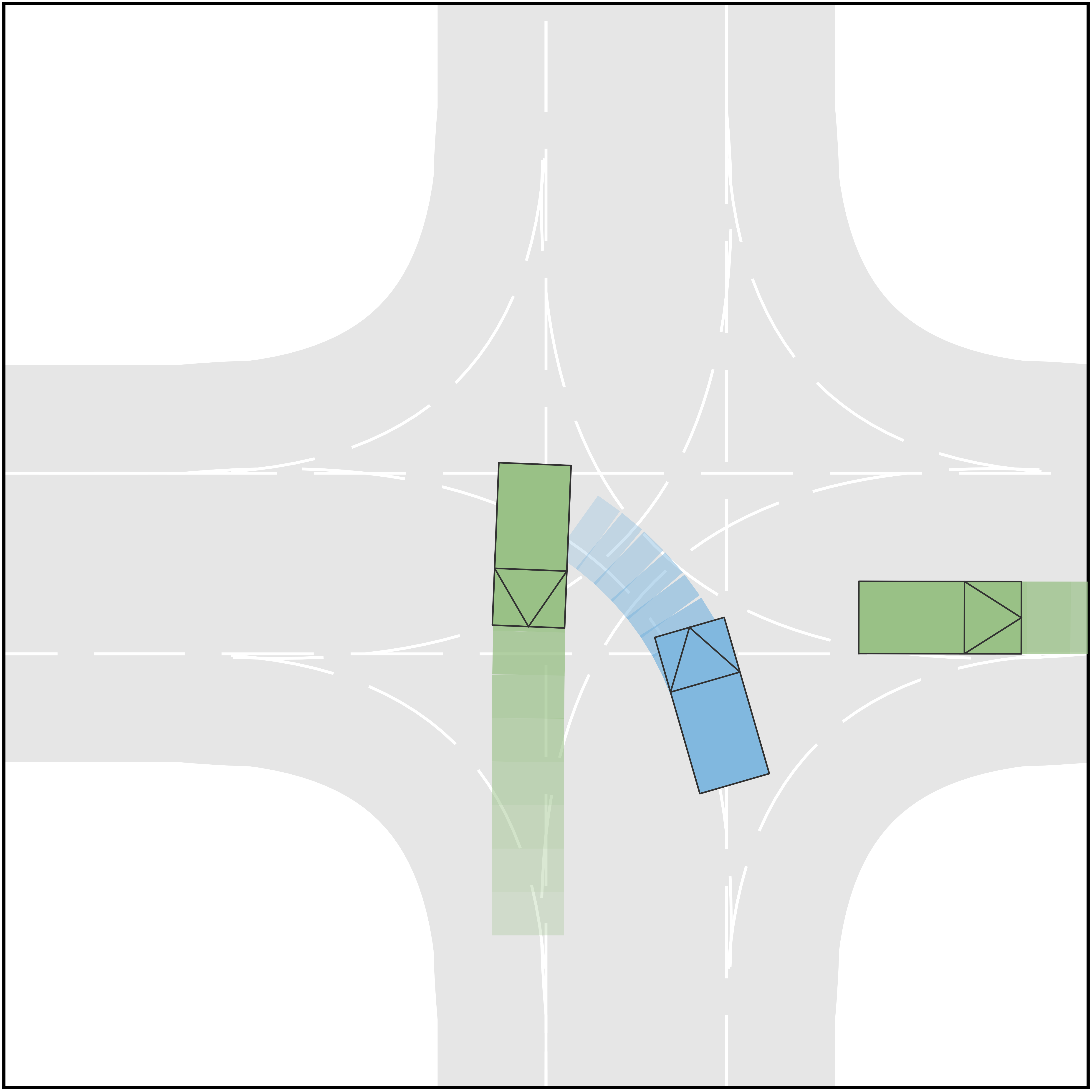}
\includegraphics[scale=0.02209]{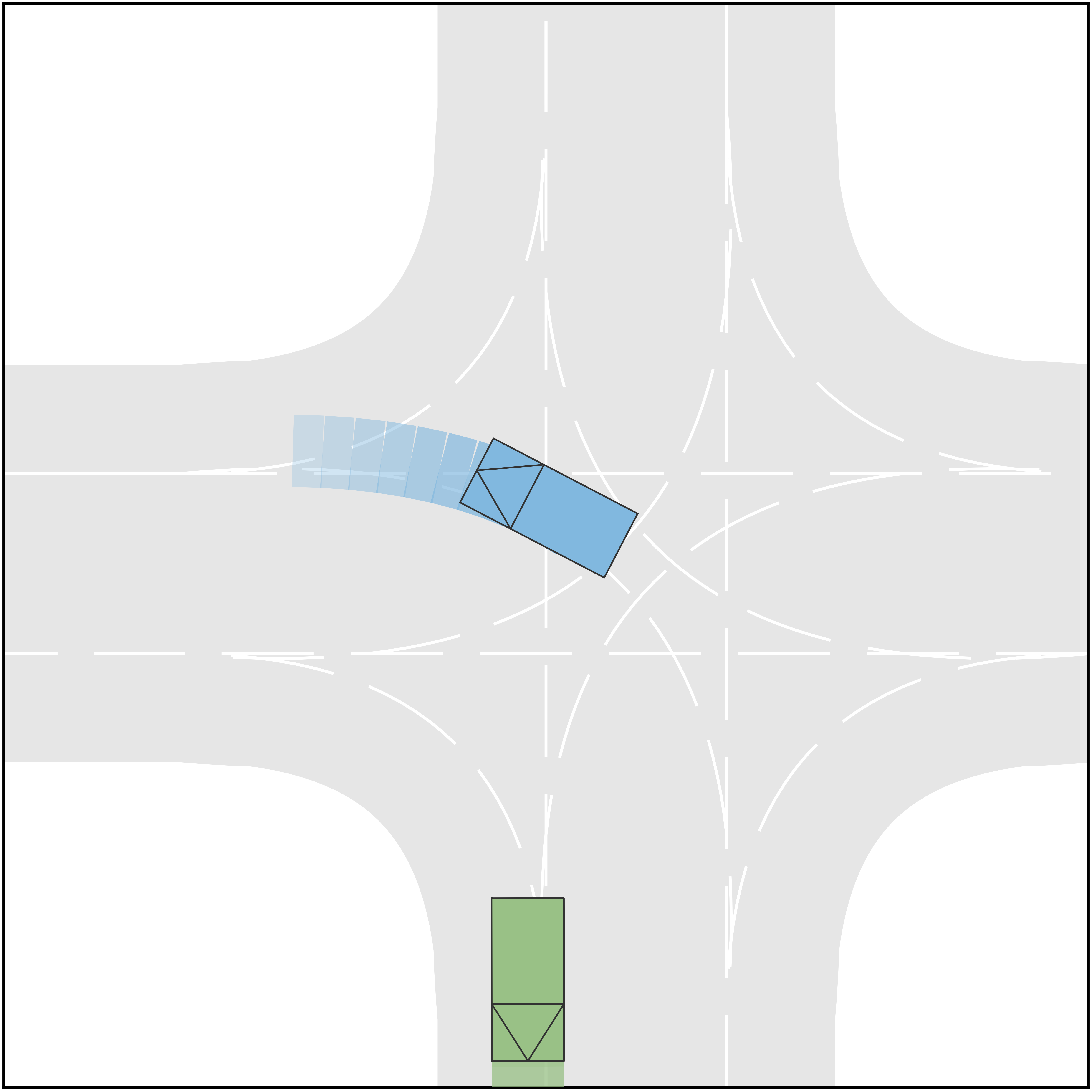}
}
\subfigure[Simulation results of Scenario B at 1.5\,s, 2.5\,s, and 4.0\,s.]{
\includegraphics[scale=0.02209]{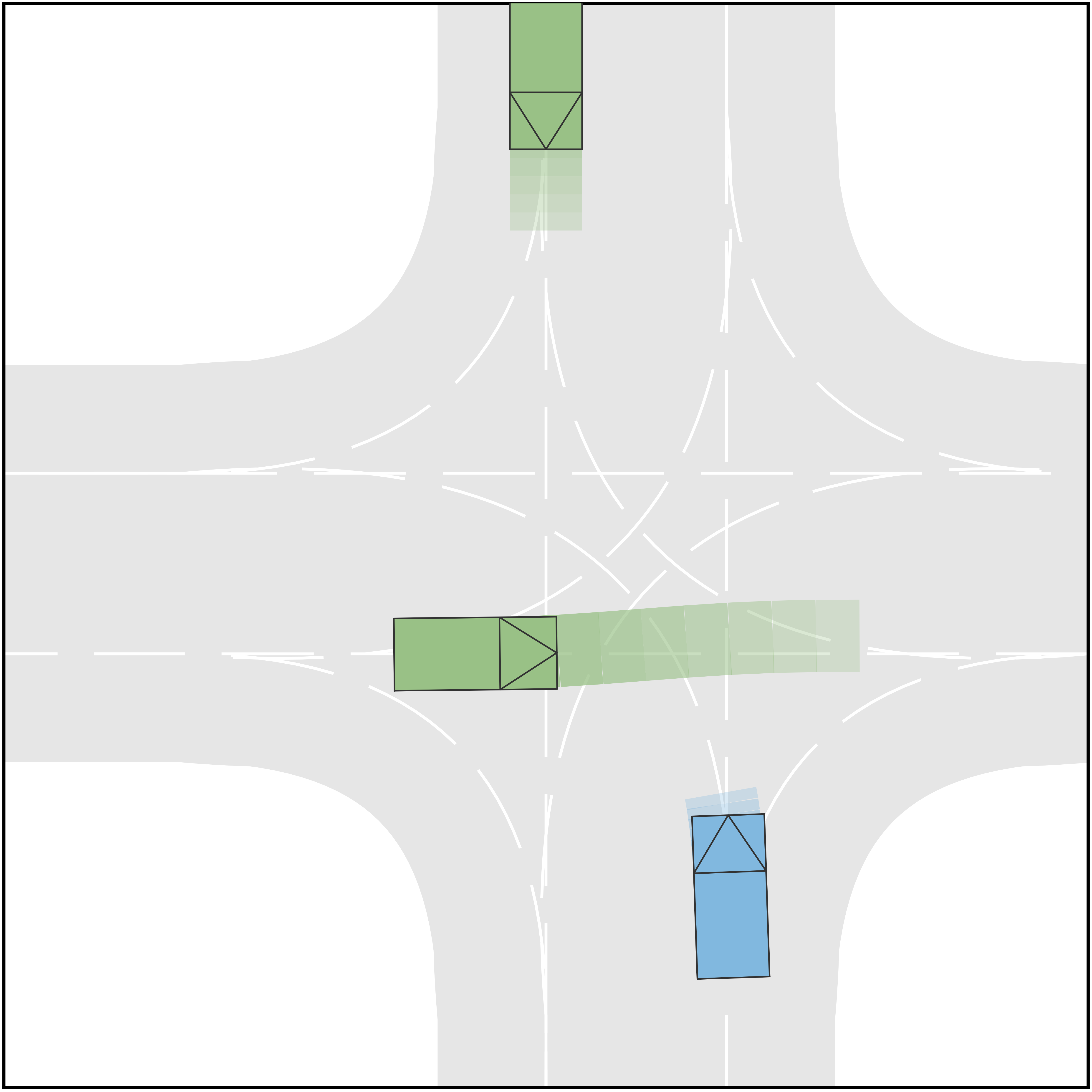}
\includegraphics[scale=0.02209]{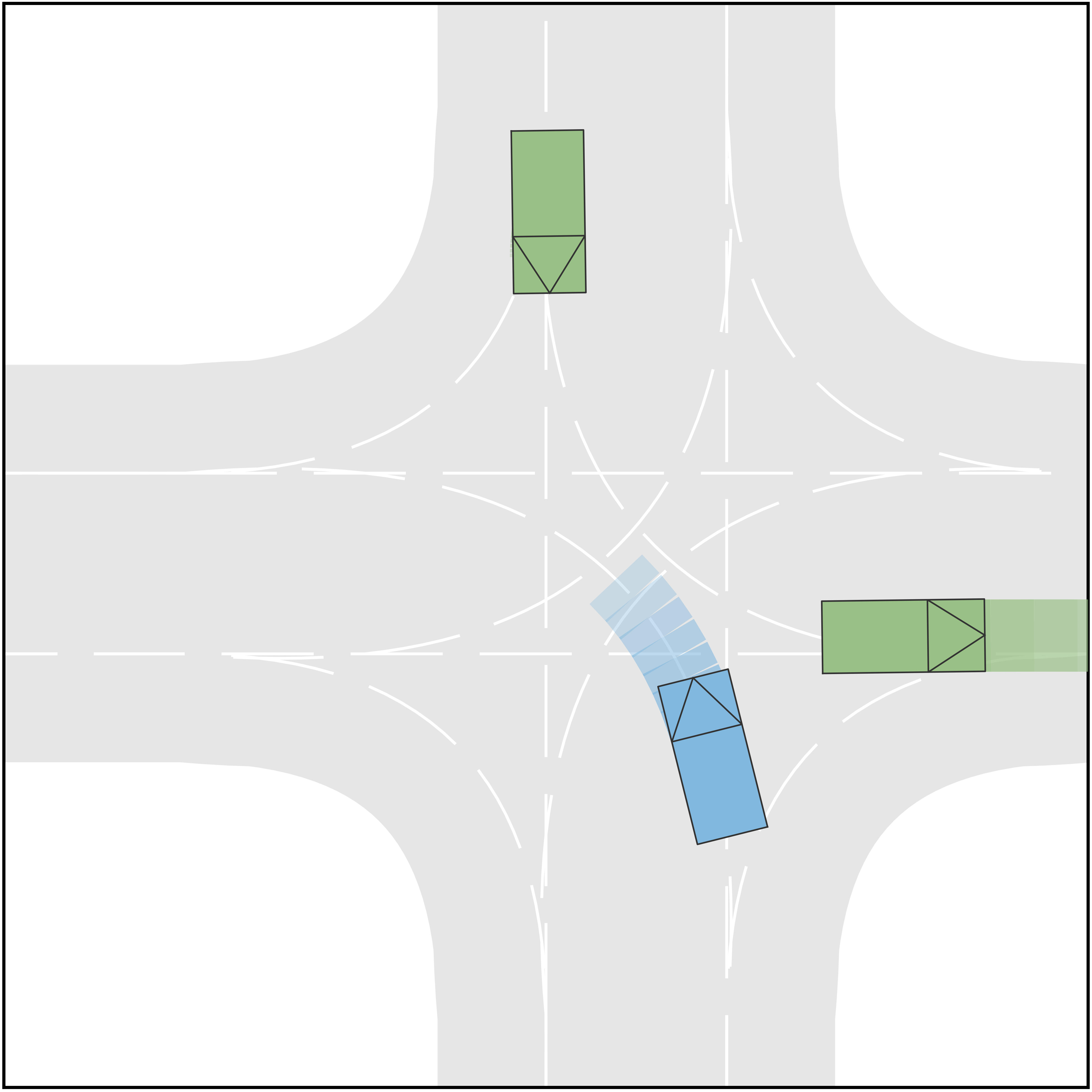}
\includegraphics[scale=0.02209]{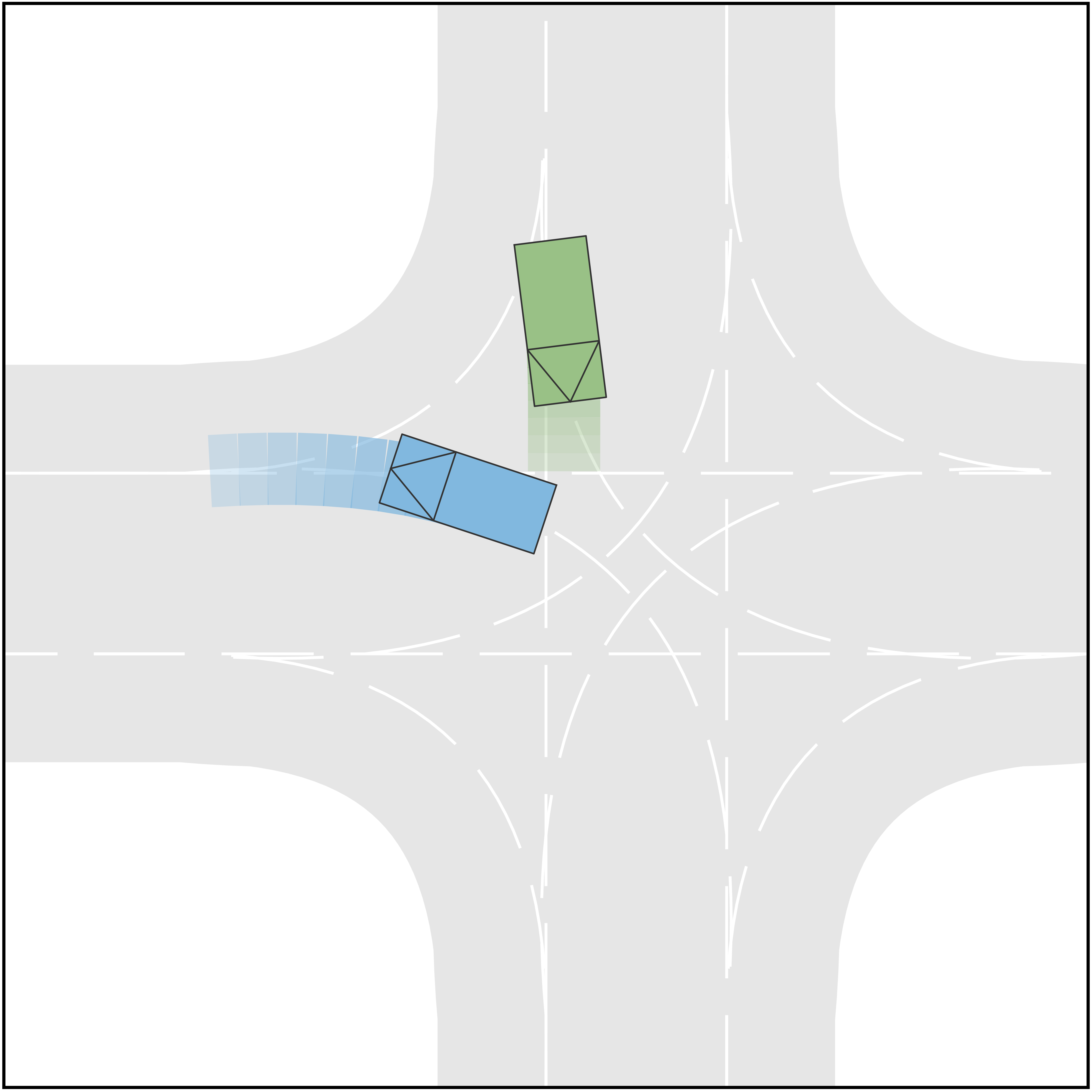}
}
\subfigure[Simulation results of Scenario C at 1.5\,s, 2.5\,s, and 3.5\,s.]{
\includegraphics[scale=0.02209]{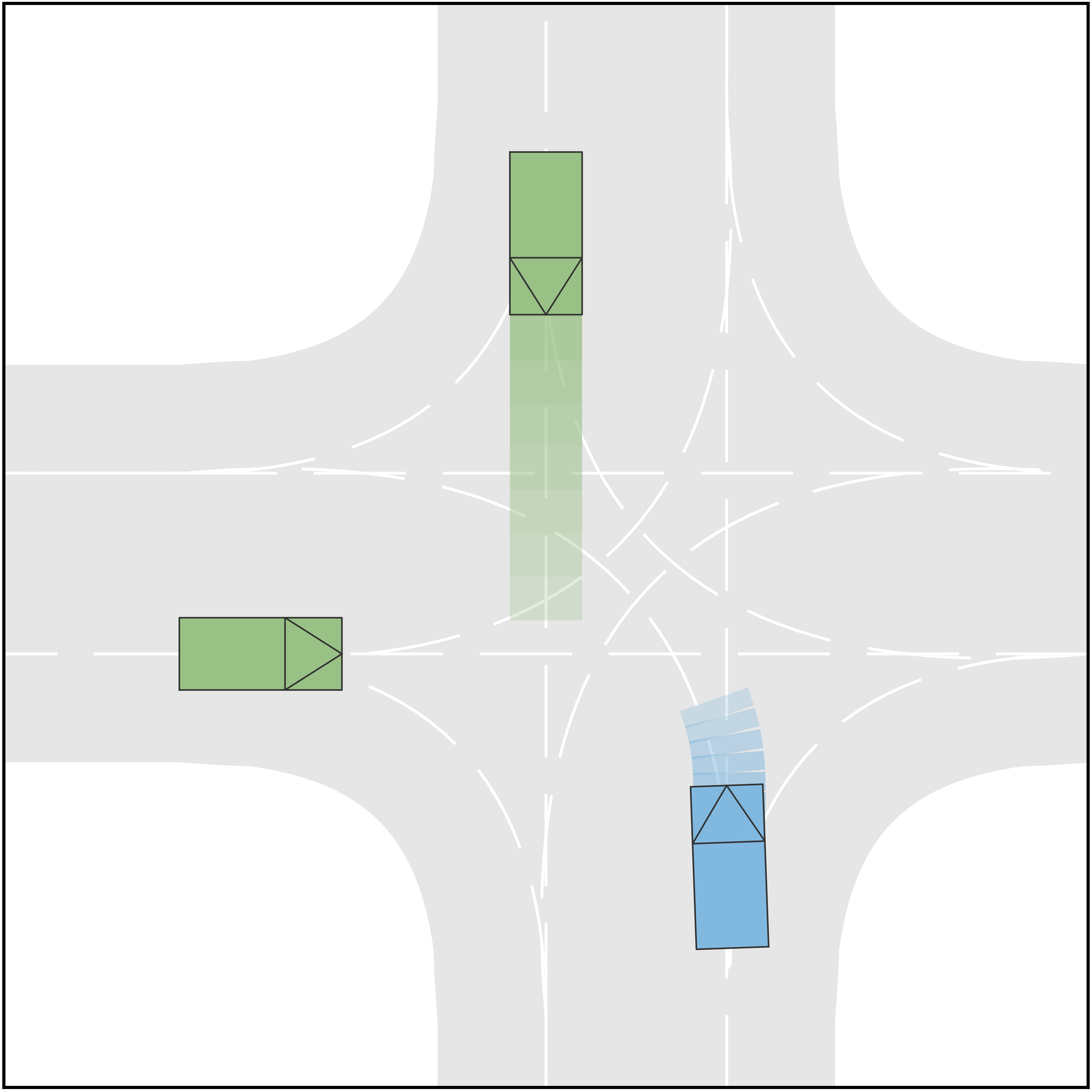}
\includegraphics[scale=0.02209]{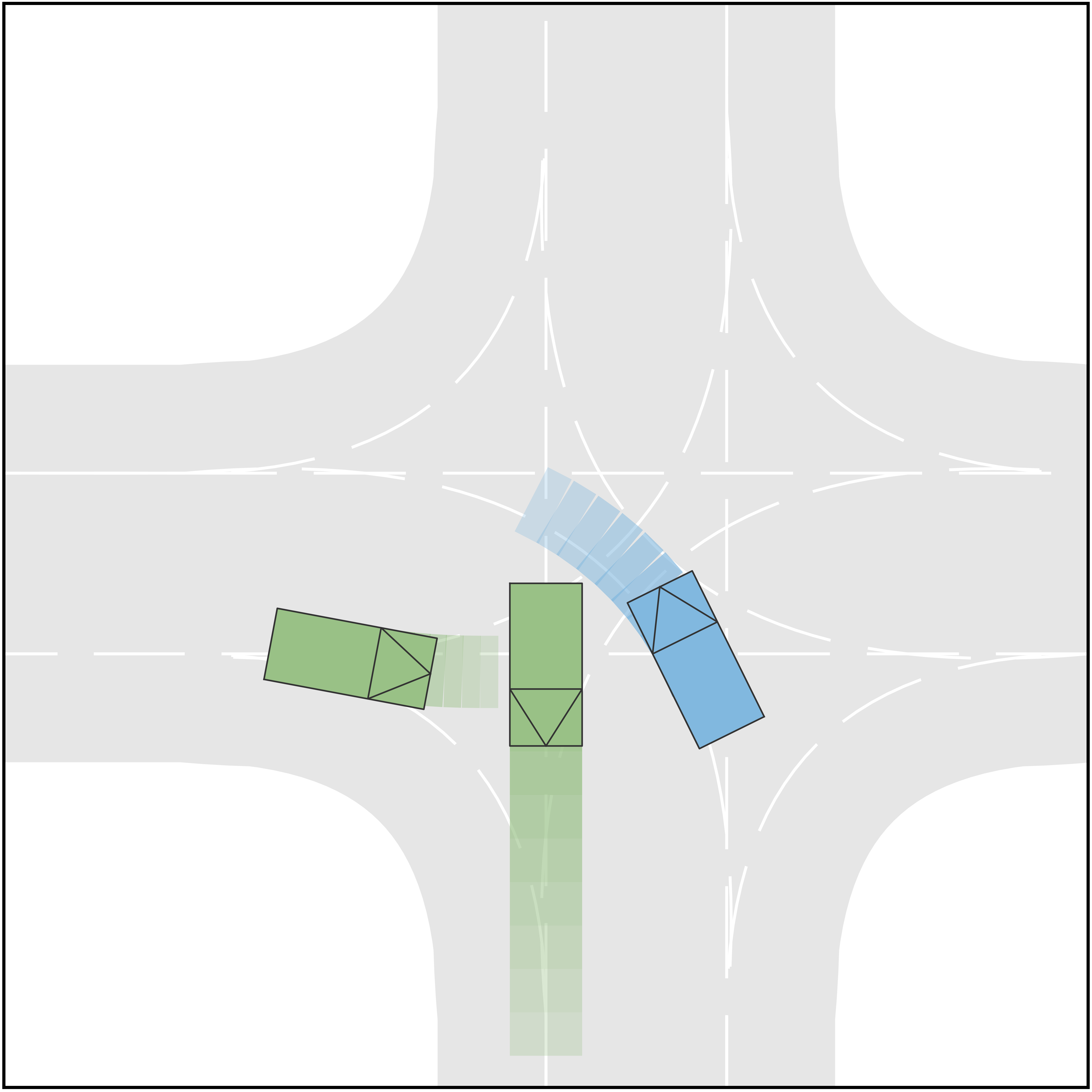}
\includegraphics[scale=0.02209]{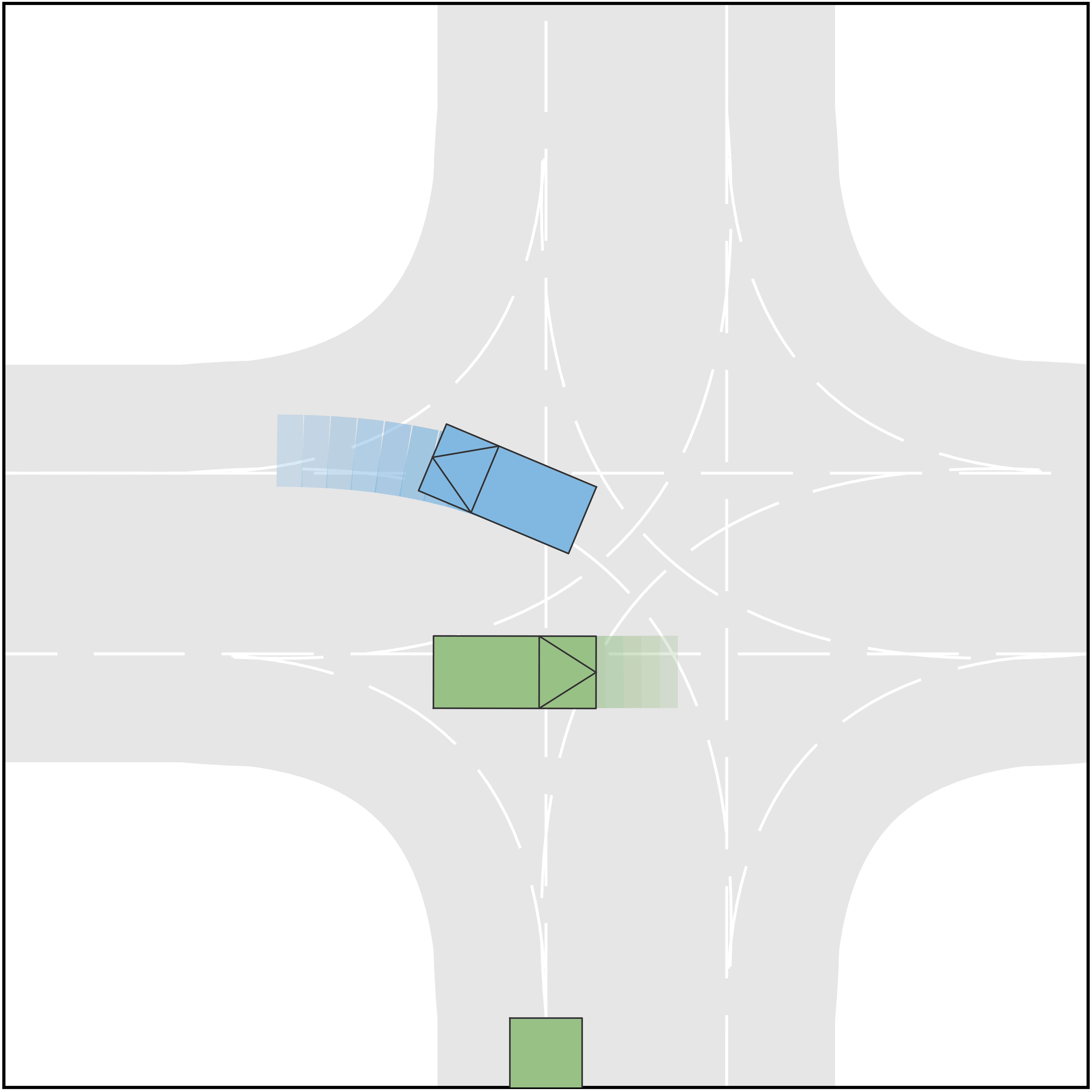}
}
\subfigure[Simulation results of Scenario D at 1.5\,s, 2.5\,s, and 4.0\,s.]{
\includegraphics[scale=0.02209]{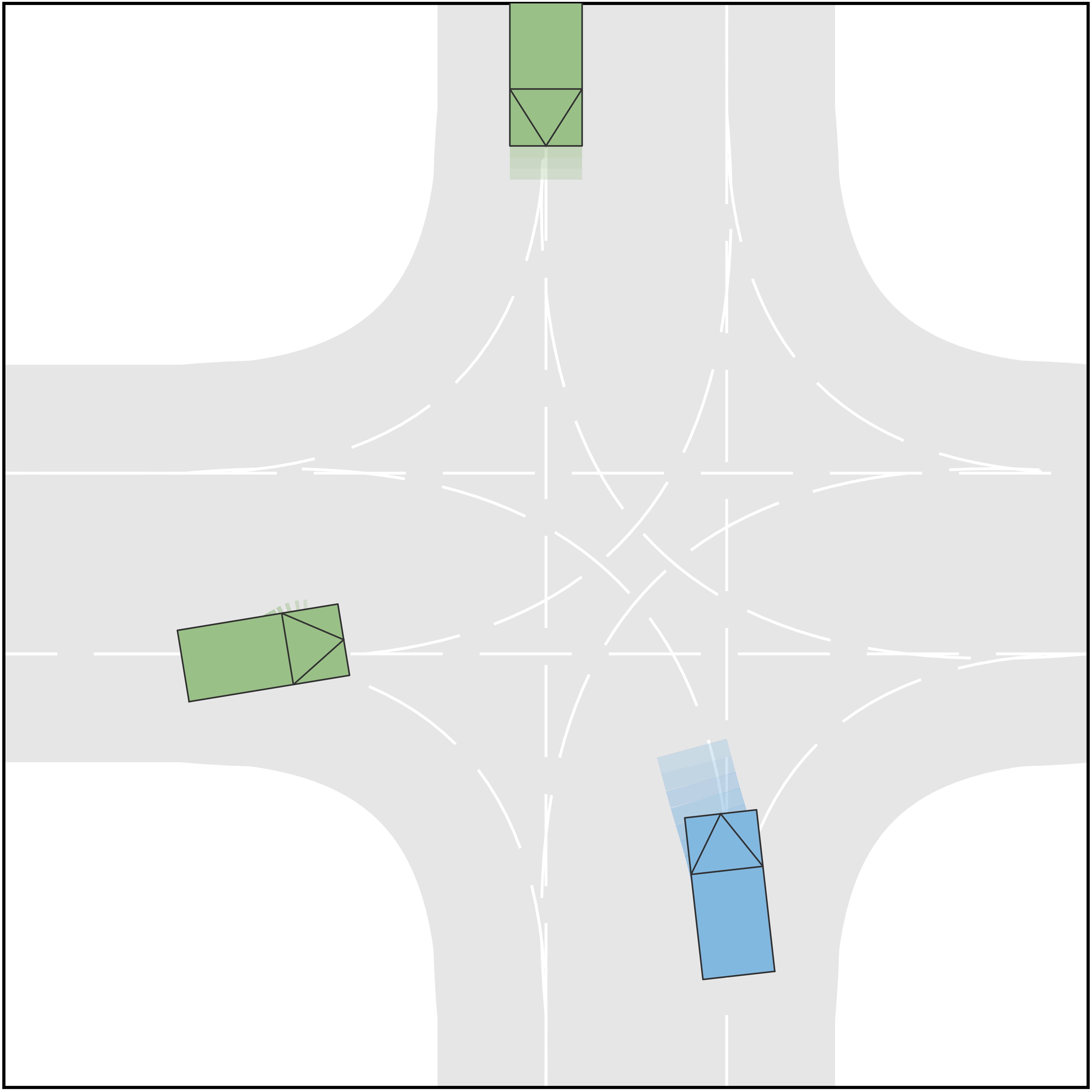}
\includegraphics[scale=0.02209]{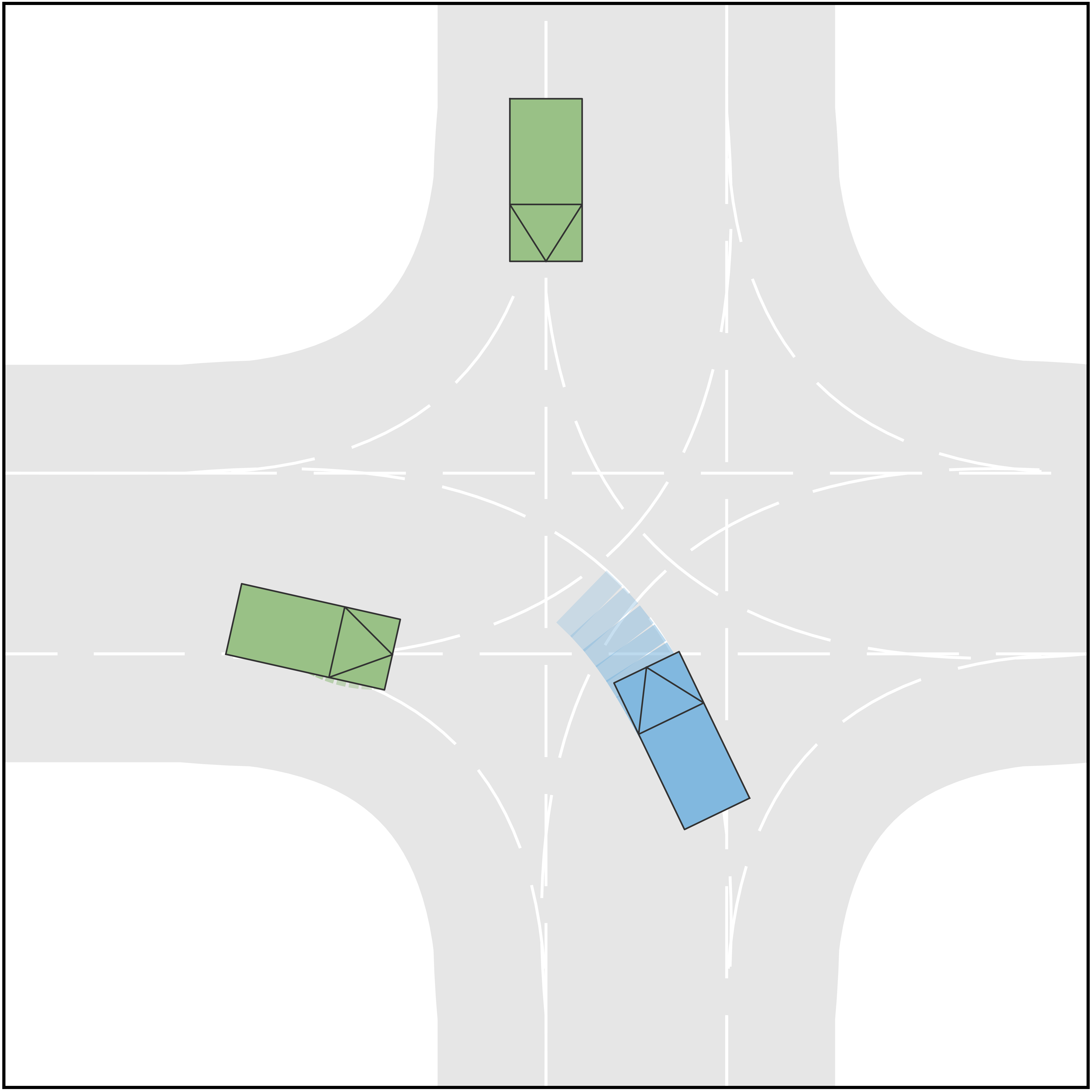}
\includegraphics[scale=0.02209]{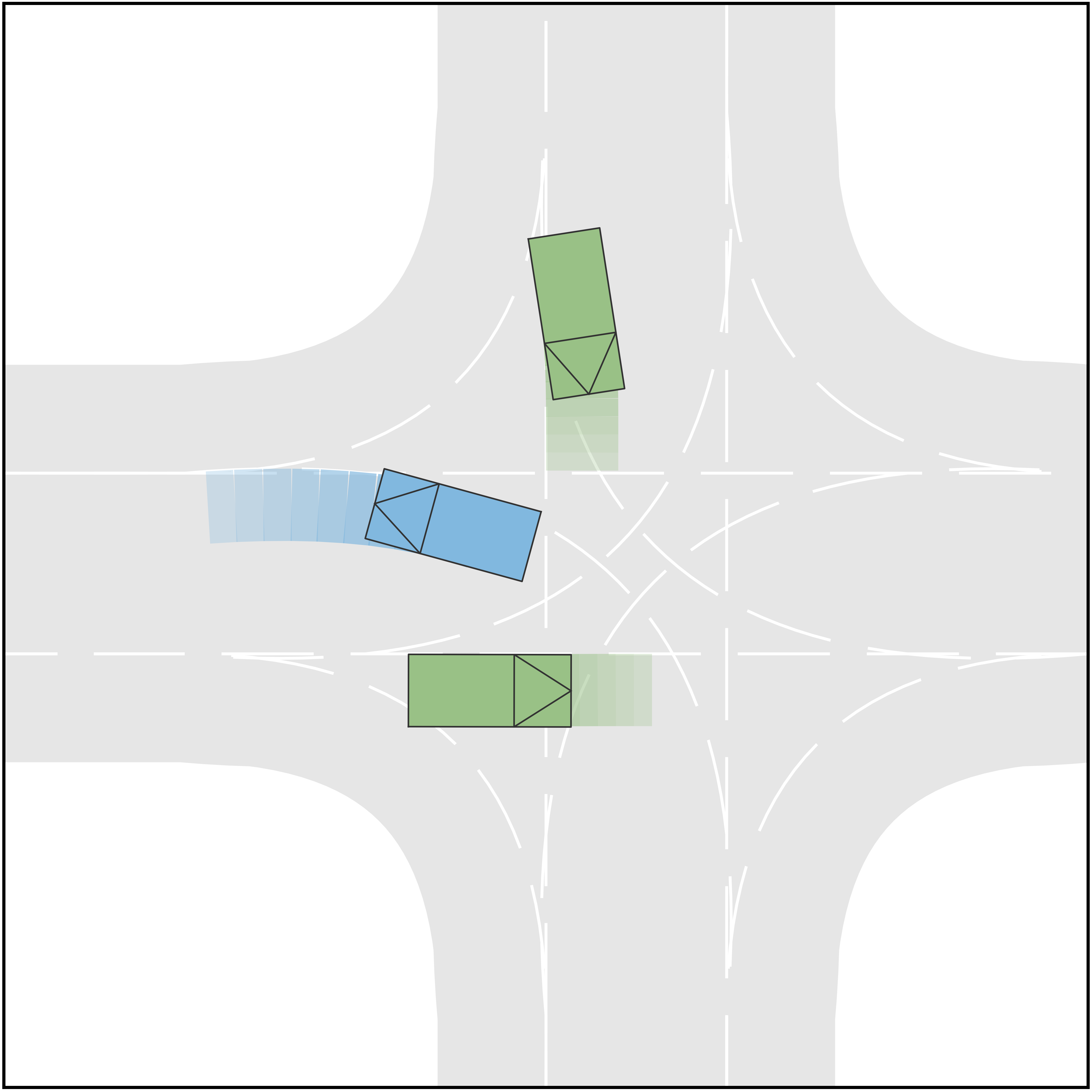}
}
\subfigure[Simulation results of Scenario E at 1.5\,s, 2.5\,s, and 3.5\,s.]{
\includegraphics[scale=0.02209]{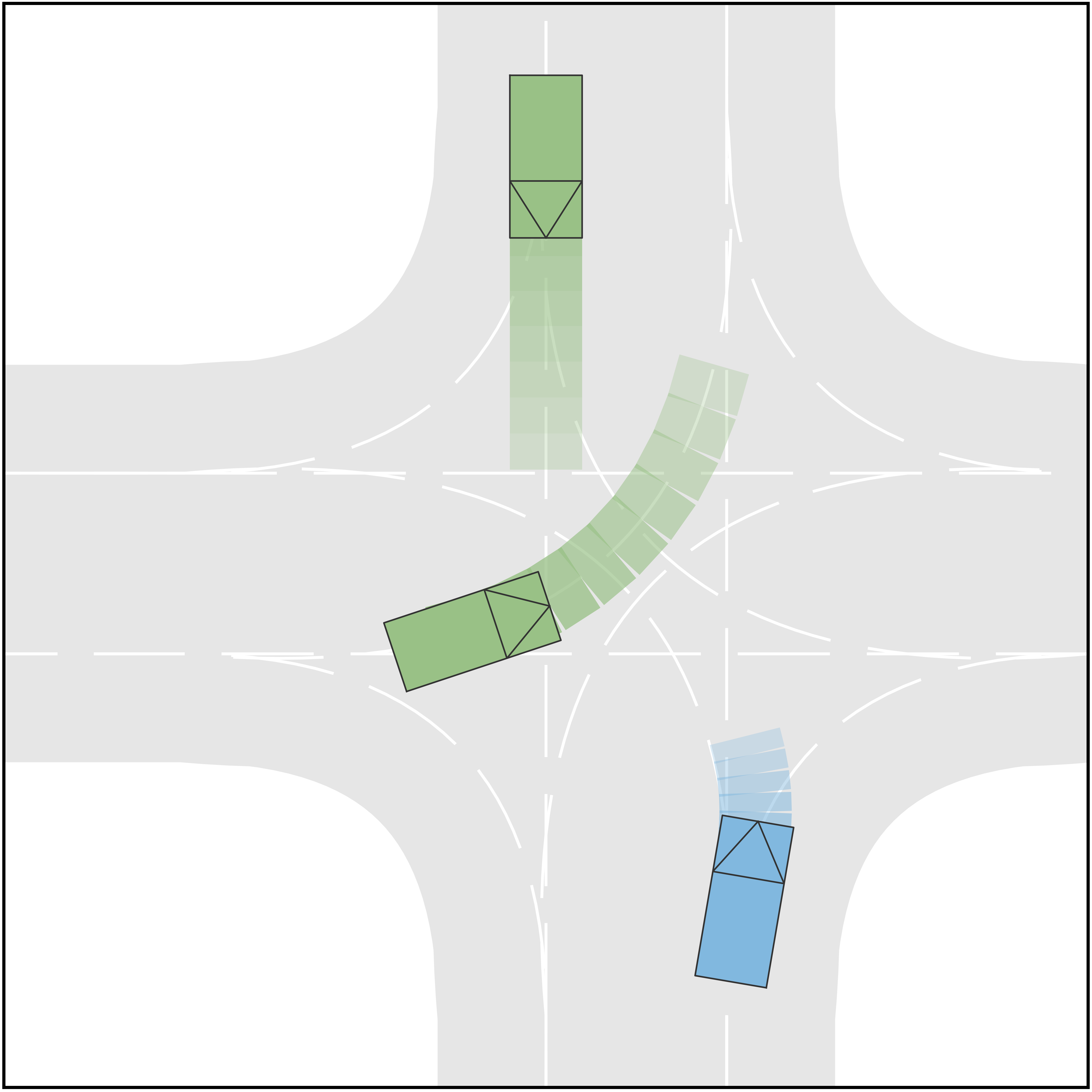}
\includegraphics[scale=0.02209]{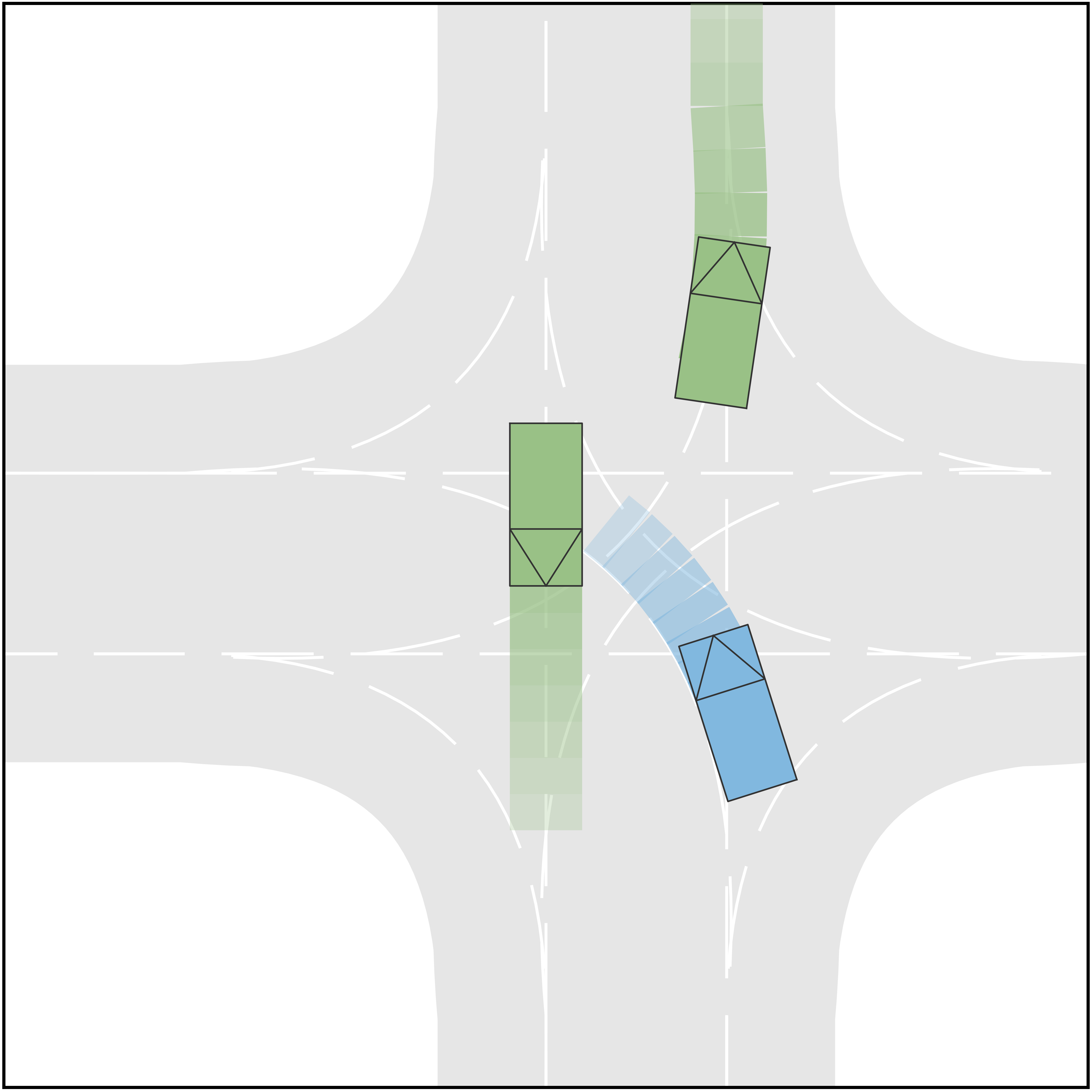}
\includegraphics[scale=0.02209]{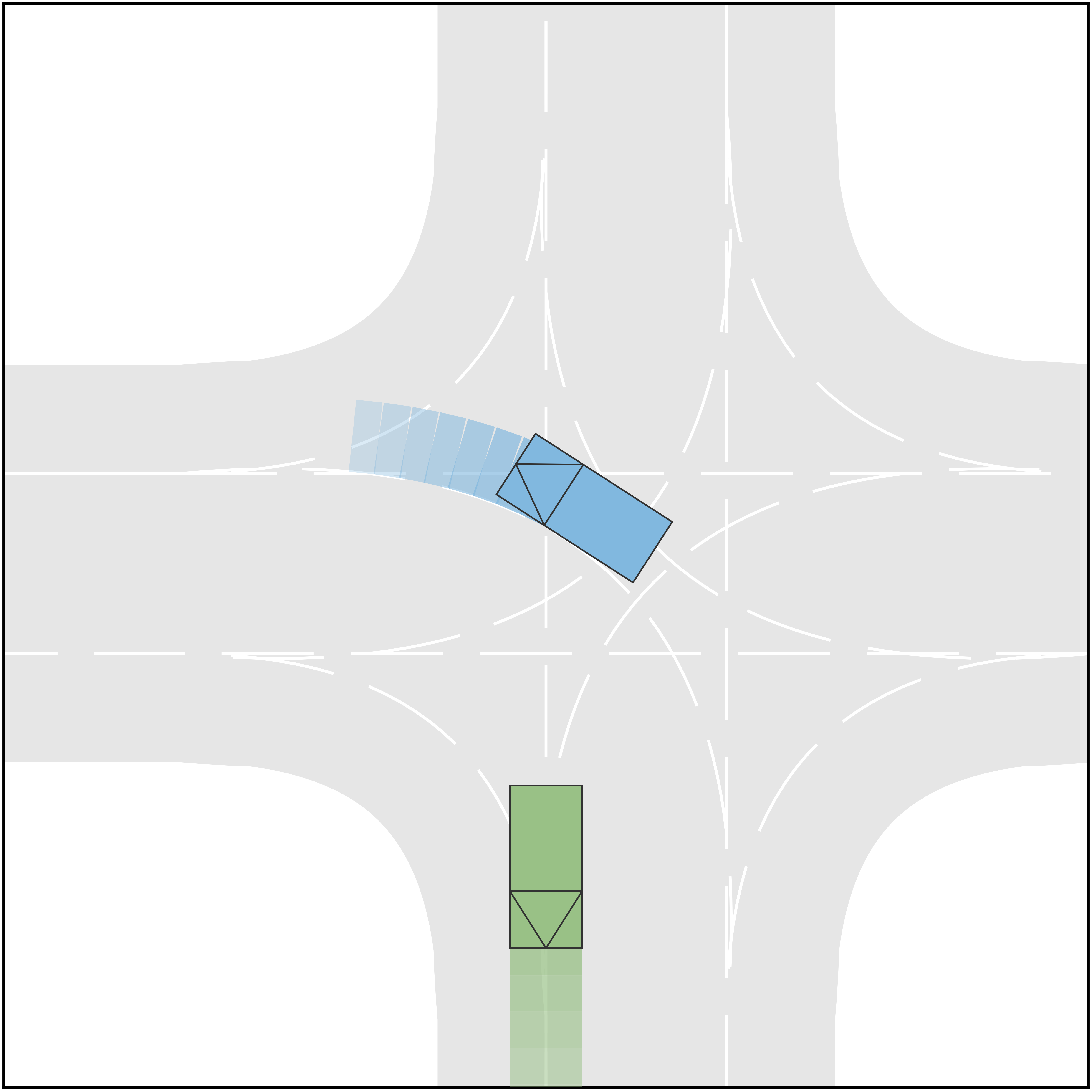}
}
\subfigure[Simulation results of Scenario F at 1.5\,s, 2.5\,s, and 3.5\,s.]{
\includegraphics[scale=0.02209]{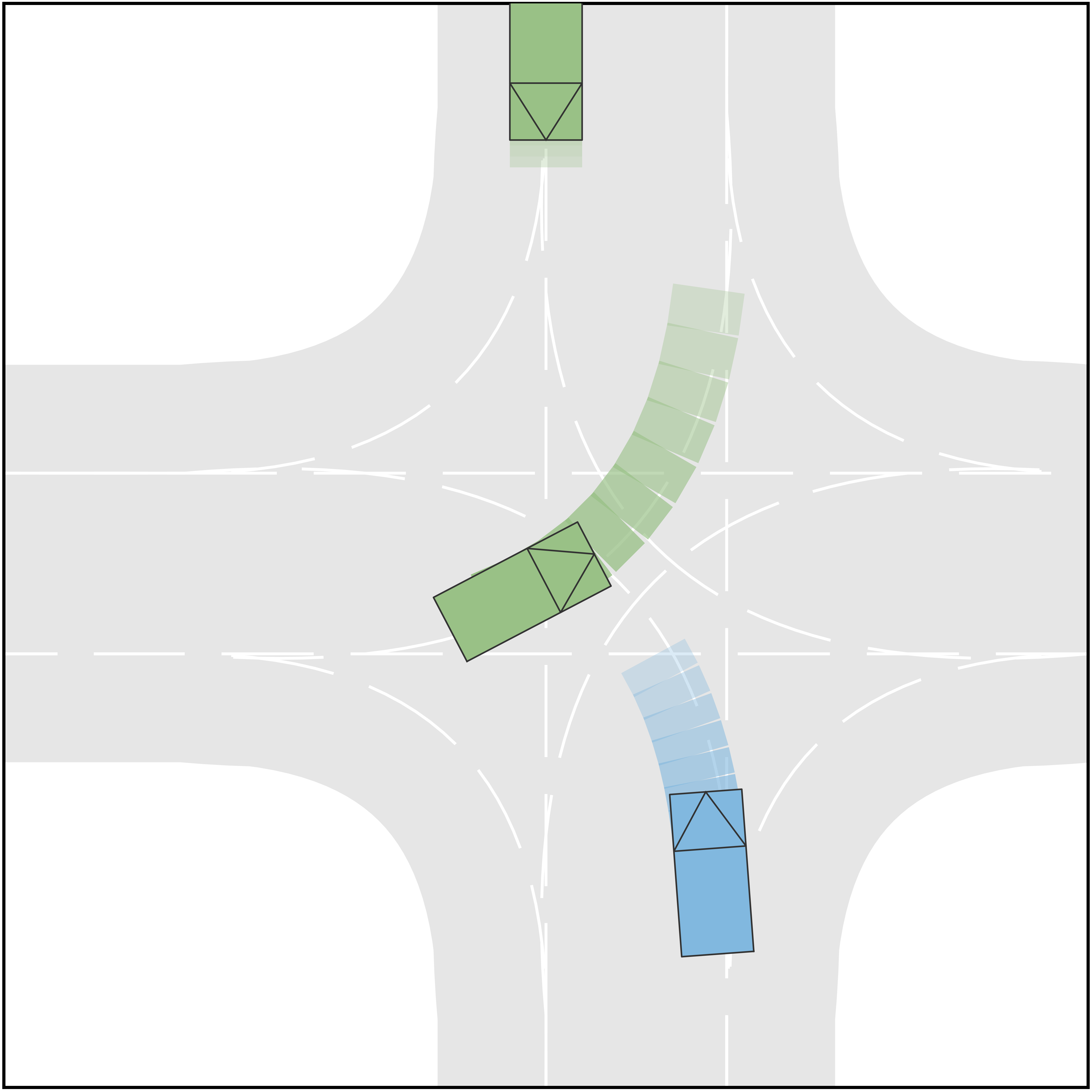}
\includegraphics[scale=0.02209]{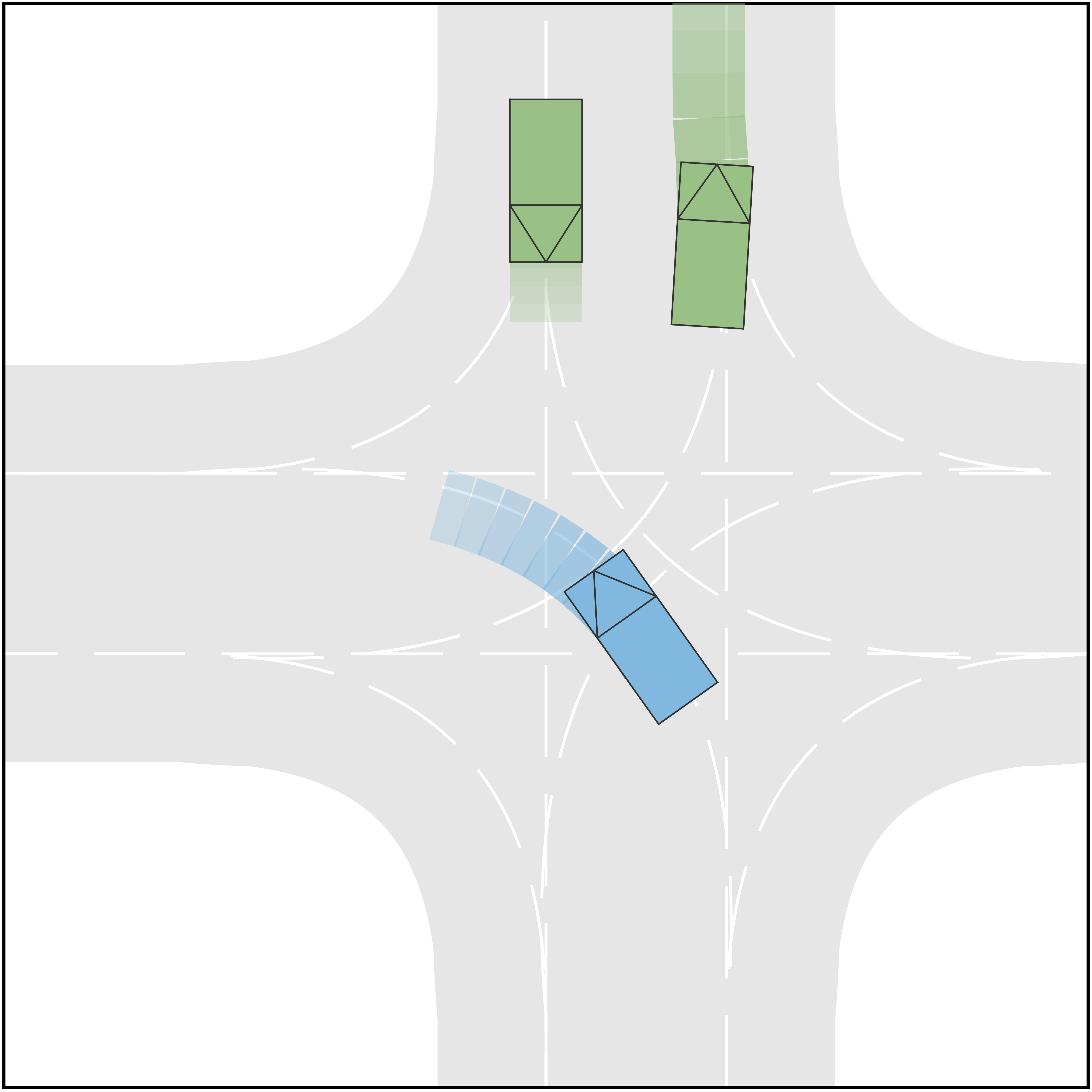}
\includegraphics[scale=0.02209]{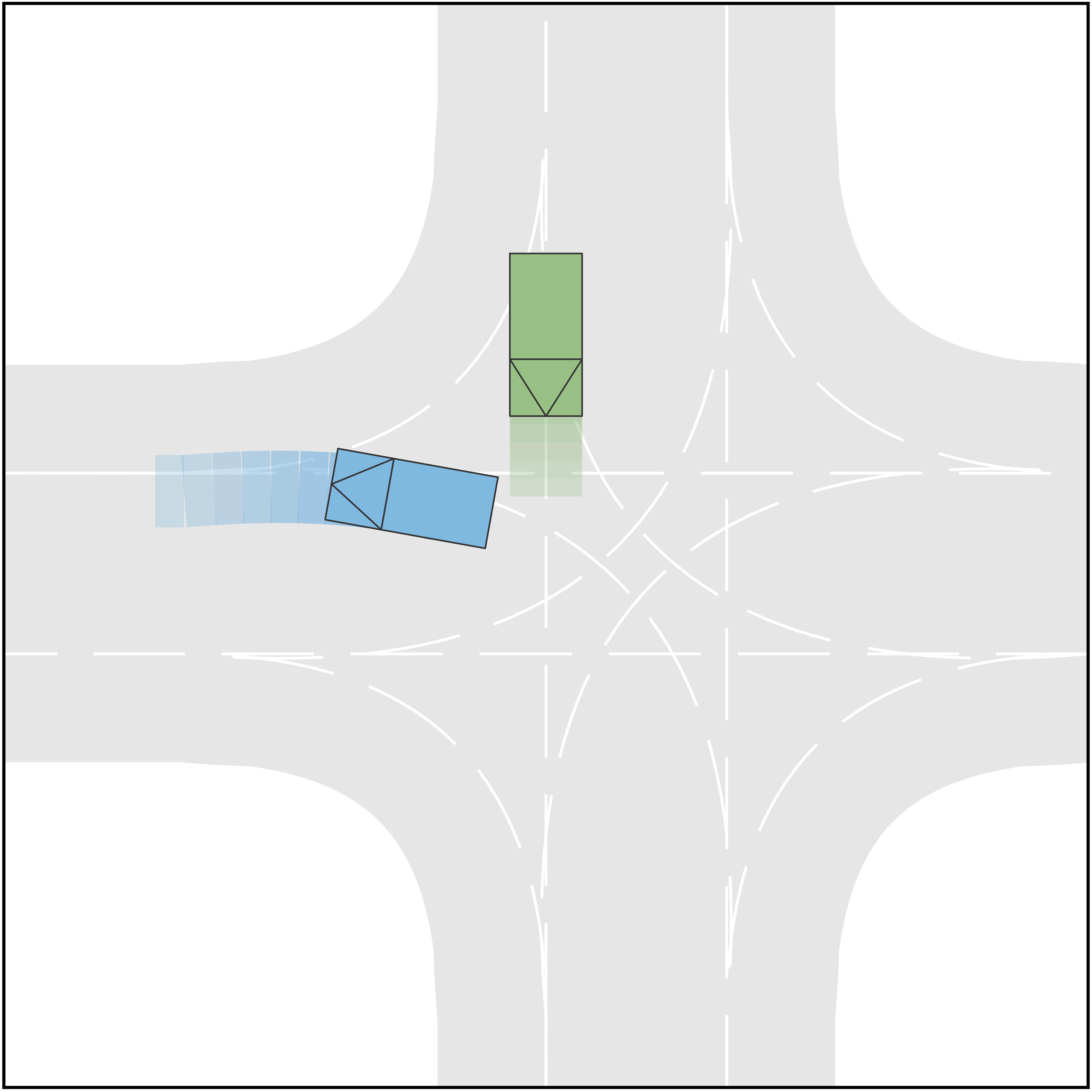}
}
\subfigure[Simulation results of Scenario G at 1.5\,s, 2.5\,s, and 3.5\,s.]{
\includegraphics[scale=0.02209]{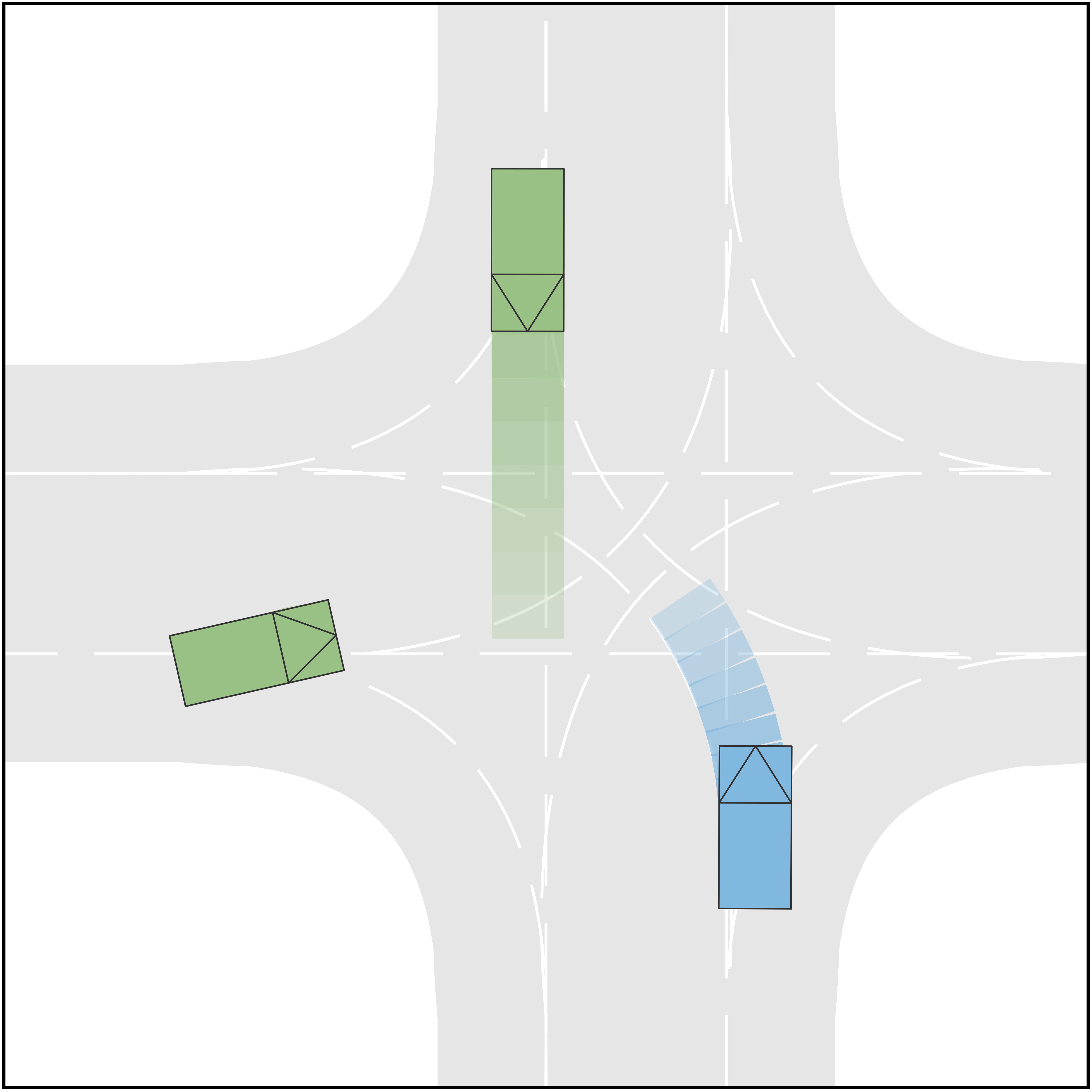}
\includegraphics[scale=0.02209]{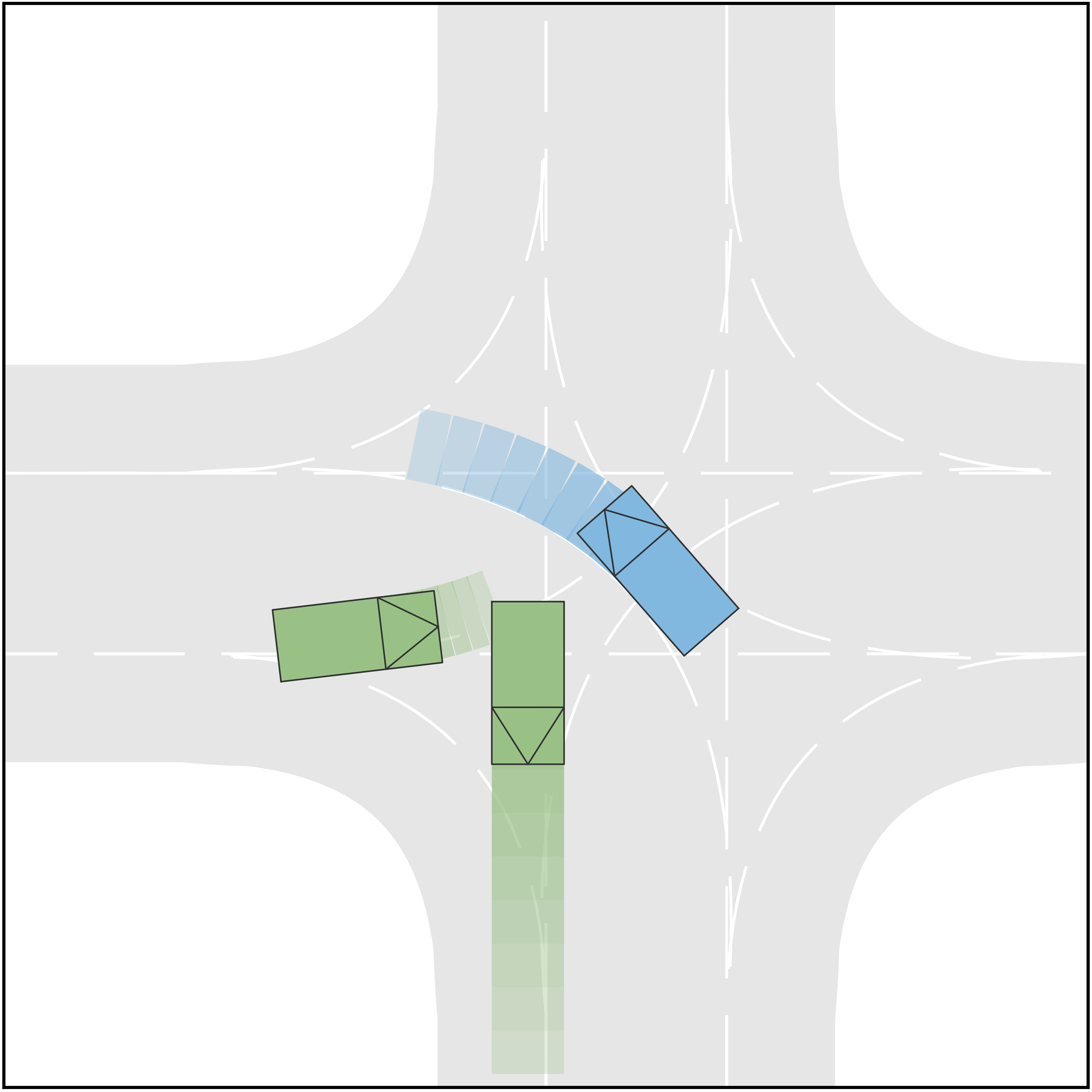}
\includegraphics[scale=0.02209]{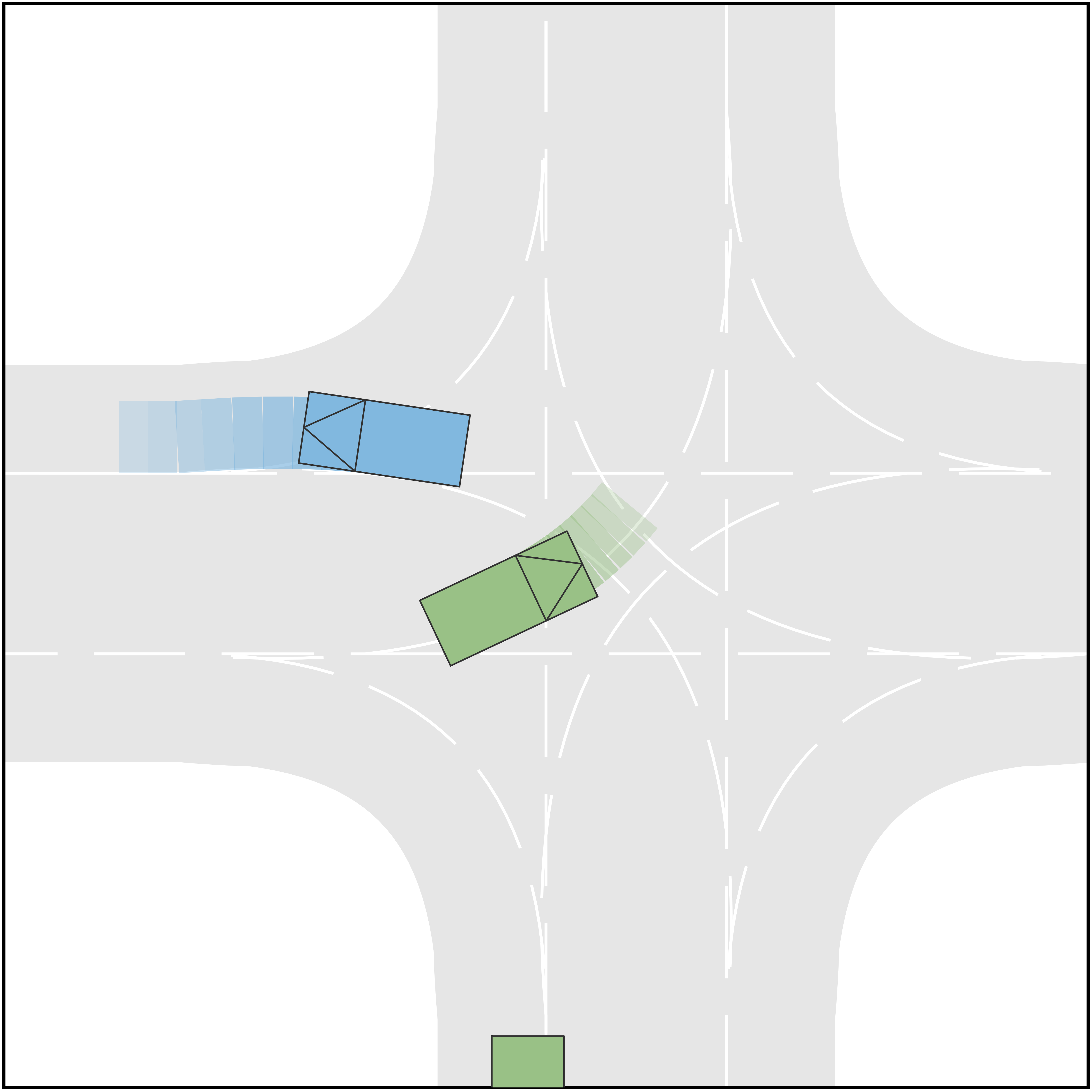}
}
\subfigure[Simulation results of Scenario H at 2.0\,s, 4.0\,s, and 5.0\,s.]{
\includegraphics[scale=0.02209]{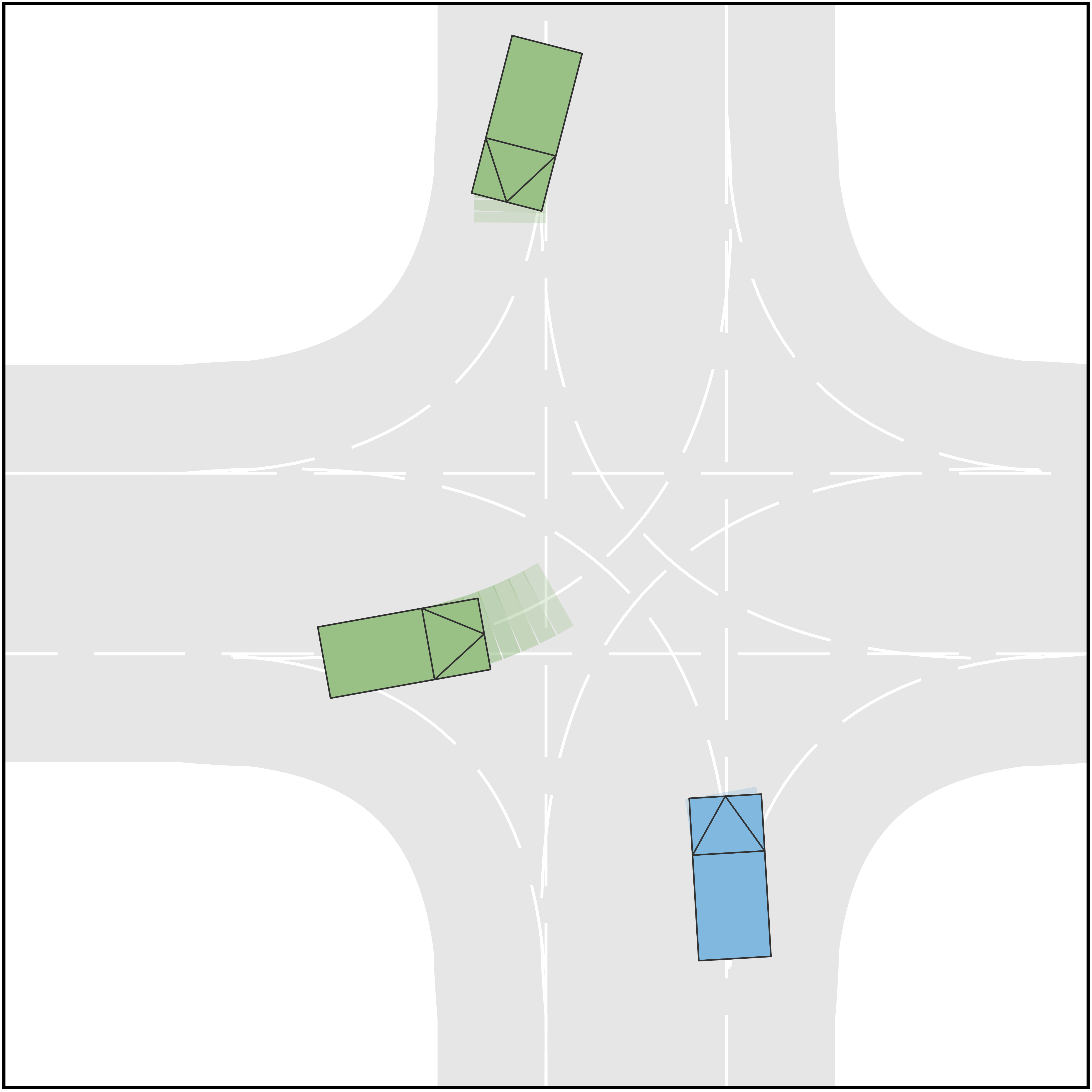}
\includegraphics[scale=0.02209]{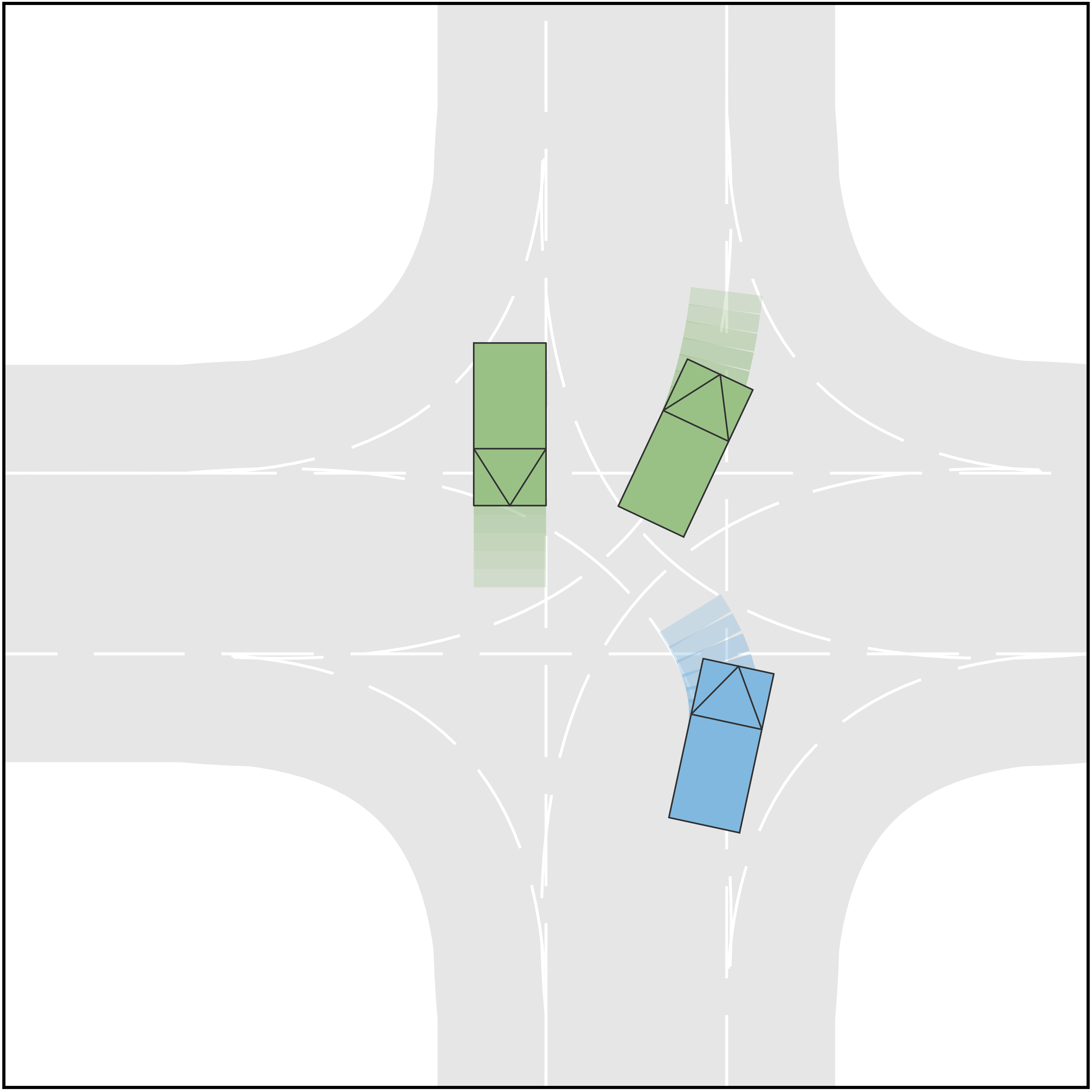}
\includegraphics[scale=0.02209]{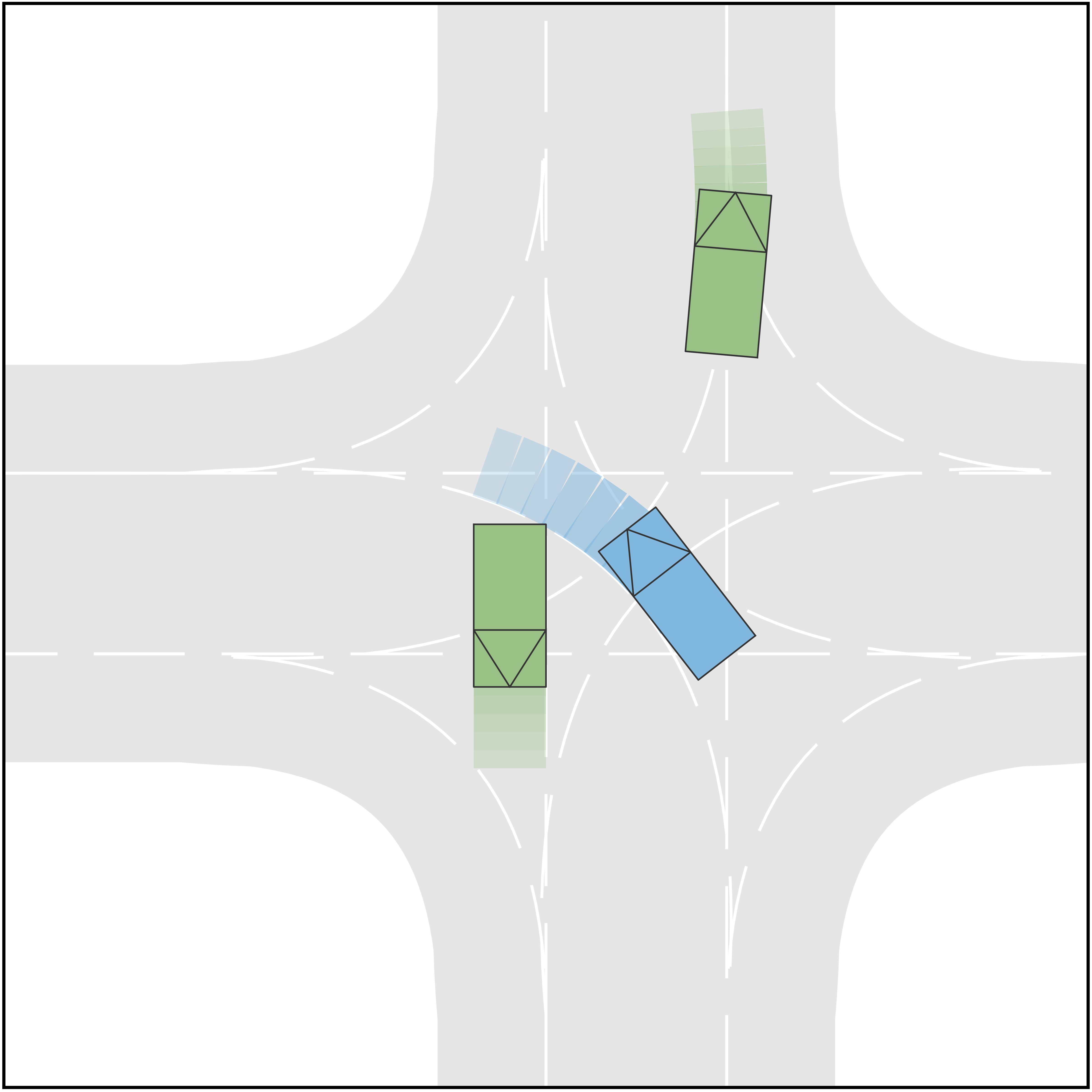}
}

\caption{Simulation results for Case II. Successful unprotected left-turning is performed in all scenarios without collisions. Adaptive strategies of AV are demonstrated in these figures, which react responsively to different driving modes of HVs to ensure security and successful passing.}

\label{fig:case_B}
\end{figure*}

In this case, the goal of AV is to perform an unprotected left-turning in an unsignalized intersection. Two human-driving vehicles are considered in this scenario, including HV1 traveling from the left-hand side of AV and HV2 traveling in the opposite lane. The initial positions of AV, HV1, and HV2 are $(15,-5)\,\textup{m}$, $(-5,10)\,\textup{m}$, and $(10,35)\,\textup{m}$, respectively, and the initial velocity is $7\,\textup{m}/\textup{s}$ for all three vehicles. For each vehicle, four different types are considered. Specifically, those types are combinations of navigation types and longitudinal types. For AV, the types concerning navigation include straight-traveling and left-turning and the types concerning longitudinal strategies include aggressive and conservative, such that AV can be an aggressive straight-traveling agent or a conservative left-turning agent, etc. Although the driving intention of the AV is to perform left-turning, it is clearly not known by HVs, and therefore straight-traveling types of AV also need to be modeled (but not selected). The same types are considered for HV1. For HV2, the navigation types include straight-going and right-turning, while longitudinal types are the same. Since vehicles can swerve away from the reference line to avoid each other in an intersection, the action space should include lateral actions and longitudinal actions. The longitudinal actions are the same as in Case I, while the set of lateral action is the set of eligible longitudinal displacement with respect to the corresponding reference lines $[-1,-0.5,0,0.5,1]\,\textup{m}$. In this case, 8 different scenarios are considered with respect to different types of HVs, and those scenarios are listed in Table II. The time span for the first stage in the trajectory tree is $1\,\textup{s}$ and that for the second stage is $2\,\textup{s}$.

\begin{figure}[t]
\centering
\includegraphics[scale=0.60]{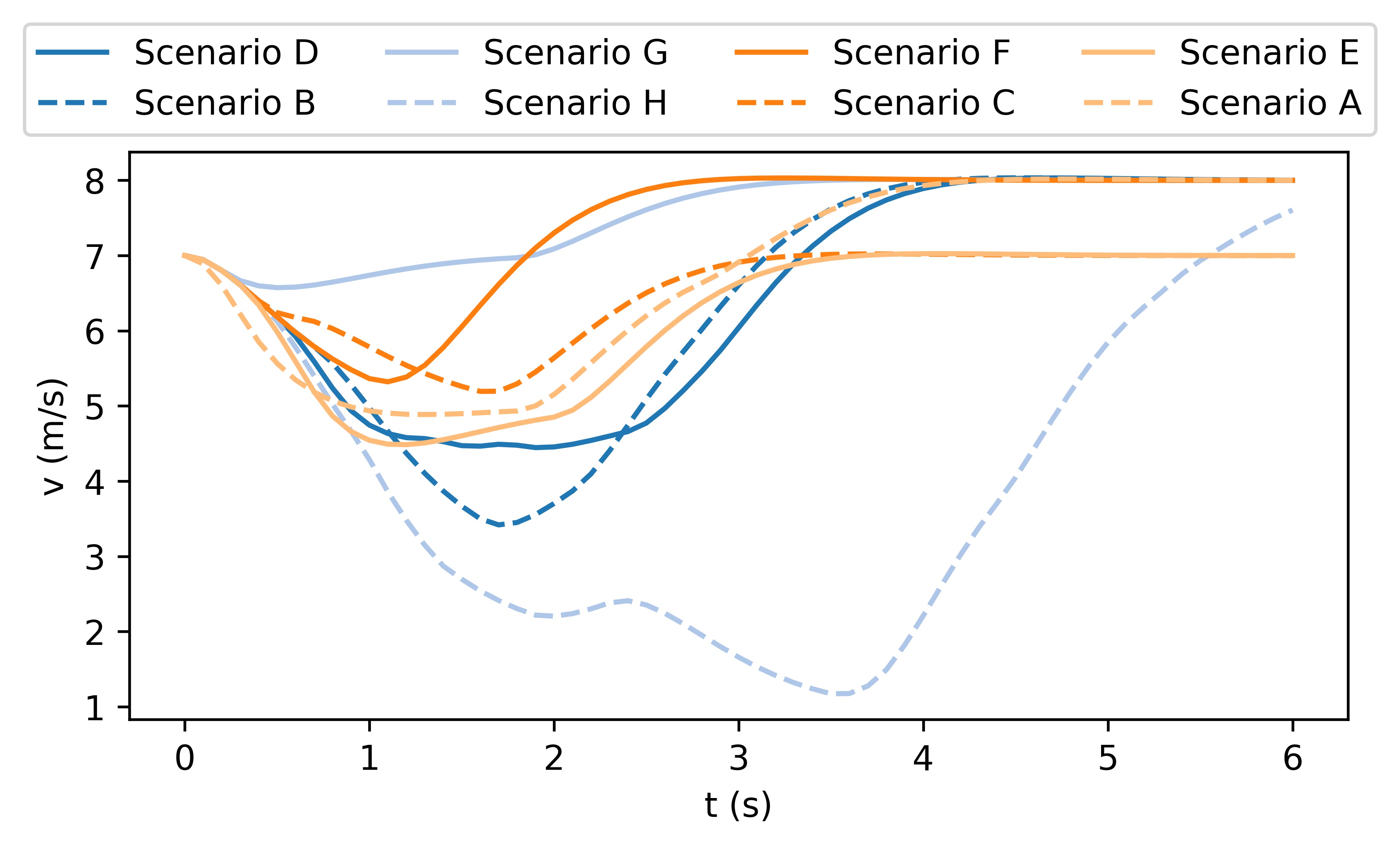}
\caption{Longitudinal velocities of AV in different scenarios in Case II.}
\label{fig:long_velocity_B}
\end{figure}

\begin{figure}[t]
\centering
\includegraphics[scale=0.60]{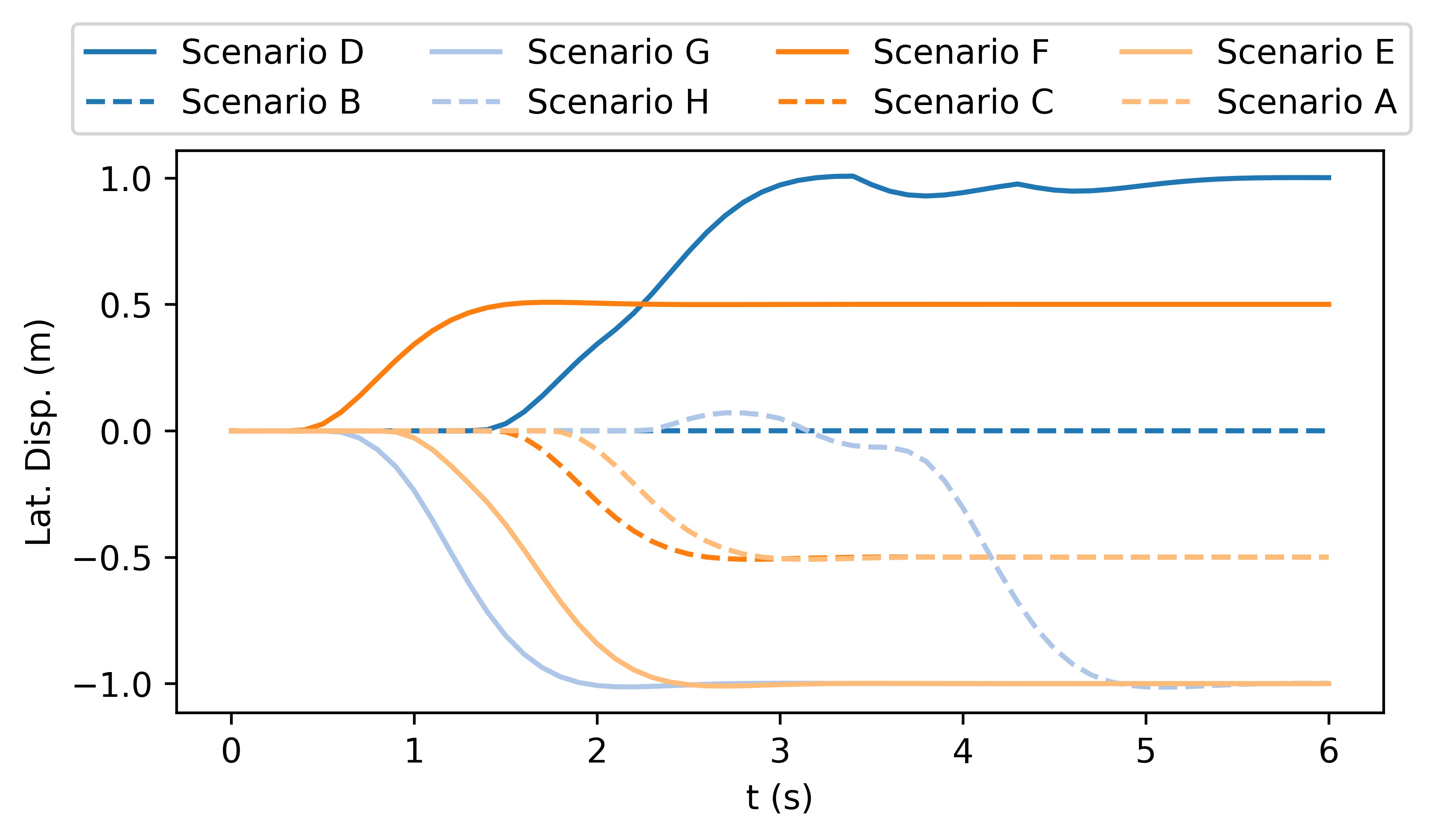}
\caption{Lateral displacements of AV in different scenarios in Case II.}
\label{fig:lat_disp_B}
\end{figure}

\begin{table}[t]
\label{table:types}
\centering
\caption{Scenario settings of Case I}
\begin{tabular}{ccc}
\hline
Scenario   & HV1                          & HV2                          \\ \hline
         A & Straight-going, Aggressive   & Straight-going, Aggressive   \\
         B & Straight-going, Aggressive   & Straight-going, Conservative \\
         C & Straight-going, Conservative & Straight-going, Aggressive   \\
         D & Straight-going, Conservative & Straight-going, Conservative \\
         E & Left-turning, Aggressive     & Straight-going, Aggressive   \\
         F & Left-turning, Aggressive     & Straight-going, Conservative \\
         G & Left-turning, Conservative   & Straight-going, Aggressive   \\
         H & Left-turning, Conservative   & Straight-going, Conservative \\ \hline
\end{tabular}
\end{table}

The qualitative simulation results are shown in Fig. \ref{fig:case_B}. In general, AV successfully identifies the intentions and types of other vehicles and manages to pass the scenario in a smooth and secure manner. Diversified driving behaviors are exhibited with respect to different driving intentions of HVs. Specifically, in Scenario A where both HVs are aggressive and straight-going, AV slows down and passes through the intersection after both HVs. In Scenario B, AV slows down to yield to the HV1, which is straight-going and aggressive. Then it accelerates to pass in front of HV2, which is considered conservative. In Scenario C where HV1 is straight-going and conservative and HV2 is straight-going and aggressive, the AV swerve to the right to pass through the intersection from behind HV2. In Scenario D where both HVs are straight-going and conservative, AV accelerates and cuts across in front of both HVs. The behavior of AV in Scenario E, F, and G are similar to that in Scenario A, B, and C, respectively. In Scenario H, although both HVs are conservative, HV1 is performing a left-turning and therefore there is no room for AV to cut across. As a result, AV slows down to pass the intersection behind both HVs. In addition to the behaviors of AV, the proposed method also manages to simulate the interactive behaviors of the HVs. For example, in Scenario A, both HVs swerve to the left to keep safety distances from AV, and in Scenario B, HV2 steers to the right to pass through the intersection behind AV, etc. Fig. \ref{fig:long_velocity_A} shows the longitudinal velocities of AV in each scenario, and Fig. \ref{fig:lat_disp_B} shows the lateral displacements. In general, AV decelerates at the beginning to avoid collision when the intentions of HVs remain unclear and then takes responsive actions to different driving behaviors of HVs.

\subsection{Comparison}

\begin{figure}[t]
\centering
\subfigure[$t=1.0\,\textup{s}$]{\includegraphics[scale=0.0355]{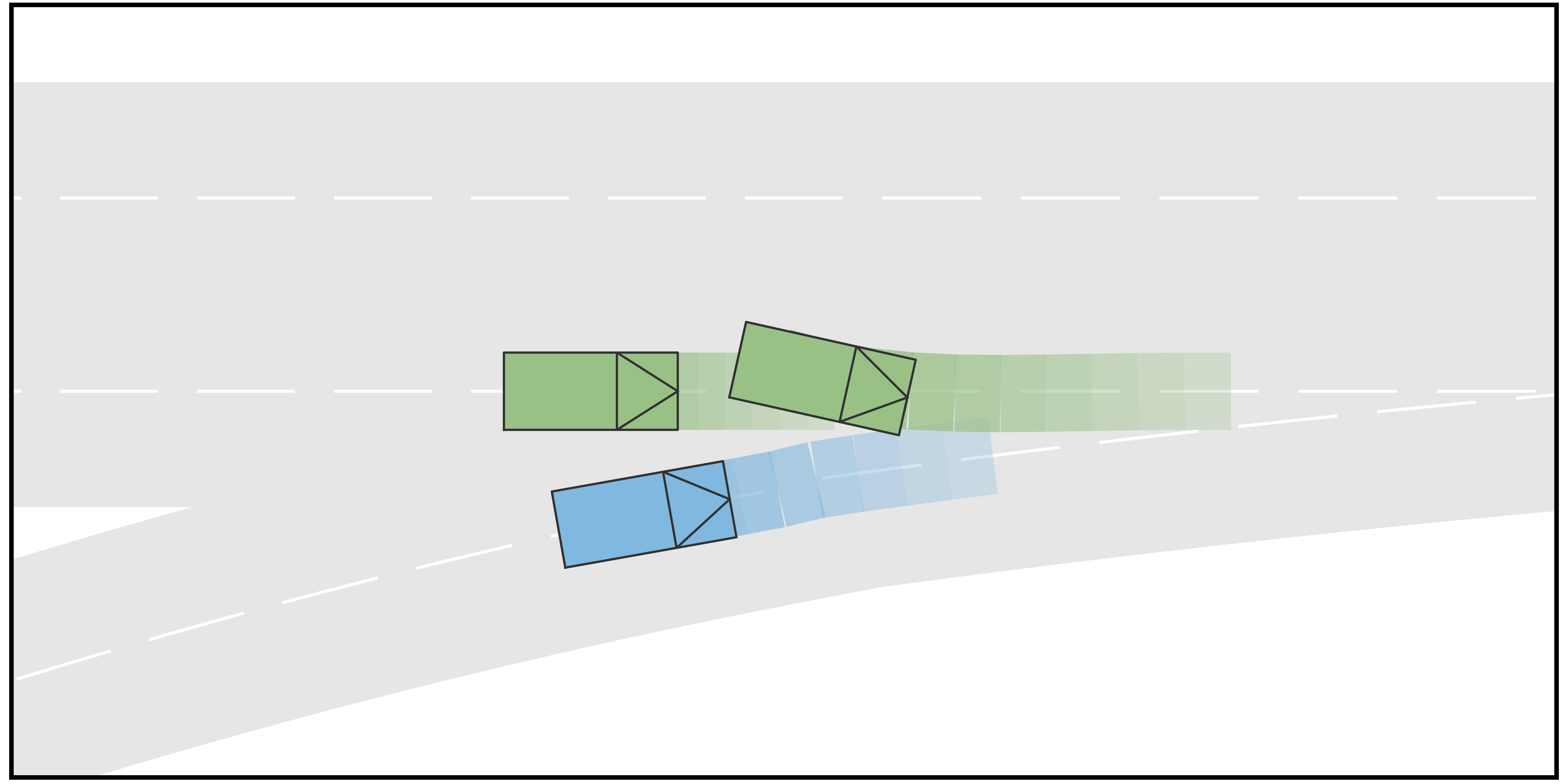}}
\subfigure[$t=1.5\,\textup{s}$]{\includegraphics[scale=0.0355]{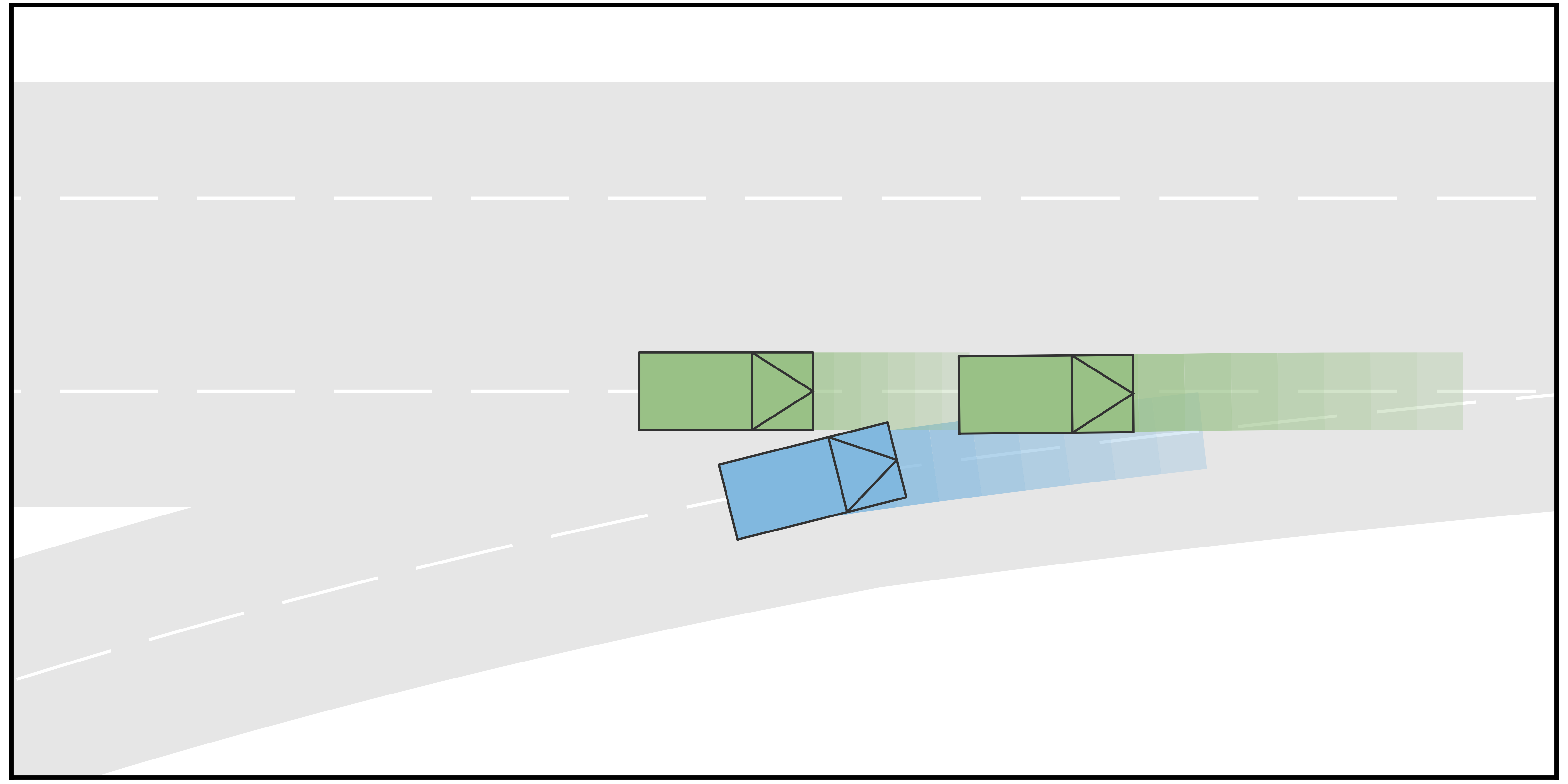}}
\subfigure[$t=2.0\,\textup{s}$]{\includegraphics[scale=0.0355]{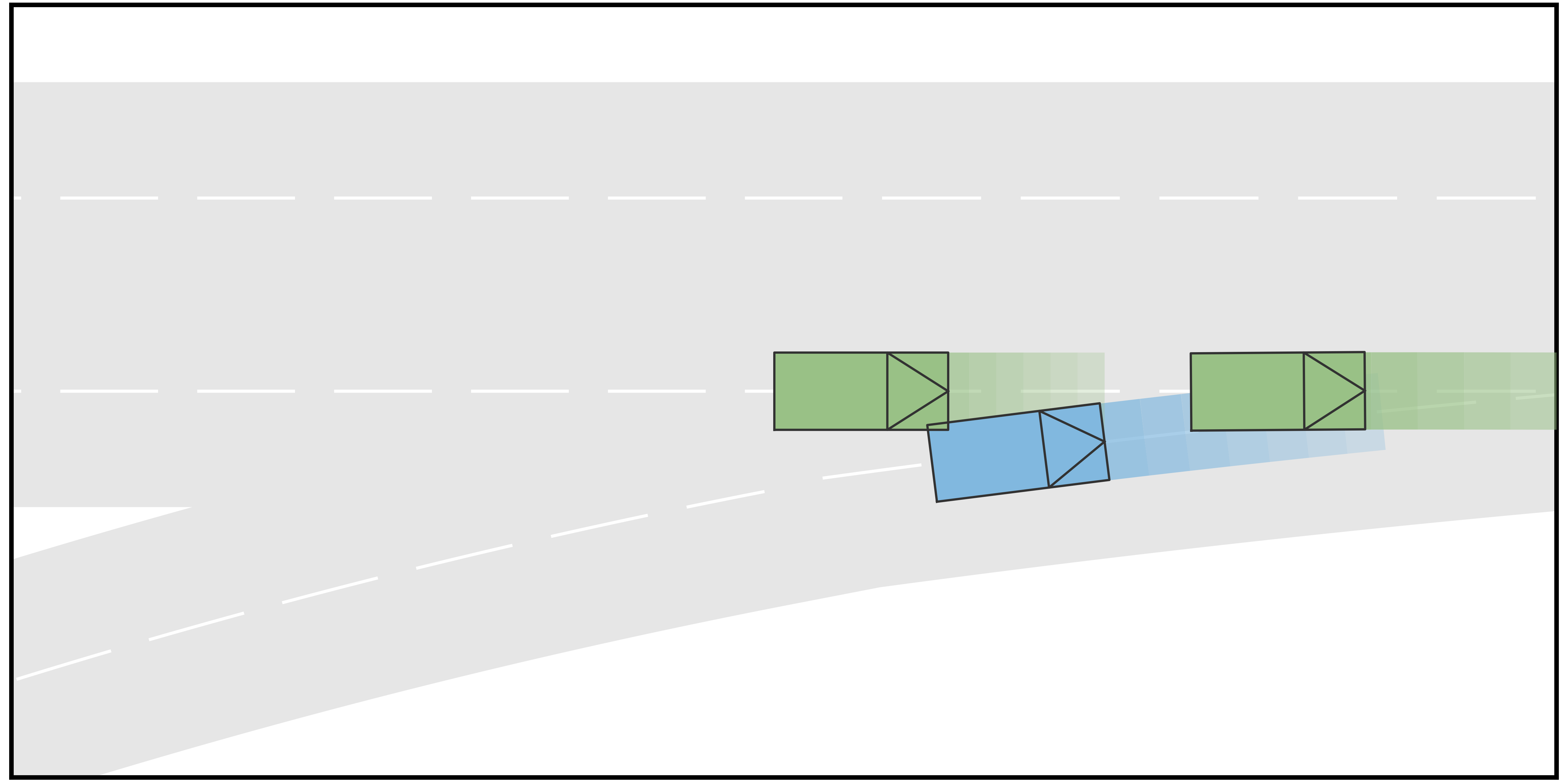}}
\subfigure[$t=2.5\,\textup{s}$]{\includegraphics[scale=0.0355]{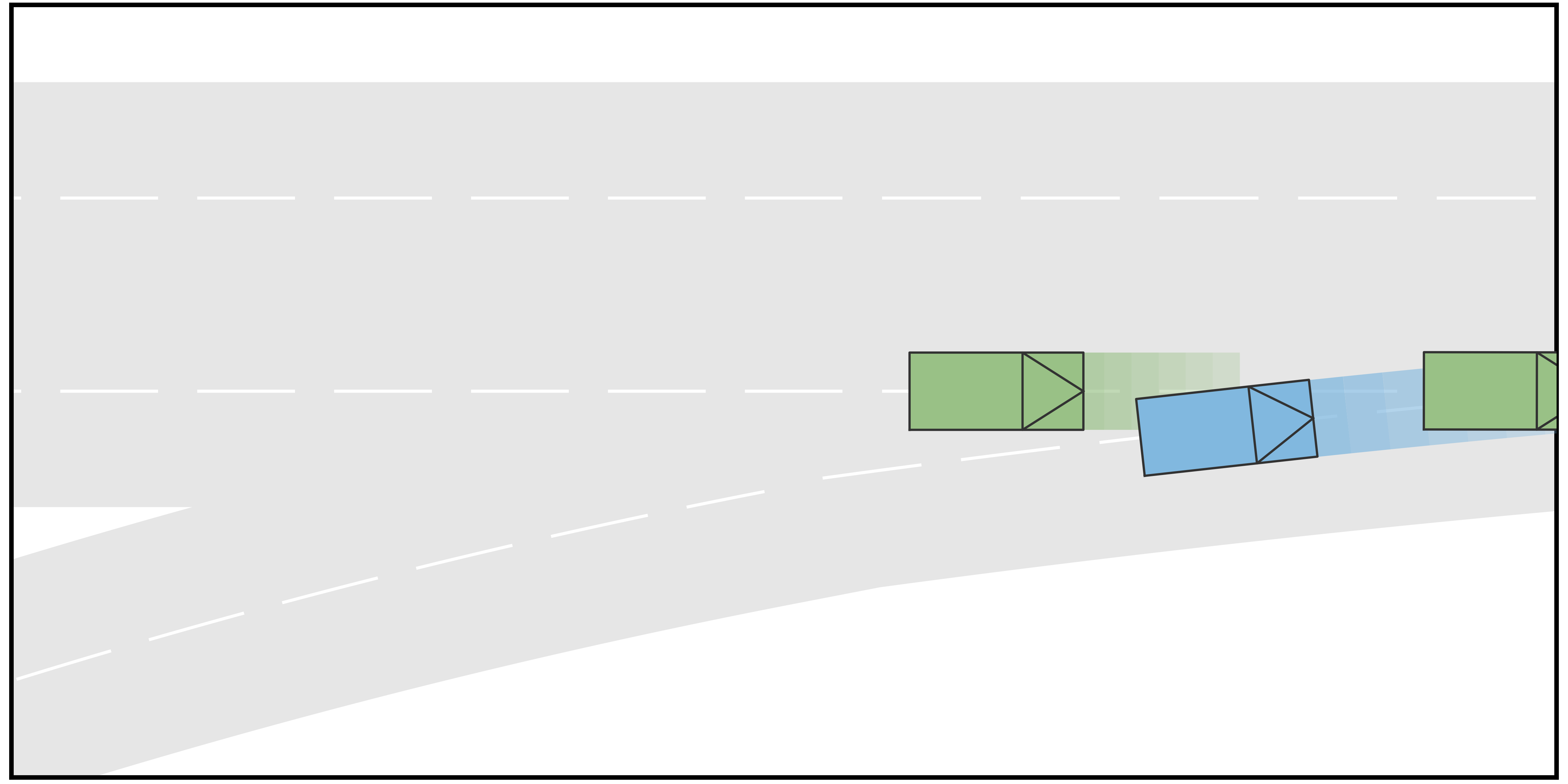}}

\caption{A failure mode of the baseline method corresponding to Scenario B, Case I. Due to ignorance of the driving modes of HVs, the AV fails to yield to the aggressive rear HV and therefore collision occurs.}

\label{fig:failure_case}
\end{figure}

To further illustrate the advantage of the proposed method, we compared it with a planning scheme provided by a standard game of complete information. Specifically, the AV is trying to play a game of complete information, in which the action space of each agent is exactly the union of the action spaces of that agent under different types, and other settings are the same as in the game of incomplete information. Meanwhile, all other HVs are still playing the game of incomplete information with determined type. Simple qualitative results are shown in Fig. \ref{fig:failure_case}. In Scenario B of the ramp-merging case, since the longitudinal position of HV1 is $2\,\textup{m}$ behind AV and the initial velocities are the same for both vehicles, a typical Nash equilibrium is that AV will merge into the main road in front of the HV1. However, this equilibrium contradicts the actual driving intention of HV1, which is an aggressive agent and intends to accelerate. Due to this discrepancy, a catastrophic result occurs such that AV collides with HV1. This example illustrates the limitations of a traditional game of complete information and demonstrates the necessity of identifying the intentions of other vehicles.

We also present a quantitative analysis to compare the performance of the proposed method and the baseline of the traditional game. Since randomness exists in Monte-Carlo sampling, we perform repeated experiments for both cases with scenario settings described in Section V-A and Section V-B. For passenger comfort, we compare the following statistics of AV: 1) the root mean square of longitudinal acceleration, 2) the average maximum longitudinal acceleration, 3) the root mean square of lateral acceleration, and 4) the average maximum lateral acceleration. For driving safety, we compare 1) the average minimum distances between AV and All of the HVs and 2) the collision rates between AV and HV. Results are shown in Table III. It can be seen that in Case I, the collision rate of the baseline is unacceptable. Although the performance in terms of passenger comfort seems to be better, this is mainly because the baseline method chooses the theoretically optimal strategy disregarding the actual driving intentions of the HVs. In Case II, the proposed method is clearly better than the baseline method in terms of both passenger comfort and driving safety. Meanwhile, we also compare the computation time of both methods. It can be seen from Table IV that the computation efficiency of the proposed method is higher than the baseline, mainly due to the smaller action space.

% Please add the following required packages to your document preamble:
% \usepackage{multirow}
\begin{table*}[t]
\centering
\caption{Comparison of performance between the proposed method and the baseline}
\begin{tabular}{cccccccc}
\hline
Case               & Method   & Avg. Max. Long. Acc.             & RMS. Long. Acc.                  & Avg. Max. Lat. Acc.              & RMS. Lat. Acc.                   & Min. Dis.           & Collision Rate \\ \hline

    \multirow{2}{*}{I} & Proposed & $5.204\,\textup{m}/\textup{s}^2$ & $2.172\,\textup{m}/\textup{s}^2$ & / & / & $\textbf{0.716}\,\textup{m}$ & $\textbf{0}/8$  \\
                   & Baseline & $\textbf{3.235}\,\textup{m}/\textup{s}^2$ & $\textbf{1.430}\,\textup{m}/\textup{s}^2$ & / & / & $0.333\,\textup{m}$ & $4/8$          \\ \hline
                   
\multirow{2}{*}{II} & Proposed & $\textbf{3.382}\,\textup{m}/\textup{s}^2$ & $\textbf{1.509}\,\textup{m}/\textup{s}^2$ & $\textbf{2.238}\,\textup{m}/\textup{s}^2$ & $\textbf{0.756}\,\textup{m}/\textup{s}^2$ & $\textbf{1.069}\,\textup{m}$ & $\textbf{0}/40$         \\
                   & Baseline & $4.953\,\textup{m}/\textup{s}^2$ & $1.915\,\textup{m}/\textup{s}^2$ & $2.887\,\textup{m}/\textup{s}^2$ & $0.914\,\textup{m}/\textup{s}^2$ & $0.547\,\textup{m}$ & $15/40$        \\ \hline
\end{tabular}
\end{table*}

% Please add the following required packages to your document preamble:
% \usepackage{multirow}
\begin{table}[]
\centering
\caption{Comparison of computation efficiency between the proposed method and the baseline}
\begin{tabular}{ccccc}
\hline
Case               & Method   & 10000 Iter. & 20000 Iter. & 50000 Iter. \\ \hline
\multirow{2}{*}{I} & Proposed & $\textbf{0.09}\,\textup{s}$ & $\textbf{0.14}\,\textup{s}$ & $\textbf{0.28}\,\textup{s}$ \\
                   & Baseline & $0.10\,\textup{s}$ & $0.15\,\textup{s}$ & $0.31\,\textup{s}$ \\ \hline
\multirow{2}{*}{II

} & Proposed & $\textbf{0.09}\,\textup{s}$ & $\textbf{0.18}\,\textup{s}$ & $\textbf{0.47}\,\textup{s}$ \\
                   & Baseline & $0.20\,\textup{s}$ & $0.35\,\textup{s}$ & $0.85\,\textup{s}$ \\ \hline
\end{tabular}
\end{table}

\section{Conclusion}

In this paper, we present an integrated decision-making and trajectory planning framework for autonomous vehicles. Based on the game of incomplete information, multimodal behaviors of human drivers are properly handled, the optimal decision is reached, and the planning is performed in a fully interactive manner. Simulation demonstrates the effectiveness of the proposed approach across multiple traffic scenarios and the improvements in terms of passengers' comfort and security over traditional game method with smaller accelerations, larger safety distances, and lower collision rates. Possible future works include modeling the correlation between the driving intentions of vehicles and integrating data-driven methods to learn the cost functions online for accurate imitation of human driving behaviors.

\bibliographystyle{IEEEtran}
\bibliography{refs}

% Generated by IEEEtran.bst, version: 1.14 (2015/08/26)
\begin{thebibliography}{10}
\providecommand{\url}[1]{#1}
\csname url@samestyle\endcsname
\providecommand{\newblock}{\relax}
\providecommand{\bibinfo}[2]{#2}
\providecommand{\BIBentrySTDinterwordspacing}{\spaceskip=0pt\relax}
\providecommand{\BIBentryALTinterwordstretchfactor}{4}
\providecommand{\BIBentryALTinterwordspacing}{\spaceskip=\fontdimen2\font plus
\BIBentryALTinterwordstretchfactor\fontdimen3\font minus
  \fontdimen4\font\relax}
\providecommand{\BIBforeignlanguage}[2]{{%
\expandafter\ifx\csname l@#1\endcsname\relax
\typeout{** WARNING: IEEEtran.bst: No hyphenation pattern has been}%
\typeout{** loaded for the language `#1'. Using the pattern for}%
\typeout{** the default language instead.}%
\else
\language=\csname l@#1\endcsname
\fi
#2}}
\providecommand{\BIBdecl}{\relax}
\BIBdecl

\bibitem{trautman2010unfreezing}
P.~Trautman and A.~Krause, ``Unfreezing the robot: Navigation in dense,
  interacting crowds,'' in \emph{2010 IEEE/RSJ International Conference on
  Intelligent Robots and Systems}.\hskip 1em plus 0.5em minus 0.4em\relax IEEE,
  2010, pp. 797--803.

\bibitem{jia2023interactive}
S.~Jia, Y.~Zhang, X.~Li, X.~Na, Y.~Wang, B.~Gao, B.~Zhu, and R.~Yu,
  ``Interactive decision-making with switchable game modes for automated
  vehicles at intersections,'' \emph{IEEE Transactions on Intelligent
  Transportation Systems}, vol.~24, no.~11, pp. 11\,785--11\,799, 2023.

\bibitem{hang2022decision}
P.~Hang, C.~Huang, Z.~Hu, and C.~Lv, ``Decision making for connected automated
  vehicles at urban intersections considering social and individual benefits,''
  \emph{IEEE Transactions on Intelligent Transportation Systems}, vol.~23,
  no.~11, pp. 22\,549--22\,562, 2022.

\bibitem{cleac2019algames}
S.~L. Cleac'h, M.~Schwager, and Z.~Manchester, ``A{LGAMES}: A fast solver for
  constrained dynamic games,'' \emph{arXiv preprint arXiv:1910.09713}, 2019.

\bibitem{hang2021cooperative}
P.~Hang, C.~Lv, C.~Huang, Y.~Xing, and Z.~Hu, ``Cooperative decision making of
  connected automated vehicles at multi-lane merging zone: A coalitional game
  approach,'' \emph{IEEE Transactions on Intelligent Transportation Systems},
  vol.~23, no.~4, pp. 3829--3841, 2021.

\bibitem{fabiani2019multi}
F.~Fabiani and S.~Grammatico, ``Multi-vehicle automated driving as a
  generalized mixed-integer potential game,'' \emph{IEEE Transactions on
  Intelligent Transportation Systems}, vol.~21, no.~3, pp. 1064--1073, 2019.

\bibitem{lopez2022game}
V.~G. Lopez, F.~L. Lewis, M.~Liu, Y.~Wan, S.~Nageshrao, and D.~Filev,
  ``Game-theoretic lane-changing decision making and payoff learning for
  autonomous vehicles,'' \emph{IEEE Transactions on Vehicular Technology},
  vol.~71, no.~4, pp. 3609--3620, 2022.

\bibitem{hang2021decision}
P.~Hang, C.~Huang, Z.~Hu, Y.~Xing, and C.~Lv, ``Decision making of connected
  automated vehicles at an unsignalized roundabout considering personalized
  driving behaviours,'' \emph{IEEE Transactions on Vehicular Technology},
  vol.~70, no.~5, pp. 4051--4064, 2021.

\bibitem{liu2023potential}
M.~Liu, I.~Kolmanovsky, H.~E. Tseng, S.~Huang, D.~Filev, and A.~Girard,
  ``Potential game-based decision-making for autonomous driving,'' \emph{IEEE
  Transactions on Intelligent Transportation Systems}, vol.~24, no.~8, pp.
  8014--8027, 2023.

\bibitem{kavuncu2021potential}
T.~Kavuncu, A.~Yaraneri, and N.~Mehr, ``Potential i{LQR}: A
  potential-minimizing controller for planning multi-agent interactive
  trajectories,'' \emph{arXiv preprint arXiv:2107.04926}, 2021.

\bibitem{zanardi2023bad}
A.~Zanardi, P.~G. Sessa, N.~K{\"a}slin, S.~Bolognani, A.~Censi, and
  E.~Frazzoli, ``How bad is selfish driving? bounding the inefficiency of
  equilibria in urban driving games,'' \emph{IEEE Robotics and Automation
  Letters}, vol.~8, no.~4, pp. 2293--2300, 2023.

\bibitem{chandra2022gameplan}
R.~Chandra and D.~Manocha, ``Game{P}lan: Game-theoretic multi-agent planning
  with human drivers at intersections, roundabouts, and merging,'' \emph{IEEE
  Robotics and Automation Letters}, vol.~7, no.~2, pp. 2676--2683, 2022.

\bibitem{fisac2019hierarchical}
J.~F. Fisac, E.~Bronstein, E.~Stefansson, D.~Sadigh, S.~S. Sastry, and A.~D.
  Dragan, ``Hierarchical game-theoretic planning for autonomous vehicles,'' in
  \emph{2019 International Conference on Robotics and Automation (ICRA)}.\hskip
  1em plus 0.5em minus 0.4em\relax IEEE, 2019, pp. 9590--9596.

\bibitem{yu2018human}
H.~Yu, H.~E. Tseng, and R.~Langari, ``A human-like game theory-based controller
  for automatic lane changing,'' \emph{Transportation Research Part C: Emerging
  Technologies}, vol.~88, pp. 140--158, 2018.

\bibitem{deng2022lane}
Z.~Deng, W.~Hu, Y.~Yang, K.~Cao, D.~Cao, and A.~Khajepour, ``Lane change
  decision-making with active interactions in dense highway traffic: A bayesian
  game approach,'' in \emph{2022 IEEE 25th International Conference on
  Intelligent Transportation Systems (ITSC)}.\hskip 1em plus 0.5em minus
  0.4em\relax IEEE, 2022, pp. 3290--3297.

\bibitem{zhao2023non}
T.~Zhao, W.~ShangGuan, L.~Chai, and Y.~Cao, ``A non-cooperative dynamic
  game-based approach for lane changing decision-making in mixed traffic
  diversion scenarios,'' in \emph{2023 IEEE 26th International Conference on
  Intelligent Transportation Systems (ITSC)}.\hskip 1em plus 0.5em minus
  0.4em\relax IEEE, 2023, pp. 4776--4781.

\bibitem{zhang2022human}
Y.~Zhang, P.~Hang, C.~Huang, and C.~Lv, ``Human-like interactive behavior
  generation for autonomous vehicles: a bayesian game-theoretic approach with
  turing test,'' \emph{Advanced Intelligent Systems}, vol.~4, no.~5, p.
  2100211, 2022.

\bibitem{ziegler2014making}
J.~Ziegler, P.~Bender, M.~Schreiber, H.~Lategahn, T.~Strauss, C.~Stiller,
  T.~Dang, U.~Franke, N.~Appenrodt, C.~G. Keller \emph{et~al.}, ``Making bertha
  drive—an autonomous journey on a historic route,'' \emph{IEEE Intelligent
  Transportation Systems Magazine}, vol.~6, no.~2, pp. 8--20, 2014.

\bibitem{urmson2008autonomous}
C.~Urmson, J.~Anhalt, D.~Bagnell, C.~Baker, R.~Bittner, M.~Clark, J.~Dolan,
  D.~Duggins, T.~Galatali, C.~Geyer \emph{et~al.}, ``Autonomous driving in
  urban environments: Boss and the urban challenge,'' \emph{Journal of Field
  Robotics}, vol.~25, no.~8, pp. 425--466, 2008.

\bibitem{montemerlo2008junior}
M.~Montemerlo, J.~Becker, S.~Bhat, H.~Dahlkamp, D.~Dolgov, S.~Ettinger,
  D.~Haehnel, T.~Hilden, G.~Hoffmann, B.~Huhnke \emph{et~al.}, ``Junior: The
  stanford entry in the urban challenge,'' \emph{Journal of Field Robotics},
  vol.~25, no.~9, pp. 569--597, 2008.

\bibitem{wei2014behavioral}
J.~Wei, J.~M. Snider, T.~Gu, J.~M. Dolan, and B.~Litkouhi, ``A behavioral
  planning framework for autonomous driving,'' in \emph{2014 IEEE Intelligent
  Vehicles Symposium Proceedings}.\hskip 1em plus 0.5em minus 0.4em\relax IEEE,
  2014, pp. 458--464.

\bibitem{zhan2017spatially}
W.~Zhan, J.~Chen, C.-Y. Chan, C.~Liu, and M.~Tomizuka, ``Spatially-partitioned
  environmental representation and planning architecture for on-road autonomous
  driving,'' in \emph{2017 IEEE Intelligent Vehicles Symposium (IV)}.\hskip 1em
  plus 0.5em minus 0.4em\relax IEEE, 2017, pp. 632--639.

\bibitem{ajanovicsearch}
Z.~Ajanovic, B.~Lacevic, B.~Shyrokau, M.~Stolz, and M.~Horn, ``Search-based
  optimal motion planning for automated driving. in 2018 {IEEE},'' in \emph{RSJ
  International Conference on Intelligent Robots and Systems (IROS)}, pp.
  4523--4530.

\bibitem{kessler2019cooperative}
T.~Kessler and A.~Knoll, ``Cooperative multi-vehicle behavior coordination for
  autonomous driving,'' in \emph{2019 IEEE Intelligent Vehicles Symposium
  (IV)}.\hskip 1em plus 0.5em minus 0.4em\relax IEEE, 2019, pp. 1953--1960.

\bibitem{Ma2022Alternating}
J.~Ma, Z.~Cheng, X.~Zhang, M.~Tomizuka, and T.~H. Lee, ``Alternating direction
  method of multipliers for constrained iterative {LQR} in autonomous
  driving,'' \emph{IEEE Transactions on Intelligent Transportation Systems},
  vol.~23, no.~12, pp. 23\,031--23\,042, 2022.

\bibitem{huang2023decentralized}
Z.~Huang, S.~Shen, and J.~Ma, ``Decentralized {iLQR} for cooperative trajectory
  planning of connected autonomous vehicles via dual consensus {ADMM},''
  \emph{IEEE Transactions on Intelligent Transportation Systems}, vol.~24,
  no.~11, pp. 12\,754--12\,766, 2023.

\bibitem{mustafa2024racp}
K.~A. Mustafa, D.~J. Ornia, J.~Kober, and J.~Alonso-Mora, ``R{ACP}: Risk-aware
  contingency planning with multi-modal predictions,'' \emph{IEEE Transactions
  on Intelligent Vehicles}, 2024.

\bibitem{werling2012optimal}
M.~Werling, S.~Kammel, J.~Ziegler, and L.~Gr{\"o}ll, ``Optimal trajectories for
  time-critical street scenarios using discretized terminal manifolds,''
  \emph{The International Journal of Robotics Research}, vol.~31, no.~3, pp.
  346--359, 2012.

\bibitem{liu2015situation}
W.~Liu, S.-W. Kim, S.~Pendleton, and M.~H. Ang, ``Situation-aware decision
  making for autonomous driving on urban road using online pomdp,'' in
  \emph{2015 IEEE Intelligent Vehicles Symposium (IV)}.\hskip 1em plus 0.5em
  minus 0.4em\relax IEEE, 2015, pp. 1126--1133.

\bibitem{cunningham2015mpdm}
A.~G. Cunningham, E.~Galceran, R.~M. Eustice, and E.~Olson, ``M{PDM}:
  Multipolicy decision-making in dynamic, uncertain environments for autonomous
  driving,'' in \emph{2015 IEEE International Conference on Robotics and
  Automation (ICRA)}.\hskip 1em plus 0.5em minus 0.4em\relax IEEE, 2015, pp.
  1670--1677.

\bibitem{zhang2020efficient}
L.~Zhang, W.~Ding, J.~Chen, and S.~Shen, ``Efficient uncertainty-aware
  decision-making for automated driving using guided branching,'' in \emph{2020
  IEEE International Conference on Robotics and Automation (ICRA)}.\hskip 1em
  plus 0.5em minus 0.4em\relax IEEE, 2020, pp. 3291--3297.

\bibitem{9526613}
W.~Ding, L.~Zhang, J.~Chen, and S.~Shen, ``E{PSILON}: An efficient planning
  system for automated vehicles in highly interactive environments,''
  \emph{IEEE Transactions on Robotics}, vol.~38, no.~2, pp. 1118--1138, 2022.

\bibitem{li2023marc}
T.~Li, L.~Zhang, S.~Liu, and S.~Shen, ``M{ARC}: Multipolicy and risk-aware
  contingency planning for autonomous driving,'' \emph{IEEE Robotics and
  Automation Letters}, 2023.

\bibitem{hu2023planning}
Y.~Hu, J.~Yang, L.~Chen, K.~Li, C.~Sima, X.~Zhu, S.~Chai, S.~Du, T.~Lin,
  W.~Wang \emph{et~al.}, ``Planning-oriented autonomous driving,'' in
  \emph{Proceedings of the IEEE/CVF Conference on Computer Vision and Pattern
  Recognition}, 2023, pp. 17\,853--17\,862.

\bibitem{huang2023differentiable}
Z.~Huang, H.~Liu, J.~Wu, and C.~Lv, ``Differentiable integrated motion
  prediction and planning with learnable cost function for autonomous
  driving,'' \emph{IEEE transactions on neural networks and learning systems},
  2023.

\bibitem{huang2023gameformer}
Z.~Huang, H.~Liu, and C.~Lv, ``Gameformer: Game-theoretic modeling and learning
  of transformer-based interactive prediction and planning for autonomous
  driving,'' in \emph{Proceedings of the IEEE/CVF International Conference on
  Computer Vision}, 2023, pp. 3903--3913.

\bibitem{fridovich2020efficient}
D.~Fridovich-Keil, E.~Ratner, L.~Peters, A.~D. Dragan, and C.~J. Tomlin,
  ``Efficient iterative linear-quadratic approximations for nonlinear
  multi-player general-sum differential games,'' in \emph{2020 IEEE
  International Conference on Robotics and Automation (ICRA)}.\hskip 1em plus
  0.5em minus 0.4em\relax IEEE, 2020, pp. 1475--1481.

\bibitem{li2022efficient}
C.~Li, T.~Trinh, L.~Wang, C.~Liu, M.~Tomizuka, and W.~Zhan, ``Efficient
  game-theoretic planning with prediction heuristic for socially-compliant
  autonomous driving,'' \emph{IEEE Robotics and Automation Letters}, vol.~7,
  no.~4, pp. 10\,248--10\,255, 2022.

\bibitem{schwarting2021stochastic}
W.~Schwarting, A.~Pierson, S.~Karaman, and D.~Rus, ``Stochastic dynamic games
  in belief space,'' \emph{IEEE Transactions on Robotics}, vol.~37, no.~6, pp.
  2157--2172, 2021.

\bibitem{mehr2023maximum}
N.~Mehr, M.~Wang, M.~Bhatt, and M.~Schwager, ``Maximum-entropy multi-agent
  dynamic games: Forward and inverse solutions,'' \emph{IEEE Transactions on
  Robotics}, vol.~39, no.~3, pp. 1801--1815, 2023.

\bibitem{shao2020discretionary}
H.~Shao, M.~Zhang, T.~Feng, and Y.~Dong, ``A discretionary lane-changing
  decision-making mechanism incorporating drivers’ heterogeneity: A
  signalling game-based approach,'' \emph{Journal of Advanced Transportation},
  vol. 2020, no.~1, p. 8892693, 2020.

\bibitem{yao2022modelling}
R.~Yao and X.~Du, ``Modelling lane changing behaviors for bus exiting at bus
  bay stops considering driving styles: A game theoretical approach,''
  \emph{Travel Behaviour and Society}, vol.~29, pp. 319--329, 2022.

\bibitem{li2024cooperative}
L.~Li, W.~Zhao, and C.~Wang, ``Cooperative merging strategy considering
  stochastic driving style at on-ramps: A bayesian game approach,''
  \emph{Automotive Innovation}, vol.~7, no.~2, pp. 312--334, 2024.

\bibitem{peters2024contingency}
L.~Peters, A.~Bajcsy, C.-Y. Chiu, D.~Fridovich-Keil, F.~Laine, L.~Ferranti, and
  J.~Alonso-Mora, ``Contingency games for multi-agent interaction,'' \emph{IEEE
  Robotics and Automation Letters}, 2024.

\bibitem{zinkevich2007regret}
M.~Zinkevich, M.~Johanson, M.~Bowling, and C.~Piccione, ``Regret minimization
  in games with incomplete information,'' \emph{Advances in Neural Information
  Processing Systems}, vol.~20, 2007.

\bibitem{lanctot2009monte}
M.~Lanctot, K.~Waugh, M.~Zinkevich, and M.~Bowling, ``Monte carlo sampling for
  regret minimization in extensive games,'' \emph{Advances in Neural
  Information Processing Systems}, vol.~22, 2009.

\bibitem{lisy2015online}
V.~Lis{\`y}, M.~Lanctot, and M.~H. Bowling, ``Online monte carlo counterfactual
  regret minimization for search in imperfect information games.'' in
  \emph{AAMAS}, 2015, pp. 27--36.

\bibitem{celli2019learning}
A.~Celli, A.~Marchesi, T.~Bianchi, and N.~Gatti, ``Learning to correlate in
  multi-player general-sum sequential games,'' \emph{Advances in Neural
  Information Processing Systems}, vol.~32, 2019.

\bibitem{hartline2015no}
J.~Hartline, V.~Syrgkanis, and E.~Tardos, ``No-regret learning in bayesian
  games,'' \emph{Advances in Neural Information Processing Systems}, vol.~28,
  2015.

\bibitem{maschler2020game}
M.~Maschler, S.~Zamir, and E.~Solan, \emph{Game Theory}.\hskip 1em plus 0.5em
  minus 0.4em\relax Cambridge, U.K.: Cambridge University Press, 2013.

\bibitem{sen2017large}
P.~K. Sen and J.~M. Singer, \emph{Large sample methods in statistics: An
  introduction with applications}.\hskip 1em plus 0.5em minus 0.4em\relax Boca
  Raton, FL: Chapman \& Hall CRC, 2000.

\end{thebibliography}

\end{document}